\documentclass[acmsmall,screen]{acmart}

\usepackage[english]{babel}
\usepackage{hyperref}
\usepackage{caption, subcaption}
\usepackage{booktabs,multirow}
\usepackage{longtable}
\usepackage[ruled, linesnumbered, vlined, procnumbered]{algorithm2e}
\usepackage{mathtools}
\usepackage{interval}
\usepackage{pifont}
\usepackage[inline,shortlabels]{enumitem}
\usepackage[per-mode = symbol]{siunitx}
\usepackage{tikz}
\usetikzlibrary{arrows.meta, calc}
\usepackage{array}
\usepackage[capitalise, noabbrev]{cleveref}
\usepackage[shortcuts]{extdash}
\usepackage[htt]{hyphenat}
\usepackage{placeins}
\usepackage{float}

\newcommand{\PO}{\mathcal{P}}
\newcommand{\POhat}{\hat{\mathcal{P}}}
\newcommand{\Inst}{\mathcal{I}}
\newcommand{\Rk}{\mathcal{R}}
\newcommand{\Tg}{\mathbf{T}}
\newcommand{\Model}{\mathcal{M}}
\newcommand{\Open}{\mathcal{U}}
\newcommand{\Active}{\hat{\symbf{A}}}
\newcommand{\half}{\tfrac{1}{2}}
\newcommand{\norel}{\bot}
\newcommand{\POeps}{\PO_{\epsilon}}
\newcommand{\alg}[1]{\texttt{#1}}

\newcolumntype{L}[1]{>{\raggedright\arraybackslash}p{#1}}
\crefname{algocf}{Algorithm}{Algorithms}
\Crefname{algocf}{Algorithm}{Algorithms}

\setlist[description]{font=\normalfont\itshape\textbullet\space}

\setcopyright{acmlicensed}
\copyrightyear{2026}
\acmYear{2026}
\acmDOI{XXXXXXX.XXXXXXX}

\acmJournal{TELO}
\acmVolume{1}
\acmNumber{1}
\acmArticle{1}
\acmMonth{2}

\begin{document}

\title{Pareto-Optimal Anytime Algorithms via Bayesian Racing}

\author{Jonathan Wurth}
\email{jonathan.wurth@uni-a.de}
\affiliation{%
  \institution{Universität Augsburg}
  \city{Augsburg}
  \country{Germany}
}
\orcid{0000-0002-5799-024X}

\author{Helena Stegherr}
\email{helena.stegherr@uni-a.de}
\affiliation{%
  \institution{Universität Augsburg}
  \city{Augsburg}
  \country{Germany}
}
\orcid{0000-0001-7871-7309}

\author{Neele Kemper}
\email{neele.kemper@uni-a.de}
\affiliation{%
  \institution{Universität Augsburg}
  \city{Augsburg}
  \country{Germany}
}
\orcid{0000-0003-2110-1006}

\author{Michael Heider}
\email{michael.heider@uni-a.de}
\affiliation{%
  \institution{Universität Augsburg}
  \city{Augsburg}
  \country{Germany}
}
\orcid{0000-0003-3140-1993}

\author{Jörg Hähner}
\email{joerg.haehner@uni-a.de}
\affiliation{%
  \institution{Universität Augsburg}
  \city{Augsburg}
  \country{Germany}
}
\orcid{0000-0003-0107-264X}

\renewcommand{\shortauthors}{Wurth et al.}

\begin{abstract}
  Anytime optimization algorithms return the best solution found so far when stopped, so their performance is a function of the budget, but the budget available at deployment is often unknown during benchmarking.
  Standard benchmarking fixes the number of instances in advance, normalizes objective values per instance, and leaves readers to infer the trade-off over budgets from plots.
  We instead compare algorithms by their win probability at each budget, the probability that one is ahead on a random instance of interest, which needs no normalization.
  The result is the anytime Pareto set, the algorithms that no other beats at every budget, from which the best algorithm under a preference over budgets can be selected once the preference is known.

  We propose \textsc{PolarBear}, a Bayesian racing procedure that identifies this set, running only the algorithms whose comparisons are still open and stopping once the evidence suffices.
  It takes its model as a plug-in and, if the model is exactly specified, comes with a guarantee on false statements.
  On seven benchmarks with known win probabilities, models that estimate each pair make no more false statements than their posterior probabilities imply unless their prior is certain too early.
  Models with one rating per algorithm make more.
  Repeated races on the same benchmark need severalfold different numbers of instances, depending on which instances they draw.
  A fixed sample size therefore leaves some comparisons unresolved and keeps running every algorithm in others after the evidence suffices.
\end{abstract}

\ccsdesc[500]{Computing methodologies~Ranking}
\ccsdesc[500]{Theory of computation~Bayesian analysis}
\ccsdesc[300]{Theory of computation~Active learning}
\ccsdesc[500]{Theory of computation~Optimization with randomized search heuristics}
\ccsdesc[300]{Computing methodologies~Uncertainty quantification}

\keywords{anytime optimization, algorithm benchmarking, Bayesian experiment design, racing algorithms}

\maketitle

\section{Introduction}
\label{sec:intro}

Optimization algorithms are compared on the problem instances of interest before one of them is deployed.
The budget available at deployment, however, is often not known when they are compared.
It may depend on the resources available, on a user's patience, or on constraints that appear only later.
What is known at deployment then takes many forms: a fixed allocation (``at most 1{,}000 evaluations''), a range (``between one and two hours''), a distribution over budgets (``probably around 90 minutes, but possibly cut short''), or a weighting of early against late performance (``solutions found in the first minute matter most'').
An \emph{anytime algorithm} can be stopped at any budget and then returns the best solution it has found.
In black-box optimization on noiseless problems every algorithm is one, since the best solution found so far can be tracked.
Its \emph{anytime performance} is the quality of that solution as a function of the budget, not a single final value.
The question we address is therefore: \emph{what should a comparison of algorithms deliver when the budget is unknown at the time of the comparison, so that an algorithm can be selected later under whatever becomes known about it?}
The comparison we consider is made on one distribution of instances and reads no features of them; choosing an algorithm for each instance from its features is an orthogonal task~\cite{kerschke2019,munoz2015,meunier2022}.
We reserve the word ``selection'' for the step after the comparison: choosing an algorithm from its result once a preference over budgets is known.

The standard tool for the anytime performance analysis of stochastic optimization algorithms is the empirical attainment function (EAF): the empirical estimate of the probability that an algorithm reaches an objective value of $z$ or better by budget $t$~\cite{lopez-ibanez2024}.
Its common summary, the EAF-based empirical cumulative distribution function (ECDF), integrates the EAF over an interval of objective values: at each budget it is the share of the interval that the algorithm attains, averaged over its runs.
The lower end of the interval is usually the optimum, or the best value known.
A benchmarking study that uses them needs manual steps before and after the experiment: the number of instances and runs is chosen in advance, and the resulting plots are read by eye to identify well-performing algorithms under a preference over budgets.
For automated uses such as algorithm configuration and design, the area over the convergence curve (AOCC) has become a common measure~\cite{lopez-ibanez2014,vanstein2024a,ye2022,denobel2021}; despite its name, it is most easily defined as the area under the ECDF, its integral over the budget.

On one instance the interval can be taken from the observed runs.
Across instances it differs, so the objective values are min--max normalized per instance before the EAF and the ECDF are aggregated~\cite{lopez-ibanez2024}.
Normalizing requires a lower and an upper bound for every instance.
For studied benchmarks these are known; otherwise they are the extremes of the values observed so far.
Every run that goes beyond them then moves them, whether it belongs to a new algorithm or to one already compared, and every earlier normalized value moves with them.
How $A$ compares with $B$ can therefore change when a third algorithm $C$ is added, although no run of $A$ or $B$ has changed.
Normalizing also treats equal distances on the normalized scale as equal achievements, although improving from normalized objective value 0.2 to 0.1 can take far more effort than improving from 0.9 to 0.8, for example.
Both properties come from treating objective values as magnitudes on a scale shared by all instances.
An evaluation that needs no such scale is \emph{scale-free}: it is unchanged when the objective values of each instance are transformed by a strictly increasing function of their own, so it needs no bounds and assumes nothing about distances.
The AOCC is computed from the normalized ECDF, so it is not scale-free either, and it reduces anytime performance to a single number.
That suits a tuning objective, which needs an order of the algorithms, but it hides the trade-off over budgets: an algorithm that converges fast and then stagnates and one that starts slowly and keeps improving can have the same area.

Whatever the measure, conclusions about the performance of algorithms are uncertain when the experiment is limited in size~\cite{vermetten2022a}: the instances that are run are a sample from those of interest, and the runs of a stochastic algorithm on them are a sample from its random seeds.
After a handful of instances one algorithm may look better than another, but how sure is that, and when is the evidence enough to stop?
A posterior probability answers the first question directly, for example as the probability that $A$ is ahead of $B$ at budget $t$ on the majority of instances of interest.
The second question is harder, because stopping once the data look convincing is itself a decision made from the data, and it changes what many statistical statements mean.
A posterior probability keeps its meaning under any such rule if the model is exactly specified, that is, if the truth is a draw from the model's prior and the observations follow its likelihood~\cite{rouder2014,hendriksen2021}.
The meaning holds on average over the truths the prior draws, not for each truth on its own.
A Bayesian sequential procedure over an exactly specified model can therefore read it after every instance and stop as soon as the evidence suffices.

Several lines of work address parts of this question (\cref{sec:related}): Bayesian inference on rankings of optimization algorithms~\cite{calvo2018,calvo2019,rojas-delgado2022}, ratings or rankings at each budget with an error statement~\cite{ravber2024,vermetten2025}, a choice among anytime algorithms under a utility that becomes known only when the choice is made~\cite{jesus2020}, and racing, which drops candidates during the experiment, with frequentist tests, as in F-race and irace~\cite{birattari2010,lopez-ibanez2016}, or Bayesian ones~\cite{benavoli2015}.
The statistical methods among these analyze each budget, or one aggregate over budgets, on its own, and the racers race at a single budget, so none of them makes a statement about all budgets at once.
\citet{jesus2020} weigh all budgets, but on a quality scale common to all instances and without an error statement.

We propose a Bayesian framework for anytime comparison that meets four criteria: it is scale-free, keeps the trade-off over budgets, quantifies its uncertainty, and runs without manual steps.
It compares two algorithms at each budget by their win probability, the probability that a run of one has found a better solution by that budget than a run of the other on a random instance of interest, a tie counting one half; since this depends only on which of the two runs is ahead, it is scale-free.
One algorithm \emph{beats} another at a budget when its win probability there is above one half.
An algorithm that another beats at every budget is the worse of the two under every preference over budgets.
With each budget as an objective, the result of the comparison is therefore the \emph{anytime Pareto set}, the algorithms that no other beats at every budget; it keeps every trade-off over budgets, and an algorithm is selected from it once a preference is known.
To identify this set, we propose \textsc{PolarBear}\footnote{\textbf{P}areto-\textbf{o}ptima\textbf{l} \textbf{a}nytime algo\textbf{r}ithms via \textbf{B}ay\textbf{e}si\textbf{a}n \textbf{r}acing.}, a racing procedure that drops an algorithm once another beats it at every budget, and keeps both algorithms of a pair when each is ahead at some budgets or when the two are practically equivalent.
It is a form of sequential experiment design, deciding during the experiment which algorithms are still run, how many instances the race takes and when it stops, so no sample size has to be fixed in advance.
The posterior over the win probabilities comes from a model, which the framework takes as a plug-in.
We study a ladder of three models: a beta-binomial model with a separate win probability for every pair, a rating scale, one rating per algorithm, fitted to pairwise outcomes (Bradley--Terry), and such a scale fitted to whole rankings (Plackett--Luce).
Beside the beta-binomial model we study the Dirichlet process posterior of \citet{benavoli2015} under two priors: a tie-point prior, which adds one pseudo-instance on which every pair ties, and a near-ignorance prior, which places that instance against each statement and so gives the statement its lowest probability over a set of priors.
If the model is exactly specified, a race makes, on average over the prior, no more false statements than the posterior probabilities of its statements imply; we call this its bound.
On recorded runs no model can be assumed exactly specified, so we measure the errors there against a truth computed from a much larger pool of runs.

We make five contributions.
\begin{enumerate}
  \item The pairwise win probability at each budget as a scale-free estimand of anytime performance, and the anytime Pareto set it defines as what a comparison should deliver when the budget is unknown (\cref{sec:scale-free}).
  \item Five existing models as plug-ins of the framework, each with the assumption under which it is exactly specified, so that its guarantee can be tested (\cref{sec:models}).
  \item \textsc{PolarBear} over any posterior on pairs, with its guarantee and the guarantee's limits; a joint test over budgets and, as its cheaper heuristic, a pointwise test; and three ways to select an algorithm from the estimated Pareto set under a preference over budgets (\cref{sec:polarbear}).
  \item A sanity check of the race under its ideal conditions, where the truth is a draw from the model's prior, and under a misspecified condition, where outcomes depend on each other across budgets as the runs of real instances do; there, a prior on Plackett--Luce under which the ratings change smoothly with the budget turns overconfident, and we report it as a negative result (\cref{sec:validation}).
  \item Races on seven pools of recorded runs, whose win probabilities over the pool can be computed exactly, against F-race and two Bayesian variants of it, one with the tests of \citet{benavoli2015} and one with those of the beta-binomial model (\cref{sec:replays}). The beta-binomial model and the near-ignorance prior stay within their bound there, while Bradley--Terry and Plackett--Luce exceed theirs by what the compression onto a rating scale loses.
  Whichever model the race uses, the number of instances until the model is confident about which algorithm beats which at each budget varies severalfold between repetitions on the same pool, depending on which instances are drawn, and within one repetition from one pair of algorithms to the next.
  A sample size fixed in advance therefore stops before the model is confident about some pairs, keeps running algorithms after it is confident about theirs, or both: a fixed design of 1{,}000 instances would cost more runs than our races make on average and still leave pairs undecided in nearly half of them.
  Across budgets the race reaches a complete answer where racing each budget on its own ends undecided at near-ties; at one budget \textsc{PolarBear} is a race of the same kind as the two Bayesian variants.
\end{enumerate}
\Cref{sec:related} discusses related work, and \cref{sec:conclusion} the limitations and future work.

\section{Scale-free Anytime Performance}
\label{sec:scale-free}

We define an anytime algorithm $A$ as a randomized procedure with access to an objective function $f\colon \mathcal{X} \to \mathcal{Y}$, the problem instance.
For a random seed $\omega \in \Omega$, a run $A(f, \omega)$ evaluates a sequence of solutions in $\mathcal{X}$, each after some amount of budget has been spent.
A budget is an amount of the resource a run consumes, such as evaluations or time.
Assuming minimization, a noiseless $f$ and a total order on $\mathcal{Y}$, the best-so-far value
\begin{equation*}
	y_A(t) = \min \, \bigl\{ \, f(x) : x \text{ evaluated by } A(f, \omega) \text{ within budget } t \, \bigr\}
\end{equation*}
can be computed from outside the algorithm by tracking every evaluated solution; we leave $f$ and $\omega$ out of the notation.
An algorithm thus induces a distribution over best-so-far trajectories $y_A$.
Algorithms are compared on a grid of budgets $\Tg = (t_1 < \dots < t_T)$ that is fixed before the experiment and contains every budget of interest, and ``at every budget'' means at every budget of this grid.
Algorithms with budget-dependent components, such as adaptation based on the share of the budget spent, define a family of anytime algorithms, one per configured budget~$t'$.
Each member $A(t')$ has a well-defined trajectory up to $t'$, continued by restarts when $t' < t_T$, and can enter a comparison as an algorithm of its own.

An anytime comparison is scale-free when its result is unchanged under a strictly increasing transformation of the objective values of each instance, a separate one per instance.
What every such transformation preserves is the order of the values on an instance, so the observations of our comparison are rankings.
All compared algorithms $\mathbf{A} = \{A_1, \dots, A_n\}$ are run on the same instance, and at every budget of the grid we record the order of their best-so-far values, ties included, and not the values themselves.
This reduction is deliberate.
Suppose $A$ reaches an objective value of 47.3 on an instance and $B$ one of 52.1.
That $A$ did better needs no further knowledge, but whether a difference of 4.8 is much or little does: it depends on the landscape and on how hard such a difference is to achieve, neither of which a black-box setting reveals.
The performance assessment of the COCO platform (Comparing Continuous Optimizers) makes the same point about values reached at a fixed budget: that one is two times smaller than another ``has in general no interpretable meaning''~\cite{hansen2016}.
Search algorithms are designed with this in mind.
Comparison-based algorithms such as the covariance matrix adaptation evolution strategy (CMA-ES) and differential evolution use only the order of the objective values, so they behave identically on $f$ and on any strictly increasing transformation of $f$~\cite{hansen2023}.
An evaluation that changes under such a transformation would report a difference between two experiments in which the algorithms did exactly the same.

Normalizations can be ordered by the transformations they are invariant to.
Min--max normalization is unchanged when the values of an instance are shifted and rescaled, but it needs a lower and an upper bound for every instance, which a black-box setting does not supply.
Taking the logarithm before min--max normalization adds the assumption that progress counts by orders of magnitude of the distance to the lower bound.
That reading holds only if the lower bound is the true optimum, and a lower bound taken from the data moves every earlier value when it moves, most of all the values near it.
Ranking is unchanged under every strictly increasing transformation, a separate one per instance, and needs no bounds.
It is the end of this ordering: an evaluation that is unchanged under all such transformations of an instance's values depends on them only through their order, which is the maximal invariant of these transformations~\cite[Example~6.2.2]{lehmann2005}.
Ranking is therefore not an alternative to normalization but the normalization that assumes nothing about the scale, and accordingly it discards all magnitude.

Discarding magnitude changes what a result says.
An algorithm that is ahead on an instance counts as ahead whatever the size of its lead: a small but consistent advantage counts in full, and a rare large loss counts no more than a narrow one.
If differences below a target value are not decision-relevant, values below the target can be set to it before ranking, so that all algorithms that reach it tie.
Discarding magnitude also has a cost.
Where the scale is known, meaningful and comparable across instances, magnitudes carry information, and an analysis that uses them needs fewer instances to reach the same conclusion.
Constructed benchmarks with known optima come closest to this, but even on the black-box optimization benchmarking (BBOB) functions of COCO, whose optima are known, ``difficulty versus function value is not uniform''~\cite{hansen2009}.
We address the setting without such a scale.

Instances are drawn from a distribution $\Inst$, the problems of interest.
The quantity we estimate is the pairwise win probability at each budget of the grid, with a tie counting one half:
\begin{equation*}
	p_{AB}(t) = P\bigl(y_A(t) < y_B(t)\bigr) + \half\, P\bigl(y_A(t) = y_B(t)\bigr),
\end{equation*}
where the probability is over the instance $I \sim \Inst$ and the seeds of one run of each algorithm on it, drawn independently, so that $p_{AB}(t) + p_{BA}(t) = 1$.
We say that $A$ \emph{beats} $B$ at $t$ when $p_{AB}(t) > \half$.
The win probability is a mean over instances and runs of a win, a tie or a loss, and it reads directly as how often $A$ is ahead of $B$ at budget $t$ on a random instance of interest, ties counting half.
A mean of normalized values, by contrast, is a mean on the scale the normalization built, and it carries that normalization's assumptions.
The win probability belongs to its instance distribution: another distribution defines another quantity.
It depends only on $A$, $B$, the budget and $\Inst$, not on which other algorithms are compared.
Summaries of a whole ranking do depend on them: an algorithm's probability of ranking first and its mean rank, which is one plus the sum of its loss probabilities against the others, both change when an algorithm is added, and an addition can reverse the order of two others.

COCO's performance assessment asks that a performance measure be quantitative, ``as opposed to a simple ranking'' of the algorithms, and ideally on a ratio scale, so that one algorithm can be called so many times better than another; it takes the runtime to reach a target value for this~\cite{hansen2016}.
The win probability meets the first part of this requirement.
Its reduction to ranks is of the observations, not of the result: $p_{AB}(t)$ is a number with a direct reading, and a pair at 0.9 means that $A$ is ahead far more often than a pair at 0.6.
We report these probabilities and their curves over the budget, not an order of the algorithms.
The win probability does not meet the second part: it says how often one algorithm is ahead of another, not by how much or how many times faster.
A runtime has a ratio scale because a target value fixes a point on the objective scale, and a meaningful target takes a known optimum or a reference algorithm.
The win probability gives up the ratio in order to need neither, and selection does not need it: whichever budget becomes available, every candidate would run with that same budget, so selecting an algorithm compares algorithms at equal budgets.
A ratio compares them at different budgets.
The probability $p_{AB}(s, t)$ that the value of $A$ at budget $s$ is better than that of $B$ at budget $t$ is scale-free in the same way, and the win probability is its diagonal $s = t$.
Off the diagonal it asks how much budget one algorithm needs to match another, the question of a fixed-target comparison asked without a target value: a fixed-target measure is invariant only if its targets are transformed with the objective, and run-length-based targets take them from a reference algorithm's budgets~\cite{hansen2016}.
That question matters when the budget is chosen to reach a goal rather than given, and we leave it to future work (\cref{sec:conclusion}).
The two measures therefore suit different settings: where meaningful targets exist, runtimes say more, and where they do not, the win probability is what remains quantitative.

Different algorithms can be ahead at different budgets: one that converges fast may lead early, and a more exploratory one late.
Both are valid choices, and which is preferable depends on how the budgets are valued.
We say that $A$ \emph{dominates} $B$ when it beats $B$ at every budget, $p_{AB}(t) > \half$ for all $t \in \Tg$, and the \emph{anytime Pareto set} holds the algorithms that no algorithm dominates:
\begin{equation*}
	\PO = \{A_j \in \mathbf{A} : \text{no } A_i \in \mathbf{A} \text{ dominates } A_j\}.
\end{equation*}
Two algorithms \emph{cross} when each beats the other at some budget, so that neither dominates the other.
An algorithm \emph{leads} at $t$ when no algorithm beats it there.
Every algorithm that leads at some budget is in $\PO$, but the converse is not true: a member of $\PO$ can lead at no budget, when a different rival beats it at each.
The anytime Pareto set is therefore not the union of the leaders over the budgets.
A \emph{preference} over budgets is a weight $w_t \geq 0$ for every budget of the grid, with the weights summing to one, and under it $A$ is preferred to $B$ when $\sum_{t \in \Tg} w_t\, p_{AB}(t) > \half$.
The weighted sum is the win probability at a budget drawn from the grid with probabilities $w_t$, so a preference can be read as uncertainty about which budget will be available.
Dominance can be read through preferences: $A$ dominates $B$ exactly when $A$ is preferred to $B$ under every preference.
This holds because a weighted mean of values above $\half$ is above $\half$, and the preference with all its weight on one budget returns that budget's win probability.
A dominated algorithm thus has a rival that is preferred to it however the budgets are weighted, and choosing among the members of $\PO$ takes a preference (\cref{sec:selection}).

Beating need not be transitive.
Suppose the instances of interest come in three equally likely types, and at some budget the three algorithms finish in the order $A, B, C$ on the first type, $B, C, A$ on the second and $C, A, B$ on the third, best first.
Then $A$ beats $B$, $B$ beats $C$ and $C$ beats $A$, each with a win probability of $\tfrac{2}{3}$.
Win probabilities over a distribution allow such cycles~\cite{ramamohan2016}, and the benchmarking literature knows them as cycle ranking~\cite{yan2022}.
A cycle has two consequences.
First, a budget can be without a leader: in the example, each of the three algorithms is beaten by another, so none of them leads at that budget.
Second, dominance need not be transitive either: if $A$ beats $B$ and $B$ beats $C$ at every budget, but the three form a cycle at one budget, then $A$ dominates $B$ and $B$ dominates $C$, while $A$ does not dominate $C$.
An algorithm outside $\PO$ therefore need not be dominated by any member of $\PO$, since its only dominators may lie outside $\PO$ themselves.
The anytime Pareto set keeps its meaning, since it is defined pair by pair: a member is an algorithm that no single rival beats at every budget.
It is empty only if dominance itself cycles, which takes a cycle that runs the same way round at every budget.

When a cycle leaves no single best algorithm under a relation such as ``beats at $t$'' or ``is preferred under a preference'', the relation still has a \emph{Smith set}, the smallest non-empty set of algorithms each of which beats every algorithm outside it~\cite{smith1973,ramamohan2016}, also called the top cycle.
It is one algorithm when one beats every other, and several otherwise.
Where no single algorithm is preferred to every other, every member of the Smith set of the preference relation is a correct selection (\cref{sec:selection}).

\section{Models of Pair Win Probabilities}
\label{sec:models}

The anytime Pareto set is defined by win probabilities over the instance distribution, which can only be estimated from finitely many algorithm runs.
We index the compared algorithms as $\mathbf{A} = \{A_1, \dots, A_n\}$ and write $p_{ij}(t)$ for the win probability of $A_i$ against $A_j$.
A model supplies a posterior distribution over $q_{ij}(t)$, its estimate of $p_{ij}(t)$, for every pair of compared algorithms and every budget of the grid, after any number of runs.
The win probability itself is defined without a model, and every truth in this paper is a value of $p$.
A model is \emph{exactly specified} when the data arise as it describes: the truth is a draw from its prior, and the observations follow its likelihood.

We consider two kinds of models: a \emph{pair model} gives every pair its own parameter $q_{ij}(t)$, and a \emph{rating model} gives every algorithm a scalar rating $\theta_i(t)$, with $q_{ij}(t)$ a function of $\theta_i(t)$ and $\theta_j(t)$.
\Cref{sec:per-pair,sec:dp} give the pair models, the beta-binomial (per-pair) model and the Dirichlet process posterior of \citet{benavoli2015} under two priors, and \cref{sec:bt,sec:pl} the rating models, Bradley--Terry and Plackett--Luce.
Each subsection ends with the conditions under which its model is exactly specified; \cref{sec:inference} gives how the posteriors are computed, and \cref{tab:models} compares the models by what they read, what they assume and what they offer for selection.

\subsection{The Per-Pair Model}
\label{sec:per-pair}

The simplest model is a beta-binomial model, which we also call the \emph{per-pair} model.
Every compared algorithm is run once on each instance, so that all pairs are compared on the same instances.
The model reads only the outcomes of each pair: at every budget, the wins $w$ and the losses $l$ of one algorithm against the other, with a tie counting as half a win and half a loss.
Every pair has its own win probability at every budget, with a uniform prior, so its posterior is $\mathrm{Beta}(1 + w, 1 + l)$, available analytically.
It assumes only that instances are exchangeable draws from $\Inst$, and a pair's outcomes then have the win probability $p_{ij}(t)$ itself.
Each pair is estimated from its own outcomes alone, so the estimates can form cycles, as the win probabilities can.

The model is approximate in two places, both on the cautious side.
Scoring a tie as half a win and half a loss treats the outcome as a draw of a win or a loss, but an outcome that can tie varies less than such a draw, so the posterior is wider than the data warrant.
Across budgets, the model treats the outcomes as independent, so the probability of a statement about all budgets is the product of the probabilities at each; the outcomes of one instance come from the same runs at every budget and are positively dependent, and under positive dependence the product is a lower bound.
Since the model has no value for individual algorithms, selection from the Pareto set rests on a lower bound (\cref{sec:selection}).
The model is exactly specified when instances are exchangeable, every pair's win probability at every budget is a draw from the uniform prior, and no ties occur; the guarantee then holds for every statement, one pair at a time.

\subsection{The Dirichlet Process: Tie-Point Prior and Near-Ignorance Prior}
\label{sec:dp}

\citet{benavoli2015} race algorithms on a posterior from a Dirichlet process over instances, and we use this posterior as a second pair model.
It reads the same outcomes as the per-pair model.
A Dirichlet process prior has a strength $s$ and a base measure, and the priors we use place the base measure on a single \emph{pseudo-instance}, whose outcome for every pair is set by the prior.
A draw from the posterior weights the pseudo-instance and the $N$ observed instances by $(\nu_0, \nu_1, \dots, \nu_N) \sim \mathrm{Dir}(s, 1, \dots, 1)$.
For a pair at a budget, $q_{ij}(t)$ is the weighted share of the instances on which $A_i$ is ahead of $A_j$, a tie counting one half.
The weights of one draw are shared by all pairs and all budgets, so the probability of a statement about all budgets is the share of draws in which the statement holds at every budget, where the per-pair model multiplies over budgets.
Every statement we consider concerns one pair, so the sharing across pairs plays no role here.
Probabilities are estimated as shares of 4{,}096 draws and therefore carry a Monte Carlo error.

The two priors differ in the outcomes of the pseudo-instance.
Under the \emph{tie-point prior}~\cite[Sec.~7]{benavoli2015}, it is a tie for every pair.
It adds to both sides alike, so $q_{ij}(t) > \half$ holds exactly when the weight of the wins exceeds that of the losses, and the probability that one algorithm beats the other does not depend on $s$; we therefore name the prior by its pseudo-instance, a point mass on a tie, rather than by its strength.
Under the \emph{near-ignorance prior}~\cite{benavoli2014}, the strength $s$ is fixed and the base measure ranges over all distributions, so the prior is a set of Dirichlet processes rather than one.
Before any instance, the lower and upper expectations of a win probability under this set are 0 and 1~\cite[Theorem~3]{benavoli2014}: the prior is uninformative about it, hence the name.
For a statement about one algorithm beating another, the pseudo-instance is a loss at every budget of the algorithm claimed to be ahead, and the probability of the statement is its lower probability over the set~\cite[Theorems~3 and~4]{benavoli2014}.
For an equivalence, the statement that a win probability lies within a region of practical equivalence around $\half$ (the rope, \cref{sec:relations}), a draw of the weights counts only if the win probability lies within the rope whatever the outcome of the pseudo-instance, so the probability of an equivalence is at most its lower probability.
We use $s = 1$, which halves the prior's imprecision at the first instance, the rule by which \citet[Sec.~3.1]{benavoli2014} choose the strength.
\Cref{tab:pseudo-counts} sets both beside the uniform prior of the per-pair model, for one pair at one budget.
Under the tie-point prior, any record without a loss gives the statement that one algorithm beats the other a probability of 1, even one win in eight instances, the minimum number of instances after which we compute a posterior.
The near-ignorance prior differs from the uniform prior in where the prior's count goes and in ignoring ties.

Both priors assume only exchangeable instances, and like the per-pair model they allow cycles.
The tie-point prior is exactly specified only for a population in which every pair ties, and every truth drawn from its prior is that population, so races on such truths cannot test its guarantee.
The near-ignorance prior has a guarantee of its own form, a bound that holds for every base measure of strength one (\cref{sec:guarantee}).
Read on a pair's \emph{win share}, the share of instances on which its first algorithm wins, a tie counting one half, the tie-point prior gives the Bayesian sign test of \citet[Sec.~7]{benavoli2015}; the near-ignorance prior was introduced for the Wilcoxon signed-rank statistic~\cite{benavoli2014}, and we read it on the win share as well.

\begin{table}[tbp]
	\caption[The three priors of the pair models]{The three priors of the pair models, for a pair $A_i$, $A_j$ at one budget with $w$ wins and $l$ losses of $A_i$, and the probability each gives to the statement that $A_i$ beats $A_j$, $P(q_{ij} > \half)$, after three records of eight instances.}
	\label{tab:pseudo-counts}
	\small
	\begin{tabular}{@{}lllccc@{}}
		\toprule
		& & & \multicolumn{3}{c}{$P(q_{ij} > \half)$, after} \\
		\cmidrule(l){4-6}
		Prior & Posterior & A tie counts as & 8 wins & 1 win, 7 ties & 4 wins, 4 ties \\
		\midrule
		Uniform (per-pair) & $\mathrm{Beta}(1 + w,\, 1 + l)$ & half a win, half a loss & 0.998 & 0.63 & 0.91 \\
		Tie-point prior & $\mathrm{Beta}(w,\, l)$ & nothing & 1 & 1 & 1 \\
		Near-ignorance prior & $\mathrm{Beta}(w,\, 1 + l)$ & nothing & 0.996 & 0.50 & 0.94 \\
		\bottomrule
	\end{tabular}
\end{table}

\subsection{Bradley--Terry on Pairwise Observations}
\label{sec:bt}

The Bradley--Terry model~\cite{bradley1952} reads one pair at a time: an instance is drawn, exactly two algorithms are run on it, and the outcome at each budget is a win, a loss or a tie, the tie counting as half a win and half a loss.
Every comparison uses a fresh instance.
Pairs read off a ranking of all algorithms on one instance are not used as separate comparisons: they come from the same runs, while the model assumes independent comparisons, and counted as such they would make the posterior narrower than the data warrant.
Every algorithm has a rating $\theta_i(t) > 0$ at every budget, the ratings of a budget lying on the simplex, and a pair's win probability is
\begin{equation*}
	q_{ij}(t) = \frac{\theta_i(t)}{\theta_i(t) + \theta_j(t)}.
\end{equation*}
The prior is uniform on the simplex at each budget, independently across budgets, so the posterior factorizes over budgets.

Through the ratings, a comparison of one pair informs every pair that involves either of its algorithms, so where the pairs do follow a common rating scale, the model decides sooner than the per-pair model.
It does so by compression: with twelve algorithms, eleven free ratings per budget stand for 66 win probabilities.
A rating scale implies strong stochastic transitivity: if $A_i$ beats $A_j$ and $A_j$ beats $A_k$, then $A_i$ beats $A_k$ with at least the larger of the two win probabilities, so cycles are excluded.
On the three-type example of \cref{sec:scale-free}, the best a rating scale can do is rate the three algorithms equally, which puts every $q_{ij}(t)$ at one half where every $p_{ij}(t)$ is $\tfrac{2}{3}$ or $\tfrac{1}{3}$.
The model is exactly specified when, at every budget, the pair outcomes are independent Bradley--Terry draws with ratings from the prior.
Across budgets, the outcomes of one comparison come from the same two runs and are positively dependent; as for the per-pair model, the product over budgets is then a lower bound.

\subsection{Plackett--Luce}
\label{sec:pl}

The Plackett--Luce model~\cite{plackett1975} extends the rating scale of Bradley--Terry to whole rankings.
It reads, for every instance and every budget, the ranking of all compared algorithms by their best-so-far values.
A ranking $\pi = (\pi_1, \dots, \pi_n)$, where $\pi_k$ is the algorithm in position $k$, arises by sequential choice: the first place is drawn with probabilities proportional to the ratings, the second from the remaining algorithms in the same way, and so on, so that
\begin{equation*}
	P\bigl(\pi \mid \theta(t)\bigr) = \prod_{k=1}^{n} \frac{\theta_{\pi_k}(t)}{\sum_{m=k}^{n} \theta_{\pi_m}(t)}.
\end{equation*}
Marginalized over the other algorithms, a pair keeps the win probability $q_{ij}(t) = \theta_i(t) / (\theta_i(t) + \theta_j(t))$ of Bradley--Terry, and $\theta_i(t)$ is the probability that $A_i$ ranks first.
Equivalently, every algorithm has a latent utility $\log \theta_i(t)$, each instance adds independent Gumbel noise to every utility, and the algorithms are ranked by the result.
The prior is that of Bradley--Terry, and like it the model cannot represent cycles.

Reading whole rankings makes the model efficient.
A run of an algorithm informs its standing against every other algorithm on the same instance, where under Bradley--Terry it informs one pair only.
With twelve algorithms, eight comparisons of every pair take 528 comparisons and 1{,}056 runs under Bradley--Terry, while eight ranked instances give every pair the same eight comparisons with 96 runs.
The pairs of one ranking come from the same runs, and the model accounts for this by describing the ranking as a whole.
The price is that every observation involves all compared algorithms, so the model must also assume how a ranking is built from them.
A pair's true win probability never depends on the other algorithms (\cref{sec:scale-free}), but how often $A_i$ rather than $A_j$ is the best of a set of algorithms can depend on what else the set holds; the model assumes that it does not, with odds $\theta_i(t)/\theta_j(t)$ for every set: the independence of irrelevant alternatives (IIA), which follows from Luce's choice axiom~\cite{luce1979}.
Under it, a ranking of a subset of the algorithms follows Plackett--Luce with the same ratings, so an algorithm can be dropped from later fits without changing what the remaining ratings mean.

Ties are handled within this assumption.
A group of tied algorithms contributes the average log-likelihood over its possible orderings~\cite{rojas-delgado2022}, which for two algorithms is exactly half a win each; we compute it by a second-order approximation instead of enumerating the orderings.
A tie parameter~\cite{turner2020} would model ties directly, but a pair's win probability would then depend on whether a third algorithm ties with either of them, which breaks the independence of irrelevant alternatives.
The model is exactly specified when, at every budget, the rankings of the instances are independent Plackett--Luce draws with ratings from the prior.
Across budgets, the rankings of one instance come from the same runs, and the product over budgets is a lower bound, as for the other models.

\subsection{Inference}
\label{sec:inference}

The posteriors of the pair models need no sampler: that of the per-pair model is analytical, and that of the Dirichlet process is read from independent draws of the instance weights.
The rating models are fitted in two stages.
The first is Pathfinder~\cite{zhang2022a}, a fast approximation that follows several optimization paths towards the posterior mode and merges the Gaussian approximations along them by Pareto-smoothed importance resampling.
We accept its approximation when the Pareto $\hat{k}$ of that resampling, the estimated shape of a generalized Pareto distribution fitted to the largest importance ratios, is at most 0.7~\cite[Sec.~2.2]{vehtari2024}.
Otherwise, the second stage samples the posterior by Markov chain Monte Carlo (MCMC), with the No-U-Turn sampler of Hamiltonian Monte Carlo~\cite{hoffman2014}, which follows posterior shapes that the approximation cannot, at several times the cost.
We accept its draws when $\hat{R}$ is at most 1.01, the bulk and tail effective sample sizes are at least 400~\cite{vehtari2021}, the Bayesian fraction of missing information is at least 0.3, and no transition diverges or reaches the maximum tree depth.
A fit that is not accepted is repeated with doubled effort, more paths and iterations for Pathfinder and more draws and warm-up for MCMC; after three failures it moves to MCMC, and after three more the computation stops with an error.
The likelihood is evaluated in log space, as a sum over rankings, since its product of probabilities underflows as algorithms and rankings grow.
Every probability estimated from draws carries a Monte Carlo standard error, which is zero for the per-pair model.

\begin{table}[tbp]
	\caption[The five models compared]{The five models by what they read, assume and supply. The table holds definitions and assumptions only; what an assumption buys or costs is measured in \cref{sec:validation,sec:replays}.}
	\label{tab:models}
	\footnotesize
	\setlength{\tabcolsep}{4pt}
	\renewcommand{\arraystretch}{1.25}
	\begin{tabular}{@{}L{44pt}L{50pt}L{60pt}L{55pt}L{66pt}L{75pt}@{}}
		\toprule
		& Tie-point prior & Near-ignorance prior & Per-pair & Bradley--Terry & Plackett--Luce \\
		\midrule
		Kind & \multicolumn{3}{L{181pt}}{pair model} & \multicolumn{2}{L{149pt}@{}}{rating model} \\
		Reads & \multicolumn{3}{L{181pt}}{each pair's outcomes on shared instances} & a pair per instance & whole rankings \\
		Assumes & \multicolumn{3}{L{181pt}}{exchangeable instances} & \multicolumn{2}{L{149pt}@{}}{one rating scale; Plackett--Luce also IIA} \\
		Cycles & \multicolumn{3}{L{181pt}}{allowed} & \multicolumn{2}{L{149pt}@{}}{excluded} \\
		Prior & all pairs tie & against the claim & uniform per pair & \multicolumn{2}{L{149pt}@{}}{ratings uniform on the simplex} \\
		Ties & \multicolumn{2}{L{118pt}}{ignored} & \multicolumn{2}{L{129pt}}{half a win each} & average over orderings \\
		Over budgets & \multicolumn{2}{L{118pt}}{weights shared by all budgets} & \multicolumn{3}{L{212pt}@{}}{product over budgets} \\
		Guarantee & none & any base measure & \multicolumn{3}{L{212pt}@{}}{when exactly specified} \\
		Selection & \multicolumn{3}{L{181pt}}{lower bound} & \multicolumn{2}{L{149pt}@{}}{posterior probability} \\
		Fit & \multicolumn{2}{L{118pt}}{draws of instance weights} & closed form & \multicolumn{2}{L{149pt}@{}}{fast approximation, then MCMC} \\
		\bottomrule
	\end{tabular}
\end{table}

\section{\textsc{PolarBear}: Pareto-Optimal Anytime Algorithms via Bayesian Racing}
\label{sec:polarbear}

\begin{figure}[t]
	\centering
	\input{figures/v2/overview}
	\caption[One round of \textsc{PolarBear}]{One round of \textsc{PolarBear}, with schematic values.
		(1) The race runs the active algorithms on $\Delta N$ new instances and
		records, per instance and budget, only their order, best first.
		(2) The model supplies the posterior of every pair's win probability at
		every budget. At one budget, the posterior mass above $\half$ and
		the mass within the rope give the pair's relation.
		(3) An algorithm that another dominates, under the test
		over budgets (\cref{sec:tests}), is eliminated. A pair is resolved by a
		crossing or by equivalence at every budget, and the other pairs stay open.
		(4) An algorithm with an open pair runs on the next $\Delta N$ instances.
		Under the pointwise test it runs up to the last budget $h$ at which one of its open pairs holds no relation, and
		under the joint test to the last budget. The allocation policy chooses
		$\Delta N$.
		The race stops when every pair is resolved, and returns the algorithms
		not eliminated with the posterior.}
	\Description{A cycle of four panels, Sample, Fit, Decide and Allocate, with
		an exit from Decide to a box that names what the race returns. Sample shows best-so-far trajectories of four
		algorithms on one instance and a table of their order at five budgets on
		every instance. Fit shows posterior intervals of three pairwise win
		probabilities over the budgets and the three posterior densities at one
		budget. Decide shows a table of relations per pair and budget with each
		pair's outcome, and a graph of the pairs. Allocate shows up to which
		budget each algorithm runs next under either test, and the expected decisions per cost
		against the number of instances until the next round.}
	\label{fig:overview}
\end{figure}

The anytime Pareto set is defined by win probabilities that the models of \cref{sec:models} only estimate, so after any number of instances it is uncertain which algorithms belong to it.
A design with a fixed number of instances spends runs on comparisons that are already settled: an algorithm that is clearly dominated after a few instances is run on all of them, and so is a pair that clearly crosses.
\textsc{PolarBear} instead decides after each batch of instances what the posterior already supports, and runs only the algorithms still in question on further instances: it drops an algorithm once the posterior shows that another dominates it, and stops comparing two algorithms, keeping both, once it shows that they cross or are practically equivalent.
No sample size is fixed in advance; the race ends when every pair is resolved in one of these ways.

It takes a set of algorithms $\mathbf{A}$, an instance distribution $\Inst$, the budget grid $\Tg$ and a model from \cref{sec:models}.
It also takes a decision threshold $\alpha$, the posterior probability that a statement must reach, and a region of practical equivalence, or rope, $\epsilon$~\cite{kruschke2018}, the distance from $\half$ within which a win probability is not decision-relevant.
It works in rounds, shown in \cref{fig:overview}: it runs the algorithms still in the race on new instances, fits the model to all rankings observed so far, decides what the posterior supports, and chooses how many instances to run before the next round.
It returns the algorithms it has not eliminated, the set $\POhat$, as its estimate of $\PO$.
Like the algorithms it compares, \textsc{PolarBear} can itself be stopped after any round, and it then returns the algorithms not eliminated so far along with the current posterior.

\Cref{sec:relations} gives the statements the race reads from the posterior at each budget, and \cref{sec:tests} how these eliminate an algorithm or resolve a pair.
\Cref{sec:procedure} gives the procedure in full, \cref{sec:guarantee} what its statements guarantee, and \cref{sec:selection} how a single algorithm is selected from the returned set given a preference over budgets.

\subsection{Relations}
\label{sec:relations}

Let $\Rk$ be the rankings observed so far.
In each round, a pair of algorithms $\{A, B\}$ holds one relation $r_{AB}(t)$ at each budget $t$, read from the posterior of $q_{AB}(t)$:
\begin{equation}
	r_{AB}(t) =
	\begin{cases}
		{\succ} & \text{if } P(q_{AB}(t) > \half \mid \Rk) \geq \alpha, \\
		{\prec} & \text{if } P(q_{AB}(t) < \half \mid \Rk) \geq \alpha, \\
		{\approx} & \text{otherwise, if } P(|q_{AB}(t) - \half| \leq \epsilon \mid \Rk) \geq \alpha, \\
		\norel & \text{otherwise.}
	\end{cases}
	\label{eq:relations}
\end{equation}
The first states that $A$ beats $B$ at $t$, a \emph{directional} statement, which is false if $p_{AB}(t) \leq \half$; the second is the same for $B$ over $A$.
$A$ dominates $B$ in the sense of \cref{sec:scale-free} exactly when the first is true at every budget.
The third is an \emph{equivalence} at $t$, the statement that $p_{AB}(t)$ lies within the rope of $\half$, which is false if $|p_{AB}(t) - \half| > \epsilon$.
Where both a directional statement and the equivalence pass, the relation is the directional one, and where neither passes, the pair holds no relation at $t$.
A relation belongs to the round whose posterior supports it and can lapse in the next, when new instances move the posterior.
Every relation is recorded with the posterior probability in the round in which it is declared, its \emph{recorded probability}, which is at least $\alpha$.

The threshold $\alpha$ is 0.99 throughout the paper, and the rope $\epsilon$ is 0.05, so that win probabilities between 0.45 and 0.55 count as negligible.
A wider rope lets more pairs be declared equivalent and ends races sooner, while a narrower one distinguishes finer differences at the cost of more instances.
Where a probability in \cref{eq:relations} is estimated from draws, a relation passes only when the estimate exceeds the threshold by $z$ times its Monte Carlo standard error, taken as the error the estimate would have if the probability were exactly at the threshold.
A probability exactly at the threshold then passes by Monte Carlo error alone on a share of about $\Phi(-z)$ of rounds; $z$ is 2 throughout.
The margin covers Monte Carlo error only; the error of an approximate posterior is what the diagnostic checks of \cref{sec:inference} are for.

\subsection{Elimination and Resolution}
\label{sec:tests}
\label{sec:resolution}

The race drops an algorithm, an \emph{elimination}, once the posterior supports that another dominates it, a statement about all budgets at once; a \emph{test over budgets} decides when it does.
Under the \emph{joint test}, $A$ eliminates $B$ when
\begin{equation}
	P\bigl(q_{AB}(t) > \half \text{ for all } t \in \Tg \mid \Rk\bigr) \geq \alpha,
	\label{eq:joint}
\end{equation}
with the Monte Carlo margin where the probability is estimated from draws, and this probability is the elimination's recorded probability.
Each model computes it as described in \cref{sec:models}, as the share of draws in which $A$ beats $B$ at every budget, or as the product of the per-budget probabilities $P_t = P(q_{AB}(t) > \half \mid \Rk)$.
The product treats the budgets as independent; shared instances make them positively dependent, and the joint probability then lies between the product and the smallest $P_t$, so a product that reaches $\alpha$ means a joint probability of at least $\alpha$.

The \emph{pointwise test} is a cheaper heuristic: $A$ eliminates $B$ when the directional statement for $A$ over $B$ of \cref{eq:relations} passes at every budget, each budget on its own.
Its condition is that the smallest $P_t$ reaches $\alpha$, and the smallest $P_t$ equals the joint probability only if the budgets are perfectly dependent; under any other dependence it is larger.
Its recorded probability is therefore the union bound $1 - \sum_{t \in \Tg} (1 - P_t)$, a lower bound on the joint probability under any dependence that says little on long grids.
For example, under the per-pair model, a pair with 8 wins in 8 instances at each of 30 budgets has $P_t = 0.998$ at every budget, so the pointwise test eliminates, with a recorded probability of 0.941, while the joint probability is 0.943 and the joint test does not.
In general, the joint test at a threshold implies the pointwise test at that threshold, while the pointwise test at $\alpha$ implies the joint test only at $1 - T(1 - \alpha)$ for $T$ budgets.
The two tests share the relations at single budgets and differ only in eliminations, and through them in how far the algorithms run.
Under the pointwise test, an algorithm runs only up to the largest budget at which one of its open pairs holds no relation, whereas under the joint test every algorithm still in the race runs to the full budget (\cref{sec:procedure}).
The joint test also needs more evidence for a pair that is close at many budgets at once, since the product over these budgets, not each budget alone, must reach the threshold.

A pair is \emph{resolved} when it needs no more data: when one of its algorithms is eliminated, or, with both staying in the race, in two ways.
The pair is equivalent at every budget, or the race declares a \emph{crossing}: in one round the pair holds two different relations at two budgets, neither of them the absence of a relation, so that a directional statement holds for $A$ at one budget and for $B$ at another, or a directional statement at one budget and an equivalence at another.
An equivalence counts toward a crossing by design: without it, a pair whose win probability lies inside the rope at some budget could be decided only by a directional statement there, and such a statement needs arbitrarily many instances as the win probability approaches $\half$.
A crossing rests on two relations, however many budgets the grid has, so unlike a statement about every budget it needs no correction over budgets; both relations must hold in the same round.
Either way, if the relations are true, neither algorithm beats the other by more than the rope at every budget, so the weaker algorithm of a resolved pair may stay in the race.
A crossing therefore resolves the pair as soon as one round shows it, even if other budgets hold no relation yet.
Waiting instead until every budget holds a relation would leave selection a posterior decided at every budget, at the cost of more instances; in this paper we resolve the pair once the race declares a crossing.
A pair that is not resolved is \emph{open}.
Under the joint test, a pair can also hold a directional statement for the same algorithm at every budget and still be open: it is not eliminated until the joint probability reaches the threshold, and not resolved, since it shows neither a crossing nor equivalence, so it keeps drawing data.

The eliminations that pass in one round can form a cycle, each algorithm dominated by the next.
An elimination on such a cycle does not take effect: its pair is resolved, both algorithms stay, and the race reports the cycle.
Every other elimination of that round takes effect, and since these form no cycle, some algorithm is not eliminated, so $\POhat$ is never empty, even where $\PO$ is (\cref{sec:scale-free}).
Under a rating model, every posterior draw orders the algorithms at each budget, so in every draw at least one elimination of a cycle fails; the probabilities that the $k \leq n$ eliminations of a cycle fail then sum to at least one, and one of them passes at most at $1 - 1/n$, so no cycle passes while $\alpha > 1 - 1/n$.
The pair models do not exclude a cycle.

Relations may change with more data, but eliminations and resolutions are final.
An eliminated algorithm runs on no further instances, so no later data bear on its elimination; the two algorithms of a resolved pair can keep running for their other open pairs, and a later posterior need no longer show the crossing or the equivalence that resolved it.
The race returns the algorithms it never eliminated, the set $\POhat$.
If every statement it relied on is true, $\POhat$ holds every member of $\PO$, since a member of $\PO$ can leave only through a false elimination.
It can also hold algorithms outside $\PO$: a pair resolved by a crossing or by equivalence keeps its weaker algorithm, and so does a pair resolved because its stronger algorithm was eliminated by a third.
Which of these algorithms $\POhat$ holds depends on the order in which pairs are resolved.

What stays bounded is the difference within the set: if every statement is true and the truth holds no cycle of dominance, no member of $\POhat$ beats another member by more than the rope at every budget, since such a pair cannot be resolved while both stay.
Let $\POeps$ be the Pareto set widened by the rope, the algorithms that no other beats by more than the rope at every budget.
If the truth is moreover strongly stochastically transitive at every budget, as a truth given by a rating scale is (\cref{sec:bt}), $\POhat$ also lies within $\POeps$.
Suppose an algorithm $B$ beats a member $C$ of $\POhat$ by more than the rope at every budget.
By the argument above, $B$ is not in $\POhat$, so another algorithm eliminated it, and by strong stochastic transitivity this eliminator beats $C$ by more than the rope as well.
Since the eliminations form no cycle, following the eliminators must end at an algorithm $D$ in $\POhat$ that beats $C$ by more than the rope at every budget, and again such a pair cannot be resolved while both stay.
Hence no algorithm beats a member of $\POhat$ by more than the rope at every budget, which is what membership in $\POeps$ requires.

Measured by regret, the loss from selecting an algorithm worse than the best under the user's preference, the race reduces only the part that data can reduce before the preference is known, the uncertainty about which algorithms are dominated.
Between two algorithms that cross, which one is better depends on the preference, so this part of the regret is left to selection (\cref{sec:selection}), as is the pair's remaining uncertainty at budgets the race did not need to decide.

We illustrate elimination and resolution on a designed race, the \emph{running example}, to which \cref{sec:procedure,sec:selection} return.
Five algorithms, $A$ to $E$, are compared at ten budgets, $t = 1, \dots, 10$.
Their rankings on each instance are drawn from a Plackett--Luce model, independently at every budget, and the race uses the Plackett--Luce model of \cref{sec:pl}, so the data follow the model's likelihood; it is the best case for the procedure.
In the truth, $A$ and $B$ cross: $A$ beats $B$ with probability 0.75 at the first budget, falling to 0.25 at the last.
$C$ is clearly dominated: every other algorithm beats it with probability at least 0.83 at every budget.
$D$ is narrowly dominated: $B$ beats it with probability 0.6 at every budget.
$E$ is equivalent to $A$ within the rope: $A$ beats it with probability 0.502 at every budget, so $A$ dominates $E$, but by less than the rope.
The anytime Pareto set is therefore $\PO = \{A, B\}$, and the widened Pareto set $\POeps = \{A, B, E\}$.

\Cref{fig:example-race} shows one of 20 races on the running example, together with the truth.
All 20 races finish and return $A$, $B$ and $E$, so each holds the anytime Pareto set and lies within the widened one.
$C$ is eliminated after 29 to 38 instances, $D$ after 291 to 812, and the races end after 555 to 2{,}088.
$E$ stays because its pair with $A$ is resolved without an elimination: 9 of the 20 races declare the two equivalent at every budget, and the other 11 a crossing of a directional statement and an equivalence.
In the race shown, $C$ leaves after 29 instances and $D$ after 729, and the race ends after 1{,}837 instances, when it declares $A$ and $E$ equivalent at every budget.

\begin{figure}[tbp]
	\centering
	\includegraphics[width=\textwidth]{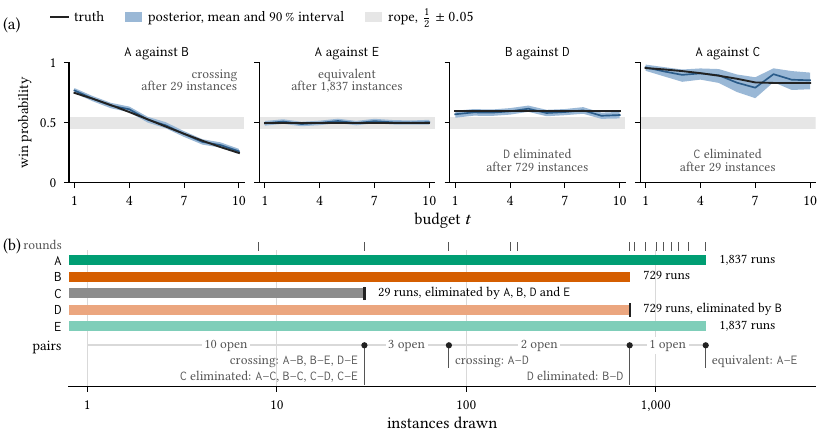}
	\caption[The running example's race]{One race of the running example against its truth. (a) Four pairs that cover the outcomes a pair can have: a crossing, an equivalence, and the elimination of the narrowly dominated $D$ and of the clearly dominated $C$. A panel shows the pair's true win probability, the rope, and the posterior at the end of the race, with the outcome the race declared and the number of instances after which it did. The interval is widest for the pair with $C$, which ran on 29 instances only. (b) The course of the race over the instances drawn, on a logarithmic axis. A bar spans the instances an algorithm ran on; it ends with a mark where the algorithm was eliminated, and without one where its last pair was resolved. The ticks above the bars mark the rounds.}
	\Description{Four panels of win probability over ten budgets, each with the truth, the posterior and the rope, and below them one bar per algorithm for its runs and one row for the pairs, which marks the instances at which pairs were resolved, names those pairs and how each was resolved, and states how many pairs are open in between.}
	\label{fig:example-race}
\end{figure}

\subsection{The Racing Procedure}
\label{sec:procedure}

\begin{algorithm}[t]
	\caption{\textsc{PolarBear}}
	\label{alg:polarbear}
	\KwIn{algorithms $\mathbf{A}$, instance distribution $\Inst$, budget grid $\Tg = (t_1 < \dots < t_T)$, model $\Model$, test (pointwise or joint), decision threshold $\alpha$, rope $\epsilon$, initial design $N_0$, allocation policy}
	\KwOut{returned set $\POhat$, posterior $P(q \mid \Rk)$}

	$\POhat \gets \mathbf{A}$;\quad $\Open \gets \binom{\mathbf{A}}{2}$;\quad $\Rk \gets \emptyset$ \tcp*{$\Open$ holds the unresolved pairs}
	$\Delta N \gets N_0$;\quad $h_A \gets t_T$ for every $A \in \mathbf{A}$ \tcp*{the initial design}

	\Repeat{$\Open = \emptyset$}{
		\tcp{Sample}
		$\Active \gets \{A \in \POhat : \{A, B\} \in \Open \text{ for some } B\}$ \tcp*{the active algorithms}
		Draw $\Delta N$ instances from $\Inst$ and run every $A \in \Active$ on each of them up to budget $h_A$\;
		For every new instance and every $t \in \Tg$, add to $\Rk$ the ranking of $\{A \in \Active : h_A \geq t\}$ at $t$\;
		\tcp{Fit}
		Compute $P(q \mid \Rk)$ under $\Model$, repeating a fit that fails its diagnostics\;
		\tcp{Relations in this round, with the Monte Carlo margin}
		\ForEach{$\{A, B\} \in \Open$ and $t \in \Tg$}{
			\lIf{$P(q_{AB}(t) > \half \mid \Rk) \geq \alpha$}{$r_{AB}(t) \gets {\succ}$}
			\lElseIf{$P(q_{AB}(t) < \half \mid \Rk) \geq \alpha$}{$r_{AB}(t) \gets {\prec}$}
			\lElseIf{$P(|q_{AB}(t) - \half| \leq \epsilon \mid \Rk) \geq \alpha$}{$r_{AB}(t) \gets {\approx}$}
			\lElse{$r_{AB}(t) \gets \norel$}
		}
		\tcp{Elimination: dominance}
		$\mathbf{E} \gets \emptyset$ \tcp*{$(A, B) \in \mathbf{E}$: $A$ would eliminate $B$}
		\ForEach{$(A, B)$ with $\{A, B\} \in \Open$}{
			\lIf{the test is pointwise and $r_{AB}(t) = {\succ}$ for every $t \in \Tg$}{$\mathbf{E} \gets \mathbf{E} \cup \{(A, B)\}$}
			\lIf{the test is joint and $P(q_{AB}(t) > \half \;\, \forall t \in \Tg \mid \Rk) \geq \alpha$}{$\mathbf{E} \gets \mathbf{E} \cup \{(A, B)\}$}
		}
		\tcp{An elimination on a cycle does not take effect; its pair is resolved}
		$\mathbf{C} \gets \{(A, B) \in \mathbf{E} : (A, B) \text{ lies on a cycle of } \mathbf{E}\}$;\quad $\Open \gets \Open \setminus \{\{A, B\} : (A, B) \in \mathbf{C}\}$\;
		$\POhat \gets \POhat \setminus \{B : (A, B) \in \mathbf{E} \setminus \mathbf{C}\}$;\quad $\Open \gets \{u \in \Open : u \subseteq \POhat\}$\;
		\tcp{Resolution: a crossing, or equivalence at every budget}
		\ForEach{$\{A, B\} \in \Open$}{
			\If{$r_{AB}$ takes two of the values $\succ, \prec, \approx$ on $\Tg$, or $r_{AB}(t) = {\approx}$ for every $t \in \Tg$}{
				$\Open \gets \Open \setminus \{\{A, B\}\}$\;
			}
		}
		\tcp{Allocation}
		\If{the test is pointwise}{
			$h_A \gets \max \, \{t \in \Tg : r_{AB}(t) = \norel \text{ for some } \{A, B\} \in \Open\}$ for every $A$ in a pair of $\Open$\;
		}
		$\Delta N \gets$ the number of instances until the next round, chosen by the allocation policy\;
	}

	\Return{$\POhat$, $P(q \mid \Rk)$}
\end{algorithm}

\Cref{alg:polarbear} states \textsc{PolarBear} as pseudocode.
A \emph{block} is one instance drawn from $\Inst$ and the runs of the algorithms on it.
The race starts with an initial design of $N_0$ blocks, run by every algorithm to the full budget.
Each round fits the model to all rankings so far, reads the relations of the open pairs, makes the eliminations and resolutions they support, and chooses the number of blocks $\Delta N$ before the next round; the race ends when no pair is open.
An algorithm is \emph{active} while one of its pairs is open, and only active algorithms run on new blocks, so an algorithm whose pairs are all resolved stays in $\POhat$ without being run again, as $B$ after 729 instances in the race of \cref{fig:example-race}.
A fit that fails the diagnostic checks of \cref{sec:inference}, or whose Monte Carlo error is too large for any probability to pass, with $z$ times the error exceeding $1 - \alpha$, decides nothing: no relation holds, no pair is resolved and no algorithm is eliminated, and the race refits with more effort.
A failed fit thus costs computation but never produces a statement.

The allocation policy chooses $\Delta N$, the number of blocks before the next round.
Ours takes the smallest $\Delta N$ that maximizes the expected number of decisions per unit of cost,
\begin{equation}
	\frac{D(\Delta N)}{\kappa + C\,\Delta N},
	\label{eq:allocation}
\end{equation}
where $D(\Delta N)$ is the expected number of decisions after $\Delta N$ more blocks, $C$ the number of evaluations of one block, with each active algorithm run as far as it currently runs, and $\kappa$ a fixed cost of the round in evaluations, what is lost each time the race pauses to decide.
A decision is a relation for an open pair at a budget or, under the joint test, the joint probability of a pair reaching the threshold.
$D$ is predicted from the current posterior: the posterior of each pair at each budget is approximated as normal, its variance shrinks in proportion to the observations it rests on, and a decision is counted where the shrunken posterior would pass the threshold with the Monte Carlo margin.
For a sampled posterior, $\kappa$ is the evaluations that could have run in the compute time of one fit, and the policy trades evaluations against rounds.
For the pair models, whose posteriors cost almost nothing to compute, $\kappa$ only sets how many blocks are issued together, which lets runs execute in parallel on more workers than there are active algorithms; a policy tuned for such posteriors is future work (\cref{sec:conclusion}).
A simpler policy doubles $\Delta N$ when a round adds no decision and halves it when it adds many; on 200 paired races on truths drawn from the Plackett--Luce prior, with 4 algorithms and 2 budgets under the joint test, it took 664 blocks and 24.4 rounds on average, against 648 and 16.4 for ours.

Each active algorithm $A$ runs on a new block up to a horizon $h_A$, and the rankings at a budget hold only the algorithms that ran that far.
Under the pointwise test, $h_A$ is the largest budget at which one of $A$'s open pairs holds no relation: every decision of that test reads relations at single budgets, and beyond $h_A$ every open pair of $A$ already holds one, so only the budgets up to $h_A$ still need data.
Under the joint test, every active algorithm runs to the full budget.
A pair with a directional statement for the same algorithm at every budget stays open until its joint probability reaches the threshold, so a budget that already holds a relation may still need data, and runs stopped short there could leave the race unable to end.

The race ends when no pair is open, and it has then \emph{finished}.
No sample size is fixed in advance, and the procedure has no cap of its own; whoever runs a race may stop it earlier, after any round.
A race ends because with enough instances every win probability supports a relation at every budget, a directional statement away from $\half$ and an equivalence inside the rope, and since an equivalence counts toward a crossing, no pair has to establish a directional statement inside the rope.
The slowest pair has a win probability of $\half + \epsilon/2$, as far from a directional statement as from an equivalence.

With a single budget, $T = 1$, the joint and pointwise tests coincide, no pair can cross, and a pair is resolved only by equivalence or by an elimination.
The anytime Pareto set is then the set of algorithms that no other beats, and the procedure is a Bayesian race at that budget with a rope, in the line of F-race~\cite{birattari2010} and of the race of \citet{benavoli2015}.

Since every decision concerns a pair, the procedure allows algorithms to join a race under way: a newcomer opens its pairs with the algorithms still in the race, and the race continues.
With eliminations final, $\POhat$ may keep a newcomer that only an eliminated algorithm dominates, as it may keep the weaker algorithm of a pair whose stronger one a third eliminated; the properties of \cref{sec:resolution} then hold relative to every algorithm the race has held.
Alternatively, an eliminated algorithm can be kept without further runs and resumed when a newcomer joins, to open its pair with the newcomer; this matters where the estimand is the Smith set of a preference over all algorithms rather than the Pareto set, since a newcomer can form a cycle with a member of $\POhat$ and an algorithm that member eliminated.
Under the pair models, the posteriors of the other pairs are unchanged, while under the rating models, which are fitted to all algorithms jointly, they change.
A newcomer that catches up on the instances drawn before it joined is flattered if it was chosen on the race's own data, since it then runs on the instances it was chosen by.
The races of this paper add no algorithm; an automated algorithm design loop that proposes algorithms and adds them to a running race, with membership in the returned set as its meta-objective, is future work (\cref{sec:conclusion}).

\subsection{The Guarantee}
\label{sec:guarantee}

Let $X$ be the number of false statements a race makes (directional statements, equivalences, or eliminations), each counted in the round in which it is first declared, and $\Lambda$ the sum of one minus their recorded probabilities.
Where these are posterior probabilities and not lower bounds on them, $\Lambda$ is the posterior expected number of false statements, the quantity that Bayesian approaches to multiple testing control in place of the number of false discoveries~\cite{newton2004,muller2004a}.
If the model is exactly specified, then on average over the prior, $\mathbb{E}[X] \leq \mathbb{E}[\Lambda]$, for any stopping rule and any allocation.
This is a known property of Bayesian inference under optional stopping~\cite{rouder2014,hendriksen2021}: a statement is false with posterior probability one minus its recorded probability, given the data at the moment it is declared, and the stopping rule and the allocation depend only on data already seen, so they leave the posterior unchanged.
Under the joint test the inequality is nearly an equality; the gap comes from the Monte Carlo margin, and under the pointwise test also from its recorded probability being a lower bound.
For an elimination, the recorded probability is the joint probability of \cref{eq:joint}: a model that factorizes over budgets is exactly specified only if outcomes are independent across budgets, and then the product is exact, while under the positive dependence of real runs the model is no longer exactly specified, and the product errs on the conservative side: eliminations pass later, and one minus their recorded probability overstates their chance of being false (\cref{sec:tests}).

The bound holds for each kind of statement separately: directional statements, equivalences and eliminations.
Since the probability of at least one false statement is at most their expected number, it also bounds the probability that a race makes any false statement of a kind.
For example, a race among $n$ algorithms makes at most $n - 1$ eliminations, since $\POhat$ is never empty, and under the joint test each has a recorded probability of at least $\alpha$, so the expected number of false eliminations is at most $(n - 1)(1 - \alpha)$, which is 0.11 for twelve algorithms at $\alpha = 0.99$.

The guarantee says nothing about a fixed truth: for a truth near a tie, statements can be false more often than their recorded probabilities say, since the average over the prior is essential and a correct likelihood alone is not enough~\cite{deheide2020}.
Nor does it hold when the model is not exactly specified, and it says nothing about cost or how soon a race ends.
On real algorithm runs both conditions fail, since the truth is one fixed population of instances and no model can be assumed exactly specified, so there the errors have to be measured against the truth.
For the near-ignorance prior the guarantee takes the form of a bound: the probability it gives a statement is no larger than the posterior probability under any base measure of strength one, so if the truth is drawn from a Dirichlet process with any such base measure, $\mathbb{E}[X] \leq \mathbb{E}[\Lambda]$ still holds, for all of these base measures at once, though not as an equality in general.
For the tie-point prior, the guarantee covers only a population in which every pair ties, so races on truths drawn from its prior cannot test it (\cref{sec:dp}).
We report $X/\Lambda$ on truths drawn from each model's prior, where it should be at most 1 and near 1 where the guarantee is nearly an equality (\cref{sec:prior-truths}), and on recorded algorithm runs (\cref{sec:five-models}).

\subsection{Selection from the Anytime Pareto Set}
\label{sec:selection}

The race returns a set, and a user who knows which budgets matter selects from it under a preference, a weight $w_t$ for every budget (\cref{sec:scale-free}).
We consider three ways to select.
The first, \emph{selection on win probabilities}, applies to every model: $A_i$ is preferred to $A_j$ when $\bar{q}_{ij} = \sum_{t \in \Tg} w_t\, q_{ij}(t) > \half$, and we select the member of $\POhat$ with the highest posterior probability of being preferred to every other member.
A rating model gives this probability as the share of draws in which $A_i$ is preferred to all others.
A pair model gives the posterior of each pair separately, and the union bound turns it into a lower bound, $L_i = 1 - \sum_{j \neq i} P(\bar{q}_{ij} < \half \mid \Rk)$.
In the truth, where no single member of $\PO$ is preferred to every other, the answer is the Smith set of the preference relation among the Pareto set (\cref{sec:scale-free}).

The other two ways apply to the rating models and select on a value per algorithm.
In \emph{selection on utilities}, the value of $A_i$ is $\sum_{t \in \Tg} w_t \log\theta_i(t)$, the weighted sum of its utilities, the logarithms of its ratings, centered at each budget; two values differ by the weighted log-odds of one algorithm beating the other, where the centering cancels, so their comparison does not depend on the other candidates.
In \emph{selection on ratings}, the value of $A_i$ is $\sum_{t \in \Tg} w_t\, \theta_i(t)$, with the ratings normalized over the candidates, which under Plackett--Luce is the probability that $A_i$ ranks first among them at a budget drawn from the preference; it depends on the candidate set, and adding a candidate can reverse the order of two others.
The two can select differently for two reasons.
Selection on utilities averages logarithms and favors an algorithm that is never far behind, while selection on ratings averages probabilities and favors one that often ranks first.
And in selection on ratings, near-copies share their first places, so each has a lower value than it would have alone.

In either, we select the candidate whose value has the highest posterior mean, or alternatively the candidate most likely to have the highest value.
The latter rests on a statement about the model's rating scale, which the guarantee of \cref{sec:guarantee} does not cover.
It is also hard to calibrate: which candidate has the highest value changes abruptly where two candidates tie, and for a function that is not differentiable there, a credible statement cannot be presumed to hold at its stated frequency~\cite{kitagawa2020}, nor is the bootstrap consistent~\cite{fang2019}.
Where two candidates are close, the probability of the highest value can therefore lie far from the frequency of a correct selection.
For the same reason we make no statement about membership in a Smith set, although the draws of the Dirichlet process, whose weights are shared by all pairs, would give a posterior probability of it.

A member of $\PO$ has, against every rival, a preference under which that rival is not preferred to it: the preference with all its weight on a budget where the rival does not beat it.
One preference need not serve against all rivals at once, so a member can be preferred to every other member under no preference at all, and likewise, on utilities or on ratings, a member can have the highest value under no preference.
Of the three ways, only selection on win probabilities has a truth that needs no model, the win probabilities themselves, so we score all three against that truth (\cref{sec:selection-results}).

\begin{figure}[tbp]
	\centering
	\includegraphics[width=\textwidth]{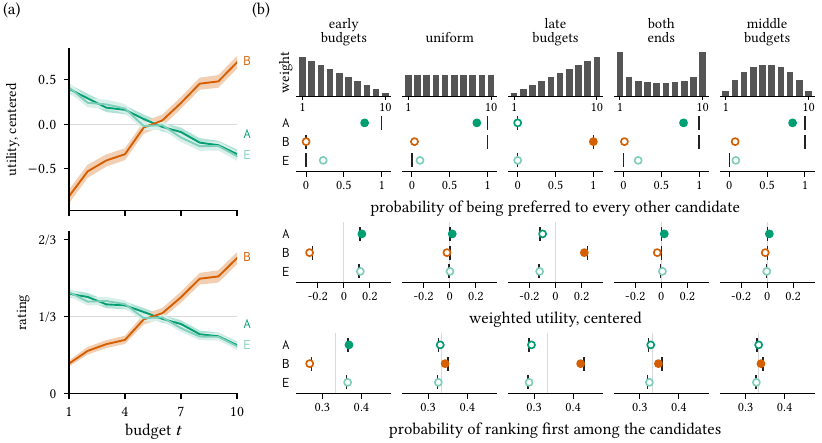}
	\caption[Selection in the running example]{Selection in the running example, from the set $\{A, B, E\}$ that the race of \cref{fig:example-race} returns. (a) The posterior of the three algorithms' utilities, centered at each budget, and of their ratings, as mean and 90\,\% interval. (b) Five preferences, each shown by its weights over the budgets, and below each the three ways to select. The first row gives each algorithm's posterior probability of being preferred to every other, and the second and third its value in selection on utilities and on ratings, as posterior mean with 90\,\% interval. The filled mark is the selected algorithm. A tick marks the truth: in the second and third rows the true value, and in the first row one for a member of the truth's Smith set and zero otherwise. The gray line is the value of three equal algorithms.}
	\Description{Two panels of utilities and ratings over ten budgets for three algorithms, and five columns, one per preference, each with the preference's weights, the three algorithms' probabilities of being preferred to every other, and their values on utilities and on ratings.}
	\label{fig:example-selection}
\end{figure}

\Cref{fig:example-selection} selects from the set that the race of \cref{fig:example-race} returns, $\{A, B, E\}$, in all three ways, under five preferences that show that a preference can be any weighting of the budgets: weights falling over the budgets, uniform weights, rising weights, weights highest at both ends, and weights highest in the middle.
In all three ways, the truth favors $A$ under falling weights and $B$ under rising weights.
Under the three preferences that weigh early and late budgets alike, $A$ and $B$ tie in truth on win probabilities and on utilities, since their crossing is symmetric over the budgets by design, while on ratings $B$ is highest, at 0.350 against 0.326 under the uniform budget preference, 0.357 against 0.323 under weights at both ends, and 0.344 against 0.329 under weights in the middle.
The reason is the candidate set: $E$, a near-copy of $A$, takes about half of $A$'s first places, which $A$ holds mostly at short budgets.
On win probabilities, no member is then preferred to every other, so the truth's selection is the Smith set $\{A, B\}$, whose members are both preferred to $E$ and neither to the other.
The race shown selects $A$, with a probability of 0.86, 0.79 and 0.84 of being preferred to every other member; since $A$ and $B$ tie exactly, this reflects on which side of the tie its posterior falls, and either selection is correct.
Over the 20 races, under the uniform budget preference, weights at both ends and weights in the middle, selection on ratings selects $B$ every time, while selection on win probabilities selects a member of the Smith set $\{A, B\}$ in 18, 19 and 17 races and selection on utilities in 18, 19 and 16; the other races select $E$.

\FloatBarrier
\section{Validation on Synthetic Data}
\label{sec:validation}

The guarantee of \textsc{PolarBear} (\cref{sec:guarantee}) holds on average over the prior of an exactly specified model.
In \cref{sec:prior-truths}, we test this claim empirically by drawing truths from each model's own prior.
In \cref{sec:persistence}, we let the outcomes depend on each other across budgets, as those of real runs do, so that models treating the budgets as independent are misspecified and the guarantee no longer applies.
In \cref{sec:scalability}, we measure how the cost of a fit grows with the number of algorithms.

\subsection{Truths Drawn from the Models' Priors}
\label{sec:prior-truths}

We draw truths from a model's prior, run races on them under that model, and count the false statements of all races against their bound.
The true ratings are drawn uniformly from the simplex at every budget, independently across budgets, and each instance ranks the algorithms by a Plackett--Luce draw from these ratings, with noise independent across instances and budgets.
Plackett--Luce is then exactly specified, and so is the per-pair model, since every pair's win probability at every budget is uniform under this prior.
Bradley--Terry is raced on pairwise observations of true ratings from the same prior: a comparison runs one pair on a fresh instance, and the pair's order there follows Bradley--Terry with the same ratings.
We perform races in two sizes, a \emph{small} one of 4 algorithms on 2 budgets and a \emph{large} one of 8 algorithms on 5 budgets, and on every truth we run a race under each test, at threshold $\alpha = 0.99$ and rope $\epsilon = 0.05$.

Each model and size is raced on new truths until the bound $\Lambda$ of the directional statements, summed over the races, reaches 30 under each test.
The allocation is that of \cref{sec:procedure}, with $\kappa$ the evaluations of eight blocks.
A race that has not ended after 2{,}165 blocks is stopped there; this cap is the number of instances that a pair with win probability $\half + \epsilon/2$ is estimated to need at $\alpha = 0.99$ and $\epsilon = 0.05$ (\cref{app:prior-truths}).
The guarantee holds for any stopping rule, so a stopped race counts like the others.

This takes 1{,}063 to 1{,}347 races per model in the small size and 129 to 208 in the large one, whose races each add more to the bound.
\Cref{fig:known-truth}, panel (a), shows $X/\Lambda$ for the directional statements with 95\,\% intervals.
The intervals resample races, not statements: the statements of one race rest on the same instances, so a race that makes one false statement often makes several, and an interval that treated the statements as independent would be too narrow.
Each interval is a percentile bootstrap: the middle 95\,\% of $X/\Lambda$ over 10{,}000 sets of races drawn with replacement from the original races, each set as large as the original.
Under the joint test, the ratios of Plackett--Luce, Bradley--Terry and the per-pair model run from 0.86 to 1.20 over the two sizes, under the pointwise test from 0.80 to 1.20, and every interval includes 1; \cref{tab:prior-truths} gives them by model and size.
The three models are consistent with the guarantee in every case, and they sit at their bound, not well below it: where a model is exactly specified, the guarantee is not conservative, and $\Lambda$ is about the number of false statements to expect.

The near-ignorance prior is raced on truths drawn from a Dirichlet process of strength $s=1$ whose base measure ranks the algorithms uniformly at random at every budget, independently across budgets.
Under the joint test, its $X/\Lambda$ is 0.43 [0.13, 0.87] in the small size and 1.15 [0.26, 2.38] in the large one, below its bound in the small size and consistent with it in the large one, as a bound that is not an equality allows (\cref{sec:guarantee}).
The tie-point prior, raced on the same truths, is at 2.03 [0.92, 3.42] and 8.17 [1.50, 17.4].
It is exactly specified only where every pair ties, so no truth drawn from its own prior can test its guarantee (\cref{sec:dp}); on these truths, in the large size, its statements are false eight times as often as their recorded probabilities say.

Eliminations matter most among the statements, since a member of $\PO$ can leave $\POhat$ only through a false elimination (\cref{sec:resolution}).
Over the two sizes, each model makes about 2{,}000 to 2{,}600 of them under the joint test, and 5 to 9 are false, against bounds of 5.39 to 8.10, too few for intervals.
Under the pointwise test the races make nearly the same eliminations.
What the pointwise test costs is a weaker promise, not a broken bound: it records the probability of an elimination by the union bound over budgets (\cref{sec:tests}), which makes its bound 1.2 to 1.7 times the joint test's.
\Cref{app:prior-truths} gives the eliminations by model, the false equivalences and the posterior fits that fail a diagnostic check.

Where Plackett--Luce and Bradley--Terry are exactly specified, their rating scale is true, and since it lets the evidence on one pair inform the others, it should save instances over the per-pair model.
Pairing the races by truth under the joint test, a per-pair race needs 1.02 [0.99, 1.05] times the instances of the Plackett--Luce race in the small size and 1.15 [1.08, 1.22] times in the large one, as geometric means over 1{,}063 and 129 pairs of races.
Raced on pairwise observations as well, the per-pair model needs 1.04 [1.01, 1.08] times the instances of Bradley--Terry, over 1{,}275 pairs of races in the small size.
The saving is small in the small size and larger in the large one, which has more algorithms and more budgets; these two sizes, however, are not enough to identify a general trend.

\begin{figure}[tbp]
	\centering
	\includegraphics[width=\textwidth]{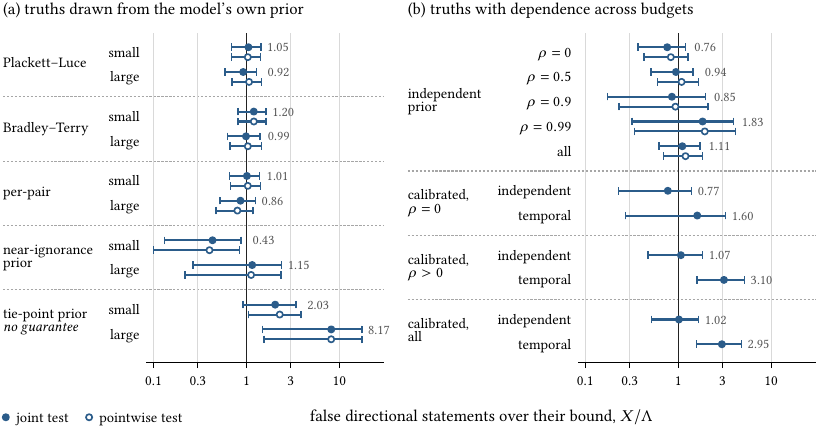}
	\caption[False statements over their bound on synthetic truths]{False directional statements over their bound, $X/\Lambda$, with 95\,\% bootstrap intervals over races. (a) Truths drawn from a model's own prior, in two sizes, under the joint test (filled) and the pointwise test (hollow); the tie-point prior has no guarantee. (b) Truths with dependence across budgets of strength $\rho$: the independent prior by strength and over all strengths under both tests, then the two calibrated priors side by side on the same truths, where $\rho = 0$, where $\rho > 0$, and over all strengths. The number beside a row is its ratio under the joint test.}
	\Description{Two panels on a logarithmic axis with a line at 1. Each row is a model and a size, or a prior and a strength of dependence, with a point and an interval for each test and the ratio under the joint test as a number.}
	\label{fig:known-truth}
\end{figure}

\subsection{Dependence across Budgets}
\label{sec:persistence}

Real algorithm runs break an assumption of Plackett--Luce, Bradley--Terry and the per-pair model: an instance is ranked at every budget by the same runs, so its outcomes at neighboring budgets are positively dependent, while these models treat the budgets as independent.
The joint test then errs on the cautious side, since under positive dependence the product over budgets is a lower bound on the joint probability (\cref{sec:tests}), but a model that is not exactly specified has no guarantee (\cref{sec:guarantee}), so whether its statements stay within their bound has to be measured.
We measure it for Plackett--Luce on truths with dependence of controlled strength, under two priors: the prior of \cref{sec:pl}, which we call the \emph{independent prior}, and a \emph{temporal prior}, under which the ratings change smoothly with the budget, so that neighboring budgets inform each other.

The temporal prior keeps the likelihood of \cref{sec:pl} and places Gaussian processes with a Mat\'ern 3/2 kernel~\cite{rasmussen2005} on the algorithms' utilities over the budgets, so that the ratings change smoothly with the budget; \cref{app:helmert} gives it in full.
It promises two things: fewer instances, since the evidence at one budget informs the neighboring ones, and a posterior that is joint over the budgets, so that the probability of a statement about all budgets is read from its draws instead of from a product.
Like the independent prior, it assumes that, given the ratings, the rankings of one instance at different budgets are independent.
This is false for runs, and the two priors differ in what it costs them: the temporal prior pools across budgets, so it counts the evidence of one instance once per budget and its posterior is narrower than the instances warrant, while the independent prior rests at each budget on that budget's rankings alone.

The truths have 4 algorithms on a linear grid of 10, 20 or 30 budgets, with utilities drawn from the temporal prior.
The noise an instance adds to an algorithm's utility is correlated across budgets with a strength $\rho$, the correlation of its Gaussian copula between neighboring budgets of the 10-budget grid; the strength leaves the truth as it is and sets how far the rankings of one instance at different budgets agree (\cref{fig:dependence-truths} in \cref{app:calibration}).
We use $\rho = 0$, 0.5, 0.9 and 0.99, with 30 truths for each strength and grid, 360 in all; at $\rho = 0$ the temporal prior is exactly specified, so these truths are its best case.
The independent prior races every truth under both tests, with the doubling policy of \cref{sec:procedure} and a cap of 256 blocks, neither of which changes what is tested (\cref{sec:guarantee}).

Over the four strengths together, the independent prior's $X/\Lambda$ for directional statements is 1.11 [0.62, 1.71] under the joint test and 1.20 [0.69, 1.82] under the pointwise test, consistent with its bound although it is not exactly specified on these truths; by strength, every interval includes 1 (\cref{fig:known-truth}, panel (b)).
Dependence changes how the false statements fall on the races: fewer races are wrong, and each wrong race is wrong at more budgets, which is why the intervals widen with $\rho$.
The pointwise and the joint test make nearly the same eliminations, 462 and 460, of which 1 and 0 are false; the union bound that the pointwise test records for them is 3.6 times the joint test's, against 1.2 to 1.7 on the two and five budgets of \cref{sec:prior-truths}.

Fitted to the first 20 and the first 80 instances of the same truths, without racing, the temporal prior's central 90\,\% intervals of the difference of two algorithms' utilities at a budget cover the true difference in 0.89 of the cases at $\rho = 0$ and in only 0.59 to 0.75 under dependence, while the independent prior's stay between 0.88 and 0.90 at every strength (\cref{tab:coverage}).
A posterior that counts an instance's evidence several times can be widened after the fit: \emph{jackknife calibration} estimates the variance of each quantity a statement reads over the instances by the infinitesimal jackknife~\cite{giordano2023} and widens the quantity's draws where this exceeds the posterior variance (\cref{app:calibration}).
Even with this calibration, however, the temporal prior is overconfident: under dependence its calibrated intervals cover the truth in 0.84 to 0.87 of the cases.
Raced with calibration under the joint test on 270 of the truths, the temporal prior's $X/\Lambda$ is 2.95 [1.57, 4.79] and the independent prior's 1.02 [0.51, 1.66] (\cref{fig:known-truth}, panel (b)).

Where both calibrated races finish, on 119 of the 270 truths, the temporal prior needs 0.94 [0.86, 1.03] of the independent prior's instances, at 2.3 [2.1, 2.5] times its compute.
Neither of its promises holds: it saves no clear number of instances, at twice the compute, and its joint posterior is overconfident under dependence across budgets.
We report it as a negative result and do not use it further in this paper.

\subsection{Scalability}
\label{sec:scalability}

Three things grow with the number of algorithms $n$: the cost of a fit, the number of pairs, and the bound on false eliminations.
We time fits of Plackett--Luce with the independent prior at one budget, for 2 to 100 algorithms and 10 to 1{,}000 instances, on truths drawn as in \cref{sec:prior-truths}, by Pathfinder and by MCMC, each time the median of five fits on 4 workers.
Its posterior factorizes over budgets, so a grid of $T$ budgets costs $T$ such fits, which can run in parallel.
The pair models are not timed, since their posteriors cost almost nothing to compute (\cref{sec:procedure}).

At 1{,}000 instances, a fit by Pathfinder takes 6.4\,s for 8 algorithms and 101\,s for 100, and a fit by MCMC 35\,s and 1{,}079\,s.
For 12.5 times as many algorithms, Pathfinder thus takes 16 times as long, about in proportion, and MCMC 30 times as long, more than in proportion.
Ten times as many instances, from 100 to 1{,}000, take both 7 to 9 times as long.
With 2 and 4 algorithms the fits barely grow with the instances, since a fit's cost follows the number of distinct rankings, of which there are at most $n!$.
\Cref{fig:fit-times} in \cref{app:supplement} shows every combination of size under both methods.
Whether this compute matters beside the algorithm runs depends on what a run costs, which can be almost nothing or far more than a fit, while the cost of a fit depends only on the numbers of algorithms and instances, not on how expensive the runs behind the rankings were.
\Cref{app:replays} compares the two on the algorithm runs of \cref{sec:replays}.

The prior does not change with the number of algorithms: under ratings uniform on the simplex, the win probability of every pair is uniform whatever $n$ is, so what the prior says about a pair does not depend on how many others are in the race.
The number of pairs grows quadratically, and with it the relations to decide.
The bound on false eliminations grows linearly, since a race makes at most $n - 1$ eliminations (\cref{sec:guarantee}): at $\alpha = 0.99$ it is 0.11 for 12 algorithms and 0.99 for 100.

\FloatBarrier
\section{Experiments on Recorded Algorithm Runs}
\label{sec:replays}

We evaluate \textsc{PolarBear} on recorded runs of real algorithms, in three parts.
In \cref{sec:truth}, we apply it to a standard benchmark setup, MA-BBOB, on which the standard tools of anytime analysis can be readily applied.
In \cref{sec:gpbbob}, we apply it to a full black-box setup, GP-BBOB, an instance distribution whose optima and objective scale are unknown and whose dimension varies widely, where these tools are not readily applicable, with algorithm time as the budget.
In \cref{sec:five-models,sec:why,sec:racers,sec:selection-results}, we compare the five models of \cref{sec:models} across both benchmarks and discuss where they differ, compare \textsc{PolarBear} with F-race and two Bayesian variants of it, and analyze selection of algorithms from the returned sets under different budget preferences.

\subsection{Setup}
\label{sec:setup}

We construct each instance distribution as a \emph{pool}: 5{,}000 instances sampled from an instance generator, with a recorded run of each of 12 algorithms on each instance.
A race draws its instances from the pool uniformly with replacement, so the pool is its instance distribution, and the true win probability of a pair at a budget is its win share over the pool, computable without a model.
The pool thus fixes one truth, which the guarantee does not cover, since it holds only on average over the prior (\cref{sec:guarantee}).
We therefore test empirically whether the false statements of a race, counted against the pool's truth, stay within the bound that an exactly specified model would meet.
The race is \emph{replayed}: it reads its runs from the pool instead of executing them anew, so all races reuse the same recorded runs.

Six pools hold instances of MA-BBOB~\cite{vermetten2023}, many-affine combinations of the BBOB functions, one pool for each dimension $d = 2$, 5, 10, 20, 30 and 40, and their algorithms are CMA-ES with eleven step-size adaptations and random search\footnote{All algorithms are our own implementations (\cref{app:cmaes}). What this section finds about them holds for these instance distributions and these implementations, and we make no statements about the algorithms in general.} (\cref{app:cmaes}).
Step-size adaptation shapes how fast an evolution strategy converges, so these variants can be expected to differ over the budget.
Each run takes $10^5$ evaluations and is read at 30 budgets from 100 to $10^5$, evenly spaced in the logarithm.
The seventh pool holds GP-BBOB instances, sparse mixtures of BBOB functions that change over the search space, whose optima and objective scale are deliberately unknown (\cref{app:gpbbob}).
Their dimension is drawn as $d \sim \lfloor \mathcal{N}(240, 30^2) \rfloor$, truncated to $[2, 320]$, so that the pool's instances have between 133 and 320 dimensions, 90\,\% of them between 190 and 287.
The distribution is not meant to resemble a particular application; it stands for a high-dimensional distribution of heterogeneous instances.
Its algorithms come from three families: eight CMA-ES variants, three variants of differential evolution, and random search.
The CMA-ES variants differ in how they adapt the covariance matrix, and three of them also in the rule that adapts the step size (\cref{app:cmaes}).
In both benchmarks, every variant of CMA-ES and of differential evolution restarts when a start degenerates or stalls, with the population doubled as in IPOP-CMA-ES~\cite{auger2005} (\cref{app:cmaes}).
Its budget is algorithm time, i.e., wall-clock time excluding the overhead of recording the best-so-far: 60\,s per run, read every 2\,s.

We run 20 races of each of the five models on each pool, 700 in all, each from its own seed, which fixes the instances it draws, under the joint test and from an initial design of eight blocks; \cref{tab:settings} in \cref{app:replays} lists every setting.
We report the false statements of the races over their bound, $X/\Lambda$, with 95\,\% bootstrap intervals over races drawn within each pool, as in \cref{sec:prior-truths}.
The returned set $\POhat$ is \emph{correct within the rope} when $\PO \subseteq \POhat \subseteq \POeps$, and \emph{exact} when it equals $\PO$.

The allocation is that of \cref{sec:procedure}, with the fixed cost of a round, $\kappa$, set per pool from the duration of Plackett--Luce's fits in a preliminary race (\cref{tab:kappa}).
The per-pair model and the two Dirichlet process priors take the same $\kappa$, so that the four decide at the same rhythm and their numbers of instances compare like for like; Bradley--Terry takes its own, from its own fits.
A race that has not ended after 2{,}165 blocks, the cap of \cref{sec:prior-truths}, is stopped there; under Bradley--Terry a block is the comparison of one pair.
The cap is a choice of this study, not part of the procedure, and a race stopped by it is reported as \emph{capped}, not as finished.

\subsection{A Standard Benchmarking Setup: MA-BBOB}
\label{sec:truth}

MA-BBOB is a traditional benchmark setting: its instances are built from the well-studied BBOB functions~\cite{hansen2009} and have known optima, so the established tools of anytime analysis, the EAF, the ECDF and the AOCC, can be readily applied.
The goal is to evaluate \textsc{PolarBear} in a setting where its results can be checked against the established analysis, not to draw conclusions about the specific algorithms.
We first set the established tools and the anytime Pareto set side by side on the truth of the $d = 30$ pool; both are computed from all 5{,}000 instances, so the comparison is between quantities, not estimates.
We then follow one race on the same pool, which identifies the Pareto set from far fewer instances.

In the truth, the Pareto set holds six of the twelve algorithms, \alg{csa}, \alg{hsa}, \alg{m-xnes}, \alg{ncsa}, \alg{psr} and \alg{tpa} (labels as in \cref{tab:algorithms}), and the widened set is the same.
All 15 pairs among them cross.
The leader changes six times over the budgets: \alg{tpa} leads at 100 evaluations, \alg{m-xnes} from 127 to 672, \alg{psr} from 853 to 1{,}743, \alg{hsa} from 2{,}212 to 7{,}279, \alg{csa} at 9{,}237 and 11{,}721, \alg{ncsa} from 14{,}874 to 48{,}939, and \alg{tpa} again from 62{,}102 on.

The ECDF, here the EAF-based ECDF~\cite{lopez-ibanez2024} over the target precisions, distances to the known optimum, from $10^{2}$ to $10^{-8}$ on a logarithmic scale, mostly agrees with the win probability: its ordering of a pair agrees with the side of $\half$ at 95\,\% of pairs and budgets.
\Cref{fig:truth}, panel (a), sets the algorithm with the highest ECDF beside the leader at each budget: the two coincide at 20 of the 30 budgets, and at the last three the ECDF is highest for \alg{xnes}, which is outside the Pareto set.
The AOCC, the area under a run's ECDF over logarithmic budgets, orders the algorithms by overall area and does not mark the Pareto set: \alg{ppmf}, outside it, has 0.345, next to \alg{psr}, inside it, with 0.356 (panel (b)).
Against the uniform budget preference on win probabilities, the AOCC orders 61 of the 66 pairs the same way, but names \alg{csa} as best instead of \alg{hsa}, and separates the two by 0.002.

The bounds of the normalization decide what these tools measure.
Across instances, the AOCC of the leading algorithms has two modes, near 0.2 and near 0.6, and the mean lies between them.
The low mode holds the instances not solved: of \alg{csa}'s instances in it, 16\,\% end below $10^{-8}$, against 85\,\% of the rest.
The mean is therefore largely a share of solved instances, and the choice of lower bound decides what counts as solved: with $10^{-2}$ in place of $10^{-8}$, \alg{hsa} moves from third to first, and \alg{tpa} falls from fourth to sixth.

The two analyses answer different questions.
The standard tools fix a final target, and instances below it count as solved alike; the win probability is scale-free and counts a consistent advantage whatever its size.
\alg{csa} and \alg{tpa} cross twice in win probability, after 204 and after 48{,}939 evaluations, and once in the ECDF: \alg{tpa}'s late lead is how much further it converges below $10^{-8}$.
\Cref{fig:eaf} shows the EAF differences for \alg{csa} against an algorithm it dominates (\alg{msr}), one it crosses twice (\alg{tpa}) and its closest rival (\alg{ncsa}): against \alg{tpa} the difference fades after $10^4$ evaluations, while the win probability falls below $\half$ at the last budgets, and against \alg{ncsa} the difference never exceeds 0.04, while the win probability reaches 0.61.
\Cref{fig:eaf-pairs} in \cref{app:supplement} shows the EAF difference and the win probability of every pair of the pool.

Where differences below some precision do not matter to the decision, the values can be clamped to that precision before ranking, so that what lies below it ties.
We do not clamp by default, since objective functions can be scaled arbitrarily and no precision is meaningful across them.
With the values clamped at $10^{-8}$, the leader is a different algorithm at 6 of the 30 budgets, and the Pareto set stays the same.

\begin{figure}[tbp]
\centering
\includegraphics[width=\textwidth]{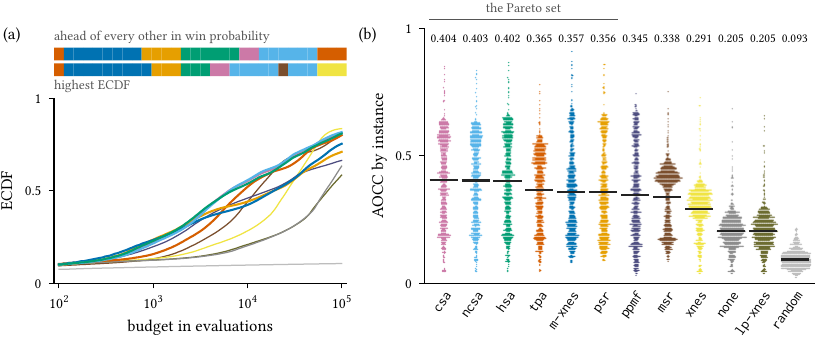}
\caption[The truth at $d = 30$ beside the standard analysis]{The $d = 30$ pool under the standard analysis. (a) The EAF-based ECDF by budget, the share of the targets from $10^{2}$ to $10^{-8}$ that are reached, one curve per algorithm in the colors of (b). The strips name, at each budget, the leader in win probability and the algorithm with the highest ECDF. (b) The AOCC by instance for the first 1{,}000 instances, with the mean over all 5{,}000 as a line and as the number above the column; the algorithms are in the order of their means.}
\Description{A panel of twelve curves over the budgets under two strips of colored segments, and a panel of twelve columns of points, one per algorithm, each with a line at its mean and the mean as a number.}
\label{fig:truth}
\end{figure}

\begin{figure}[tbp]
\centering
\includegraphics[width=\textwidth]{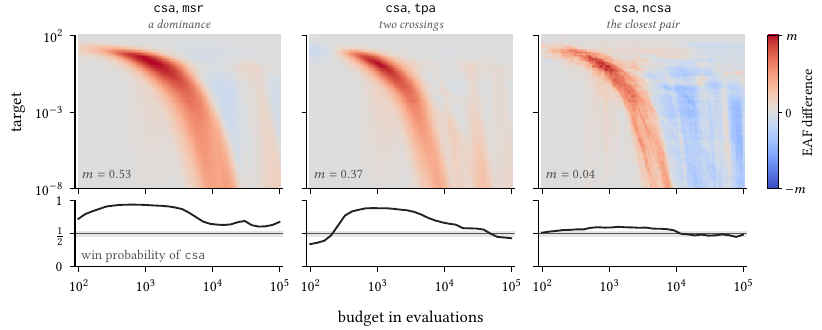}
\caption[EAF differences for three pairs]{Differences of the empirical attainment functions at $d = 30$ for \alg{csa} against an algorithm it dominates, one it crosses twice, and its closest rival: the difference in the share of instances that reach a target by a budget, red where the share of \alg{csa} is higher and blue where the other's is. The colors of a panel span the pair's largest absolute difference, $m$. Below each panel is the win probability of \alg{csa} over the other by budget, with the rope as a band.}
\Description{Three heat maps over budget and target, one per pair, each over a strip with the pair's win probability by budget.}
\label{fig:eaf}
\end{figure}

To show how a race arrives at this Pareto set, we follow one of the pool's races under the per-pair model, whose posterior is exact, so that its bands carry no sampling error.
It returns \alg{csa}, \alg{hsa}, \alg{m-xnes}, \alg{ncsa}, \alg{psr} and \alg{tpa}, the Pareto set exactly, after 724 instances and 18 rounds.
It eliminates the other six algorithms, none of them falsely, and 3 of its 1{,}484 directional statements are false.
All 15 pairs of the returned set are resolved as crossings, 13 of them within 131 instances, from the budgets where the pair is clear; the two pairs that set the race's length, \alg{hsa} against \alg{csa} and \alg{csa} against \alg{ncsa}, are resolved only after 724.
Resolution saves runs as well as elimination: \alg{tpa}, \alg{psr} and \alg{m-xnes} stay in the returned set after 219, 199 and 107 runs, since their pairs were resolved early, and the race takes 3{,}424 runs in all.

\Cref{fig:one-race} shows five of its pairs, chosen by rule: the last elimination (\alg{csa} against \alg{msr}), the first crossing (\alg{m-xnes} against \alg{tpa}), the pair whose truth crosses twice (\alg{csa} against \alg{tpa}), and the two pairs that end the race; \cref{fig:one-race-pairs} in \cref{app:supplement} shows all 15 pairs of the returned set.
Beyond the relations the race declares, the posterior shows how a pair changes over the budgets: \alg{csa} is behind \alg{tpa}, then ahead, then behind again, following the truth's two crossings; \cref{app:replays} describes the other four pairs in the same way.
A directional statement can be declared inside the rope: the race declares that \alg{csa} beats \alg{ncsa} at 100 evaluations, where 0.68 of the final posterior's mass lies within the rope, and the statement is true.

\begin{figure}[tbp]
\centering
\includegraphics[width=\textwidth]{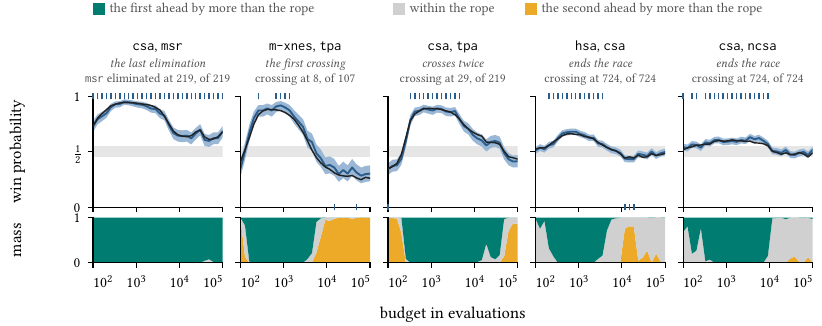}
\caption[Five pairs of one race]{Five pairs of the per-pair race of one seed on the $d = 30$ pool. Above: the final posterior of the pair's win probability against the pool's truth, with the mean and its 90\,\% band, the rope, ticks at the budgets where the race declared a directional statement, and the pair's outcome with the instances it took. Below: by budget, the posterior mass with the first algorithm ahead by more than the rope, within the rope, and with the second ahead by more than the rope.}
\Description{Five columns, one per pair, each with a panel of win probability over the budgets and below it a strip of three stacked shares over the budgets.}
\label{fig:one-race}
\end{figure}

Across the six dimensions, the Pareto set changes little: \alg{csa}, \alg{hsa}, \alg{ncsa} and \alg{tpa} are in it at every dimension, \alg{m-xnes} from $d = 5$ on, \alg{xnes} up to $d = 10$ and \alg{psr} from $d = 20$ on, while \alg{msr}, \alg{lp-xnes}, \alg{ppmf}, \alg{none} and \alg{random} never are.
The last budgets belong to \alg{tpa} at every dimension, which reflects how much further it converges.
From $d = 20$ on, the leaders before it follow one order, \alg{m-xnes}, \alg{psr}, \alg{hsa}, \alg{csa} and \alg{ncsa}, each handing over later at a larger dimension; \cref{app:replays} gives the budgets of the hand-overs and three departures from this order.
The default, \alg{csa}, is in the Pareto set at every dimension but leads at only one or two budgets.

At $d = 2$, \alg{ncsa} is the one member of any pool's Pareto set that leads at no budget: at every budget some rival beats it, but no single rival beats it at all of them.
\Cref{fig:ncsa} in \cref{app:supplement} shows every rival's win probability against it: the strongest rival lies between 0.54 and 0.73 at every budget, and every rival falls below $\half$ at some budget.
An analysis that looks at one budget at a time would never return it.

\subsection{A Black-Box Setup: GP-BBOB}
\label{sec:gpbbob}

GP-BBOB is a setting where the established tools struggle: an instance distribution whose optima and objective scale are unknown, and whose dimension varies between 133 and 320.
Its CMA-ES variants differ in how they represent the covariance matrix, and with it in the cost of a generation (\cref{tab:generation-times}).
As a hypothetical research question, suppose a practitioner whose problems resemble these instances asks: does expensive covariance adaptation pay off in algorithm time, where a cheaper generation leaves more evaluations within the same time?
To answer it with the standard analysis, one would need bounds for every instance although its optimum is unknown, and a normalization that makes objective values comparable across dimensions from 133 to 320; the win probability needs neither.
The standard analysis also fixes the number of instances in advance and runs every algorithm on all of them, whereas a race decides during the experiment how many instances it needs, and stops running an algorithm once its pairs are resolved.

\begin{figure}[tbp]
\centering
\includegraphics[width=\textwidth]{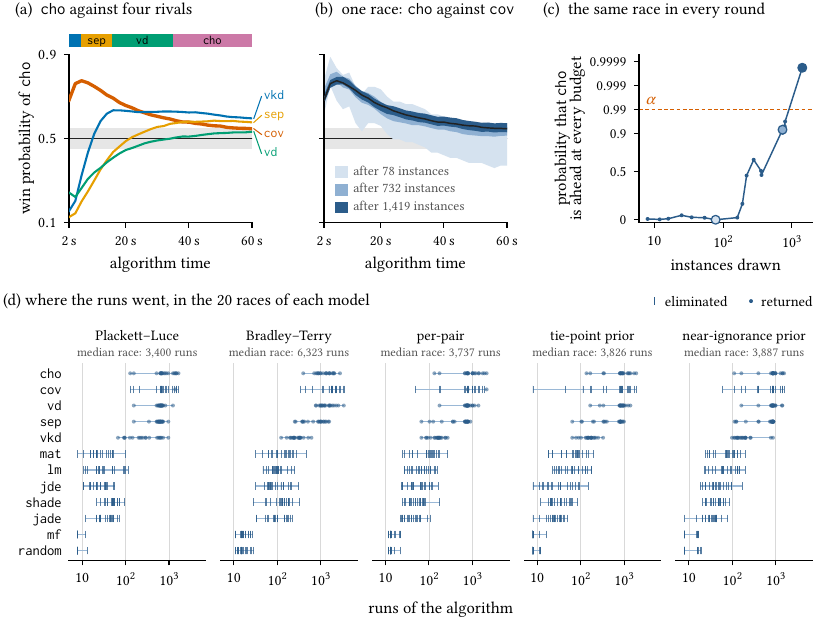}
\caption[The GP-BBOB pool and its races]{The GP-BBOB pool and its races. (a) The pool's truth: the win probability of \alg{cho} against \alg{cov}, which adapts the same covariance in another implementation, and against the other three members of the Pareto set, with the rope as a band. The strip names the leader at each budget. (b) \alg{cho} against \alg{cov} in the per-pair race of one seed: the 90\,\% band of the posterior in three rounds, the mean in the last of them, and the truth. (c) The same pair in every round of that race: the probability that \alg{cho} beats \alg{cov} at every budget, with the threshold $\alpha$ of the joint test as a dashed line. The larger marks are the rounds of (b); above 0.9 the axis is logarithmic in the remaining probability. (d) The runs of each algorithm, one per instance it ran on, in each of the 20 races of a model: a tick for a race that eliminated the algorithm, a dot for one that returned it. The median race is given in runs, since a race under Bradley--Terry runs two algorithms on an instance.}
\Description{Three panels in a row: four curves of win probability over time under a strip of colored segments, three nested bands over time around a curve, and a rising line of points over the instances drawn with a dashed horizontal line. Below them, five panels, one per model, each with twelve rows of marks, one row per algorithm.}
\label{fig:gpbbob}
\end{figure}

In the truth, the Pareto set holds \alg{sep}, \alg{cho}, \alg{vd} and \alg{vkd} (\cref{tab:algorithms}).
Up to 34\,s only cheaper variants lead, \alg{vkd} up to 4\,s, \alg{sep} up to 14\,s and \alg{vd} up to 34\,s; from 36\,s on, \alg{cho}, which adapts the full covariance, leads to the end.
Expensive adaptation thus pays off, but only late, and by a margin inside the rope: \alg{cho} passes \alg{vkd} at 10\,s, \alg{sep} at 22\,s and \alg{vd} at 36\,s, and against \alg{vd} it stays within the rope from 22\,s on and reaches 0.533 at 60\,s.
It pays off only in the cheaper implementation: \alg{cov} adapts the same full covariance as \alg{cho}, but is dominated by \alg{cho} and by \alg{vd}, and \alg{cho} beats it with 0.55 to 0.78 at every budget (\cref{fig:gpbbob}, panel (a)), at a time per generation of 8.6\,ms against 11.3\,ms.
The variants of differential evolution are dominated by the diagonal and low-rank variants of CMA-ES.

The question also sets the length of a race.
In the per-pair race of \cref{fig:gpbbob}, panels (b) and (c), the seven algorithms outside the Pareto set other than \alg{cov} are eliminated within 78 instances, and the race runs to 1{,}419 instances to eliminate \alg{cov}, which \alg{cho} beats at 60\,s with only 0.547.
After 78 instances, the 90\,\% band of the posterior at 60\,s reaches from 0.37 to 0.55, and in the last round from 0.53 to 0.57.
The probability that \alg{cho} beats \alg{cov} at every budget, which the joint test compares with $\alpha$, stays below 0.05 up to 161 instances, lies between 0.46 and 0.63 from 220 to 362 and between 0.91 and 0.97 from 732 to 798, and reaches 0.9998 at 1{,}419, where \alg{cov} is eliminated.

Over all 100 races of the five models (panel (d)), \alg{cov} is the last elimination in 92.
Three races return it, two of them at the cap; since \alg{cov} is in the widened Pareto set, their sets are still correct within the rope.
Under the four models that draw an instance for every remaining algorithm, each of the other seven algorithms outside the Pareto set runs on a median of 8 to 99 instances, and \alg{cov} on 737 to 905.
Resolution saves runs here too: \alg{vkd} is in every returned set, and once its pairs are resolved it stops, after a median of 158 to 269 instances.

\FloatBarrier

\subsection{Comparing Models across All Experiments}
\label{sec:five-models}

\begin{figure}[tbp]
\centering
\includegraphics[width=\textwidth]{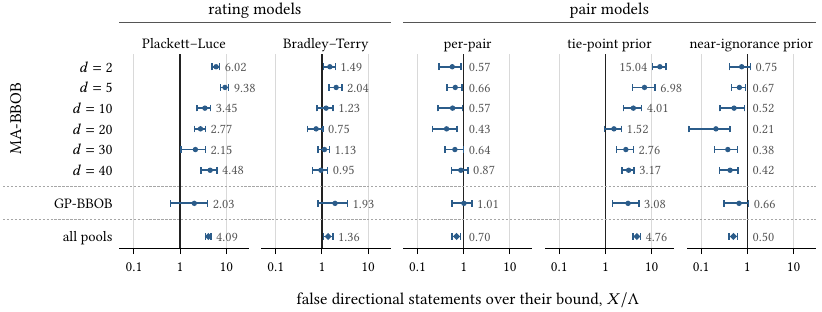}
\caption[False directional statements over their bound on the pools]{False directional statements over their bound, $X/\Lambda$, for the five models on every pool and over all pools, with 95\,\% bootstrap intervals over races. Left of the line at 1, the bound holds. The number beside an interval is its ratio.}
\Description{Five panels, one per model, on a shared logarithmic axis with a line at 1. Each has a row per pool and one for all pools, with a point and an interval.}
\label{fig:headline}
\end{figure}

We compare the five models over all seven pools, $7 \cdot 20 = 140$ races each, by the false statements they make, the sets they return and the runs they spend.
\Cref{tab:results} collects these over all pools; \cref{app:replays} gives them by pool.

\Cref{fig:headline} shows the false directional statements over their bound for every model and pool.
Over all pools, the per-pair model is at 0.70 [0.56, 0.84] and the near-ignorance prior at 0.50 [0.40, 0.61], both below their bound, and on every pool the interval of either includes 1 or lies below it.
On truths drawn from its own prior, the per-pair model sat at its bound (\cref{sec:prior-truths}); the pools are easier for it than its prior expects.
The rating models exceed their bound: Plackett--Luce is at 4.09 [3.55, 4.66] over all pools and above its bound on six of the seven, at 9.38 at $d = 5$, and Bradley--Terry is at 1.36 [1.07, 1.74], above its bound at $d = 2$ and $d = 5$ and consistent with it on the other five.
The tie-point prior, which has no guarantee, is at 4.76 [3.90, 5.71] and above its bound on six pools, at 15.0 at $d = 2$.
Equivalences divide the models the same way, more strongly for the rating models: Plackett--Luce makes 20.6 times as many false equivalences as its bound and Bradley--Terry 4.6 times, while the three pair models stay within theirs.

Of the $5 \cdot 7 \cdot 20 = 700$ returned sets $\POhat$, 690 are correct within the rope.
The other ten are all Plackett--Luce's at $d = 5$, and each lacks \alg{m-xnes}, a member of the Pareto set that the race wrongly eliminated.
Plackett--Luce makes 11 false eliminations in its 820, against a bound of 1.26: these ten and one at $d = 20$.
That one names a dominator of \alg{xnes} that does not dominate it; another algorithm does, so \alg{xnes} is outside the Pareto set and the race's set stays correct within the rope.
The other four models make no false elimination, against bounds of 1.83 to 2.83.
These zeros do not rank the four models: with a bound below three, at most a few false eliminations are expected under any of them.
A set that is correct within the rope is exact in about half the races, 66 to 70 of 140 for each model, and otherwise also holds members of the widened set.

Relative to the per-pair model, with the pools weighted equally, Plackett--Luce costs 0.86 [0.82, 0.89] and Bradley--Terry 1.84 [1.73, 1.95], and the two priors of the Dirichlet process are within 7\,\% of it (\cref{tab:results}).
The rating scale of Plackett--Luce saves a third on some pools and costs more on one: its ratio runs from 0.66 [0.60, 0.73] at $d = 20$ to 1.14 [1.01, 1.28] at $d = 40$.
Bradley--Terry pays for observing one pair at a time.
\Cref{tab:cost} in \cref{app:supplement} gives the absolute cost by pool, in evaluations on MA-BBOB and in hours of algorithm time on GP-BBOB.

All races above use the joint test.
Under the per-pair model at $d = 5$ and $d = 30$, we also race the pointwise test (\cref{sec:tests}), which loses nothing in the returned sets or the false statements on either pool.
It saves a third of the evaluations at $d = 5$, where it costs 0.66 [0.62, 0.70] of the joint test, and nothing measurable at $d = 30$, where it costs 1.00 [0.94, 1.06].
\Cref{app:replays} gives the comparison in full.

\Cref{fig:race-lengths} shows how many runs each race makes, which compares Bradley--Terry, with one pair per instance, with the four models that run every remaining algorithm on an instance.
The median race of a pool makes 2{,}427 to 8{,}215 runs under these four models and 5{,}220 to 20{,}614 under Bradley--Terry.
The length varies between repetitions on the same pool: under the four models, the longest race of a pool makes up to 4.4 times as many runs as the shortest on MA-BBOB and up to 6.5 times on GP-BBOB, and it draws up to 7.3 and 12.1 times as many instances; under Bradley--Terry, it draws up to 4.1 times as many on MA-BBOB and 3.0 times on GP-BBOB.
No race ends before 155 instances, and 24 of the 700 reach the cap (\cref{sec:setup}).
Where Plackett--Luce and the tie-point prior end earlier, they do so partly by overstating (\cref{fig:headline}).

A number of instances fixed before the experiment would leave pairs open.
The number that completes every per-pair race of a pool is 852 at $d = 30$ and 1{,}779 to 2{,}124 on the other six, 1.3 to 2.2 times the pool's median race (\cref{fig:open-pairs} in \cref{app:supplement} shows the open pairs over the race).
The study that introduced MA-BBOB runs every algorithm on 1{,}000 instances~\cite{vermetten2023}; for our twelve algorithms, that is 12{,}000 runs, the dashed line in \cref{fig:race-lengths}.
Depending on the pool, the races of the four models make 0.20 to 0.70 of these runs on average, although 228 of their 480 races draw more than 1{,}000 instances, since beyond the first rounds a race runs only the algorithms whose pairs are still open.
A fixed design of 1{,}000 instances would thus cost more and still leave pairs open in these 228 races.
Bradley--Terry, which runs one pair per instance, makes 0.47 to 1.71 times as many runs as the fixed design.

\begin{figure}[tbp]
\centering
\includegraphics[width=\textwidth]{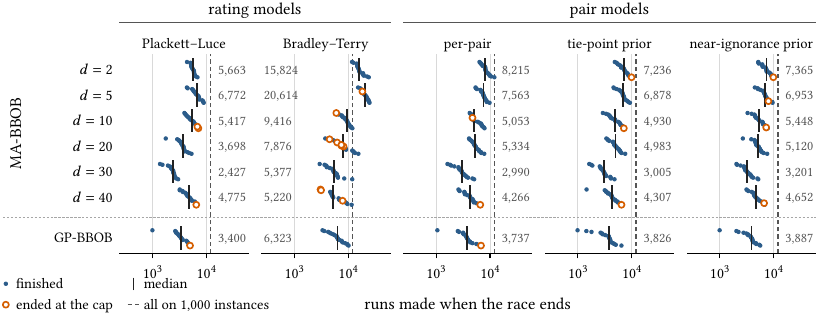}
\caption[Race lengths]{The number of algorithm runs a race has made when it ends, per model and pool, for each of 20 repetitions. A race under Bradley--Terry runs two algorithms on an instance and the other models every remaining algorithm, so runs are comparable across models while instances are not. The dashed line is the fixed design of all twelve algorithms on 1{,}000 instances, 12{,}000 runs. A larger open circle is a race that ends at the cap: it stops at the last round whose successor would pass 2{,}165 instances or, under Bradley--Terry, would give one pair more than 2{,}165 comparisons, so capped races end at different lengths. The tick is the median of the 20 races, also given by the number.}
\Description{Five panels, one per model, with a row per pool, a dot per race at the number of runs made when the race ends, and a dashed line at 12,000 runs.}
\label{fig:race-lengths}
\end{figure}

\begin{table}[tbp]
\caption[Results over all seven pools]{The five models over all seven pools, 140 races each. The returned set $\POhat$ is correct within the rope when $\PO \subseteq \POhat \subseteq \POeps$, and exact when it equals $\PO$. False eliminations stand against their bound; for statements the table gives the false ones over their bound, $X/\Lambda$. Evaluations are relative to per-pair, pools weighted equally.}
\label{tab:results}
\small
\setlength{\tabcolsep}{4.5pt}
\begin{tabular}{@{}lccccccc@{}}
\toprule
 & \multicolumn{2}{c}{Returned sets} & False & \multicolumn{2}{c}{Statements, $X/\Lambda$} & & Races at \\
\cmidrule(lr){2-3}\cmidrule(lr){5-6}
Model & within & exact & eliminations & directional & equivalence & Evaluations & the cap \\
\midrule
Plackett--Luce & 130 & 69 & 11 & 4.09 & 20.55 & 0.86 & 4 \\
 & & & {\footnotesize against 1.26} & {\footnotesize [3.55, 4.66]} & {\footnotesize [17.7, 23.5]} & {\footnotesize [0.82, 0.89]} & \\
Bradley--Terry & 140 & 66 & 0 & 1.36 & 4.57 & 1.84 & 10 \\
 & & & {\footnotesize against 2.51} & {\footnotesize [1.07, 1.74]} & {\footnotesize [2.90, 6.32]} & {\footnotesize [1.74, 1.95]} & \\
Per-pair & 140 & 66 & 0 & 0.70 & 0.30 & 1 & 3 \\
 & & & {\footnotesize against 2.83} & {\footnotesize [0.56, 0.84]} & {\footnotesize [0.00, 0.66]} & {\footnotesize [0.96, 1.04]} & \\
Tie-point prior & 140 & 70 & 0 & 4.76 & 0.50 & 0.93 & 3 \\
 & & & {\footnotesize against 1.83} & {\footnotesize [3.90, 5.71]} & {\footnotesize [0.00, 1.58]} & {\footnotesize [0.91, 0.96]} & \\
Near-ignorance prior & 140 & 68 & 0 & 0.50 & 0.90 & 0.99 & 4 \\
 & & & {\footnotesize against 2.19} & {\footnotesize [0.40, 0.61]} & {\footnotesize [0.22, 1.85]} & {\footnotesize [0.96, 1.02]} & \\
\bottomrule
\end{tabular}
\end{table}

\subsection{Why the Models Differ}
\label{sec:why}

The models differ for two reasons: the rating models assume one rating scale per budget, which the truth follows only approximately, and the tie-point prior becomes certain too early.
A rating model compresses every pair onto that scale.
Where the truth has no such scale, the estimate of a rating model does not converge to $p_{ij}(t)$ as instances grow, but to the model's \emph{limit}: the value a fit to all instances of the pool would give, over the algorithms the round still holds.
\Cref{fig:why}, in its two left panels, scores the statements of the rating models against their limit as well as against the truth.
Scored against its limit instead of the truth, the false directional statements of Plackett--Luce fall from 4.09 to 1.28 [0.94, 1.72] times their bound and those of Bradley--Terry from 1.36 to 1.09 [0.83, 1.47], and every interval of either model, over all pools and on each pool, includes 1.
Of the 1{,}024 false directional statements of Plackett--Luce, 784 would be true if its limit were the truth; at $d = 5$, where its ratio against the truth is 9.38, it is 1.41 [0.79, 2.13] against the limit.
Most of the excess of the rating models is thus the distance between their limit and the truth, not an error of estimation.

At $d = 5$, Plackett--Luce falsely eliminates \alg{m-xnes} in ten of its 20 races.
\alg{hsa} beats \alg{m-xnes} at every budget except the smallest, where its win probability is 0.491, so \alg{m-xnes} is not dominated.
At that budget, the limit puts \alg{hsa} ahead, at 0.52 to 0.55 over the 20 races, and so turns a near-tie into dominance.

The limit does not explain all false equivalences of Plackett--Luce: against its limit they are still 2.21 [1.33, 3.18] times their bound, down from 20.55 against the truth (\cref{app:replays}).
The plainest case of what no rating scale can represent is a cycle.
The truths of four pools hold cycles at single budgets, each with a smallest win probability less than 0.02 above one half, and no pool's truth holds a cycle of dominance, so every Pareto set is non-empty.
The pair models allow cycles (\cref{sec:resolution}), yet none of the 6{,}757 rounds of their 420 races holds a cycle of dominances.

The tie-point prior is certain too early.
Under it, eight instances without a loss make a directional statement certain (\cref{tab:pseudo-counts}), and 56\,\% to 88\,\% of its false statements by pool, 877 of 1{,}115 over all pools, are declared in the first round, on eight instances.
\Cref{fig:why}, in its right panel, leaves out the statements of the first round: without them, the ratio over all pools is 1.51 [1.17, 1.87], against 4.76 with them, so the first round explains most of the excess but not all of it.
Because the prior ignores ties, one win and seven ties also make a statement certain, and the prior does worst where ties are frequent: its ratio falls from 15.0 at $d = 2$ to 7.0 and 4.0 at $d = 5$ and 10 and lies between 1.5 and 3.2 above, while the share of ties over all pairs and budgets falls from 19\,\% at $d = 2$ to 1\,\% at $d = 40$.

\begin{figure}[tbp]
\centering
\includegraphics[width=\textwidth]{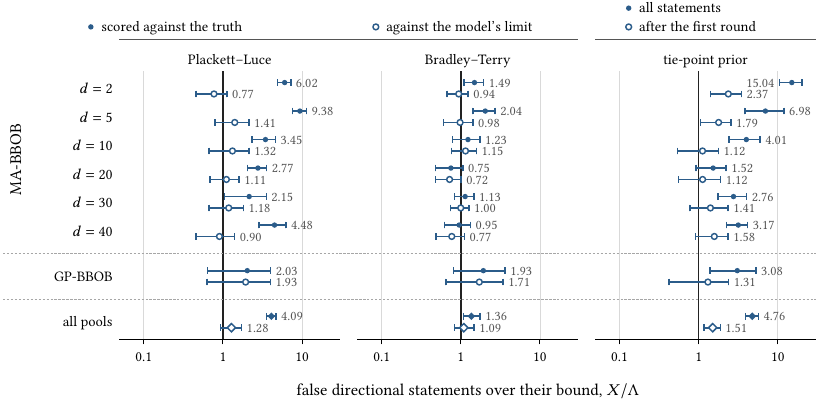}
\caption[False statements with one cause left out]{False directional statements over their bound, $X/\Lambda$, as in \cref{fig:headline} (filled) and with one cause of false statements left out (hollow). Plackett--Luce and Bradley--Terry: the same statements scored against the model's limit instead of the truth. Tie-point prior: only the statements declared after the first round. The number beside an interval is its ratio.}
\Description{Three panels on a logarithmic axis with a line at 1, each with a row per pool and one for all pools, and a filled and a hollow point with intervals in every row.}
\label{fig:why}
\end{figure}

\subsection{Comparing \textsc{PolarBear} with F-Race and B-Race}
\label{sec:racers}

We compare \textsc{PolarBear} under the per-pair model with three racers that compare algorithms at one budget, F-race and two Bayesian variants of it.
We race them first at one budget, where \textsc{PolarBear} is itself such a racer, and then at every budget of the grid, which a per-budget racer covers only with one race per budget.

The first racer is F-race~\cite{birattari2010}, as irace implements it: a Friedman test on the ranks within instances while more than two algorithms survive, and a paired Wilcoxon signed-rank test on the differences of the values once two are left.
The other two replace F-race's tests by Bayesian ones and keep its procedure; we build them for this comparison and call them \emph{B-race}.
\emph{B-race (Benavoli tests)} uses the tests of the Bayesian race of \citet{benavoli2015}, and \emph{B-race (pairwise tests)} tests the best survivor against every other with the per-pair posterior of \cref{sec:per-pair}, so that its eliminations have a bound; \cref{tab:racers-one} in \cref{app:replays} sets the three beside \textsc{PolarBear} step by step.
We race them on the six MA-BBOB pools, 20 seeds each, at a threshold of 0.99, F-race at significance 0.01, and with the cap of our races, 2{,}165 instances; triples below are in the order F-race, Benavoli tests, pairwise tests.

At one budget, $T = 1$, we take the last budget of the $d = 30$ pool, $10^{5}$ evaluations.
\textsc{PolarBear} under the per-pair model returns \alg{tpa}, the leader at that budget, in 20 of 20 races, and so do both B-races, with 0.73 and 0.77 of its evaluations.
The difference comes from when the races decide, not from the test: B-race (pairwise tests) reads the same posterior at the same threshold after every instance, while our race decides at the end of a round, 9 times per race on average, because it takes the fixed cost of a round from Plackett--Luce (\cref{sec:setup}).
Decided after every instance from the eighth on, on the instances they drew, the per-pair races end at the same instance as B-race (pairwise tests) in 16 of the 20 races and earlier in four, with 0.89 of its evaluations (\cref{app:replays}).
At one budget, \textsc{PolarBear} is thus a race of the same kind as the B-races: decided as often, it ends no later than B-race (pairwise tests), and its sets are correct within the rope.
F-race ends 16 of its 20 races at the cap, for a reason of its own, given below.

To cover every budget, a racer runs one race per budget in every repetition, 30 races on the same instances.
Per decision, the racers are at about their nominal rate of 1\,\%: 1.2\,\%, 0.5\,\% and 0.6\,\% of about 38{,}500 eliminations each are false.
What they do not do is finish: 32\,\%, 18\,\% and 21\,\% of their per-budget races end at the cap undecided, with about two survivors, so nearly every repetition leaves some budget undecided, 120, 119 and 120 of 120.
A race under the per-pair model ends with every pair resolved (\cref{sec:resolution}) in 118 of its 120 races on these pools.

The races end undecided at near-ties, budgets at which no algorithm beats every other by more than the rope.
These are 89 of the 180 budgets of the six pools.
Where the leader's margin, its smallest win probability against a rival minus $\half$, is below 0.025, the two B-races end 57\,\% and 64\,\% of their races at the cap, between 0.025 and the rope 8\,\% and 9\,\%, and outside the rope 2 and 1 of 1{,}820 (\cref{fig:racers}; \cref{fig:racers-all} in \cref{app:supplement} shows all six pools).
\textsc{PolarBear} does not separate the near-ties either: at the 14 late budgets of $d = 30$, 13 to 20 of its 20 per-pair races declare no relation between the leader and its nearest rival, and all 20 finish, since each such pair is resolved as a crossing from the budgets where it is clear (\cref{sec:resolution}).
The difference is not one of power: at a near-tie, which algorithm is best at that budget has no answer beyond the rope, and a per-budget racer has to give one at every budget, while \textsc{PolarBear} does not.

F-race also ends at the cap away from near-ties: in 330 of its 1{,}820 races outside the rope, 18\,\%, and in the 16 races at one budget.
All of these are down to two algorithms, and F-race then tests the pair with the Wilcoxon test, which ranks the absolute differences of the two algorithms' values and compares the rank sums of the two signs, so it weighs every win by its size and is not scale-free (\cref{sec:scale-free}).
On these pools the sizes nearly cancel what the win probability shows: at the last budget of $d = 30$, \alg{tpa} wins against \alg{ncsa} with probability 0.548, and its wins hold 0.485 of the rank sum.
With a sign test in the Wilcoxon test's place, at the same significance and on the same instances, 329 of the 330 races finish, every one for the algorithm with the higher win probability.
The Wilcoxon test also makes 333 of F-race's 471 false eliminations.
With two survivors, F-race thus tests which algorithm wins by more, not which wins more often; scored on the win probability, its results reflect a different question, not a weaker test.
Correspondingly, the comparison with the two B-races, which test win probabilities throughout, is the more representative one (\cref{app:replays}).

A per-budget racer's set is the union of its survivors over the 30 races of a repetition, one per budget.
It is correct within the rope in 115, 95 and 97 of their 120 repetitions; it never holds an algorithm outside $\POeps$, but misses a member of $\PO$ in 5, 25 and 23.
For the two B-races, 19 of these misses each are \alg{ncsa} at $d = 2$, which leads at no budget (\cref{sec:truth}).
The sets of \textsc{PolarBear} under the per-pair model miss no member.
We make no claim in evaluations, since the cap truncates the racers' cost; in instances, the longest race of a repetition uses all 2{,}165 in 120, 119 and 120 of 120, against a median of 672 to 1{,}369 for a per-pair race.
At one budget, \textsc{PolarBear} is a race like the B-races; across budgets, it reaches a complete answer within a cap that the per-budget racers exhaust undecided.

\begin{figure}[tbp]
\centering
\includegraphics[width=\textwidth]{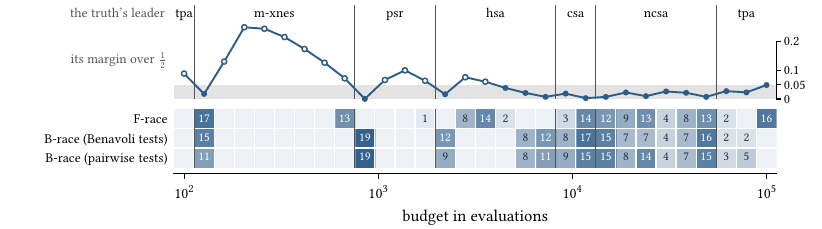}
\caption[Where per-budget racers end undecided]{The $d = 30$ pool by budget. Top: the truth's leader, with a rule at every change of leader. Middle: the leader's margin over $\half$, with the rope as a band and a filled mark where the margin is inside it. Bottom: for each racer, the number of its 20 races at that budget that end at the cap; a cell without a number has none. All 20 races of \textsc{PolarBear} under per-pair on this pool finish, after 213 to 852 instances.}
\Description{The names of the leading algorithms over thirty budgets, a strip with the leader's margin and the rope, and three rows of thirty shaded squares with numbers, one row per racer.}
\label{fig:racers}
\end{figure}

\subsection{Selection from the Pareto Set}
\label{sec:selection-results}

A user selects from the returned set under a preference (\cref{sec:selection}), and we check these selections against the truth of each pool.
We first use selection on win probabilities, the only way of \cref{sec:selection} whose truth needs no model, for all five models on all seven pools.
A preference is a weight for every budget, the weights summing to one, so the preferences form a simplex whose vertices put all their weight on one budget.
We take its center, the uniform budget preference, and its two vertices at the first and at the last budget of the grid as examples here; with 140 races per model, that gives 420 selections per model.

Under a preference, the truth's selection is the Smith set of ``is preferred'' among the Pareto set (\cref{sec:selection}), and in 20 of the 21 combinations of pool and preference it is a single algorithm.
A selection is \emph{strictly correct} when it lies in this set, and \emph{correct within the rope} when no algorithm of the pool is preferred to it by more than the rope.
With its selection, a model states a value that the selection is correct: a posterior probability for the rating models and the lower bound $L_i$ for the pair models.
We score it as the statements of \cref{sec:five-models}: $X$ counts the selections that are not strictly correct and $\Lambda$ sums one minus the stated values, with a lower bound below zero counted as zero, and with 95\,\% bootstrap intervals over races.

\Cref{fig:selection} gives the ratios $X/\Lambda$ by preference and over all three.
Over all three, Plackett--Luce is at 2.42 [2.09, 2.79] times its bound and Bradley--Terry at 1.34 [1.12, 1.58], while the per-pair model, the tie-point prior and the near-ignorance prior stay below theirs, at 0.64 [0.48, 0.82], 0.60 [0.46, 0.75] and 0.47 [0.34, 0.59].
Few selections are wrong by more than the rope, 5, 13, 2, 2 and 0 of the 420 of each model; the others that are not strictly correct are choices among algorithms that the truth separates by less than the rope.

By preference, the pair models sit at their bound at the first budget and well below it under the other two, while Plackett--Luce is over its bound under every preference and Bradley--Terry under the uniform budget preference.
The stated probability of the rating models thus overstates how often their selections are strictly correct: under the uniform budget preference, Plackett--Luce states 0.96 on average and is strictly correct in 0.81 of its races (\cref{tab:selection}).
The tie-point prior stays within its bound here, at 0.60, although its directional statements are at 4.76 times their bound (\cref{sec:five-models}).

The misses gather where the truth's selection is preferred to every other algorithm by a small margin.
At the first budget of $d = 5$ this margin is 0.003: Plackett--Luce selects \alg{m-xnes} in none of its 20 races, having eliminated it in 10 of them (\cref{sec:why}), and the pair models in 0.35 to 0.40 of theirs, against a mean stated value of 0.43 to 0.49.
At the first budget of $d = 20$, with a margin of 0.0035, the pair models are strictly correct in 0.45 of their races, against a mean stated value of 0.73 to 0.76.
Each pool fixes one truth, which the guarantee does not cover (\cref{sec:guarantee}), and near a tie such a truth can make selections wrong more often than their stated values say.

For the rating models we also score selection on utilities and on ratings against the same truth (\cref{tab:selection-ways}).
Under a preference with all its weight on one budget, the three ways agree in every race.
Under the uniform budget preference, selection on utilities selects another algorithm than selection on win probabilities in 2 and 1 of the 140 races under Plackett--Luce and Bradley--Terry, and selection on ratings in 8 and 15, where under Bradley--Terry it is strictly correct in 0.75 of the races against 0.84.
It answers another question, the probability of ranking first among the returned set (\cref{sec:selection}), so it can depart from the truth on win probabilities where near-copies share the set, as in the running example of \cref{sec:polarbear}.
Selection is therefore no advantage of the rating models; what they offer beyond the pair models is a value per algorithm, with its posterior.

\begin{figure}[tbp]
\centering
\includegraphics[width=\textwidth]{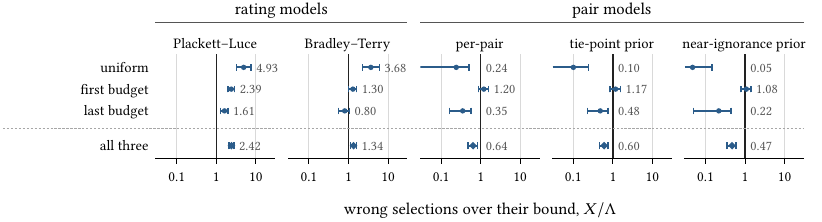}
\caption[Wrong selections over their bound]{Wrong selections over their bound, $X/\Lambda$, over all seven pools, where $\Lambda$ is the sum of one minus the stated values: under each preference and under all three, 420 selections per model. A lower bound $L_i$ below zero counts as zero. A selection counts as wrong when it is not strictly correct. The number beside an interval is its ratio.}
\Description{Five panels, one per model, on a logarithmic axis with a line at 1, each with a row per preference and one for all three, and a point with its interval in every row.}
\label{fig:selection}
\end{figure}

\section{Related Work}
\label{sec:related}

We discuss related work in five groups, from Bayesian ranking of algorithms to per-instance selection.
\Cref{tab:related} sets the methods beside \textsc{PolarBear} on six properties: whether a method reads only the order of results within an instance, how it treats budgets, whether its sample size is decided during the experiment, its error statement, its notion of equivalence, and whether it needs optima or bounds.

\begin{table}[tbp]
\caption[Related work by six properties]{Related work by six properties: whether the method reads only the order of results within an instance, how it treats budgets, whether the sample size is decided during the experiment, the error statement it makes, its notion of equivalence, and whether it needs optima or bounds.}
\label{tab:related}
\scriptsize
\setlength{\tabcolsep}{3pt}
\renewcommand{\arraystretch}{1.3}
\begin{tabular}{@{}L{58pt}L{47pt}L{47pt}L{49pt}L{62pt}L{56pt}L{38pt}@{}}
\toprule
Work & Scale-free & Over budgets & Adaptive stopping & Error statement & Equivalence & Needs optima or bounds \\
\midrule
\citet{calvo2018,calvo2019} & yes & each budget on its own & no & posterior, credible intervals & no & no; targets for fixed-target analysis \\
\citet{rojas-delgado2022} & yes & no & no & posterior, credible intervals & no & no \\
\citet{benavoli2017} & no: differences of values & no & no & posterior probabilities & rope on the difference of values & no \\
\midrule
\citet{jesus2020} & no: quality values & yes: a utility over time and quality & no & none & no & a common quality scale \\
\citet{fawcett2023}, \citet{rook2024} & no: aggregate scores & no & no & corrected tests within a group of the ranking & statistical ties & no \\
\citet{vermetten2025} & no: normalized indicator & a ranking per budget & no & a level per ranking & statistical ties & bounds and a reference point \\
\citet{ravber2024} & yes: outcomes of runs & a rating per budget & no & a rating interval per budget & ties as outcomes & error values \\
\citet{yan2022} & no: success at a precision & profiles over cost & no & none & no & known optima \\
\midrule
\citet{campelo2020} & no: differences of means & no & runs per instance; instances by formula & family-wise level and power & smallest relevant effect, for planning & no \\
\citet{eftimov2025} & no: skewness of values & no & runs, until the skewness is small & none & no & no \\
\midrule
F-race \cite{birattari2010} & yes, but for the last two candidates & no & one survivor or a cap & a level per test & no & no \\
irace \cite{lopez-ibanez2016} & with the Friedman test & no & few survivors or a budget & a level per test, as a heuristic & no & no \\
Bayesian race \cite{benavoli2015} & yes & no & a cap & a joint posterior threshold per test & within 0.05 of a tie: one is removed & no \\
SPRINT-Race \cite{zhang2015} & yes & no: objectives at one budget & every pair decided & a bound on any wrong decision & indifference zone: both stay & no \\
MO-irace \cite{bracamonte-espinoza2026} & yes: mean ranks within instances & no: objectives of one run & few survivors or a budget & none & no & no \\
Hyperband \cite{li2018} & order of losses per budget & successive budgets, the final loss as target & no: resources fixed & a bound on the returned final loss & no & a maximum resource \\
Utilitarian configuration \cite{graham2023,graham2025,graham2026} & no: utility of runtime & no: runtime to a solution & yes & a guarantee at any time & no & no; a utility of runtime \\
\midrule
\textsc{PolarBear} & yes & yes: crossings and a joint test & when every pair is resolved & a bound on false statements of each kind, when the model is exactly specified & rope on the win probability: both stay & no \\
\bottomrule
\end{tabular}
\end{table}

\paragraph{Bayesian ranking of algorithms}
\citet{calvo2018,calvo2019} introduced the Bayesian Plackett--Luce model for comparing optimization algorithms.
It ranks the algorithms on each instance or run, places a Dirichlet prior on their ratings, and reports each algorithm's probability of ranking first, which they call its probability of being the best, with its posterior mean and credible interval~\cite{calvo2019}.
Their follow-up paper already looks across budgets: besides a fixed-target analysis, it estimates the model at each of several budgets and tracks every algorithm's probability of ranking first as the budget grows~\cite{calvo2019}.
Each budget is analyzed on its own, on a sample of fixed size, and the paper notes that the number of runs needed ``is hard to say'' and that with Bayesian statistics one can observe empirically how the uncertainty shrinks as runs are added.
\citet{rojas-delgado2022} extend this line to the Bradley--Terry and Mallows models and to further posterior summaries: the probability that one algorithm outperforms another, and that it is among the best $k$.
We take Plackett--Luce with its prior (\cref{sec:pl}) and Bradley--Terry (\cref{sec:bt}) from this line.
Our estimand is the pair's win probability rather than the probability of ranking first, the latter of which changes with the set of compared algorithms (\cref{sec:scale-free}); our statements, an elimination or a crossing, cover all budgets at once; the sample size is decided during the experiment, with a guarantee; and three of our five models estimate each pair without a rating scale (\cref{sec:per-pair,sec:dp}).

\citet{benavoli2017} give Bayesian tests for comparing classifiers across data sets, with a region of practical equivalence of one percentage point on the difference of accuracies, a width they note other domains may set differently.
Equivalence is one of the decisions their tests can reach, beside either classifier being the better.
Their signed-rank test reads differences of values, and all their tests treat samples of fixed size.
Our rope follows theirs and that of \citet{kruschke2018}, moved to the win probability, where it needs no unit of the measured values.

\paragraph{Anytime evaluation and ranking over budgets}
The empirical attainment function, its ECDF and the AOCC are discussed in \cref{sec:intro} and applied in \cref{sec:truth}, and runtimes in \cref{sec:scale-free}.
\citet{jesus2020} share our goal most closely: a choice among anytime algorithms under a utility that becomes known only when the choice is made.
They describe an algorithm by its performance profile, the probability that a run reaches a quality by a time, weight the profile by a utility function over time and quality, and choose the algorithm with the largest weighted volume; for a new instance, the profile is predicted from the runs on similar training instances.
Their measure reads how much better a quality is, so it needs a quality scale that is common across instances; they name the quality relative to the optimum as one possibility and use it in their study.
A preference here weights budgets only and is applied to win probabilities, and their choice is made per instance, from fixed data and without an error statement.

Robust ranking groups solvers into statistically tied sets: \citet{fawcett2023} resample the runs behind competition scores and bound the family-wise error with Holm--Bonferroni within each group, \citet{rook2024} replace the score by a vector of performance measures and sort the resampled vectors by non-domination, and \citet{vermetten2025} compute such a ranking at each of several budgets.
\citet{ravber2024} turn runs into wins, losses and ties at each of 16 budgets by comparing their function error values, the distance to the known optimum, and compute a Glicko-2 rating with its deviation at each, reading intervals that do not overlap as significant differences.
Since both runs of a comparison share the optimum, a win or a loss records only which run is ahead, as ours do, but the ratings built from these outcomes form a rating scale, as in \cref{sec:pl,sec:bt}.
These methods give a ranking or a rating at each budget, each with its own error statement, and none for the budgets together, whereas a pair here has a relation at each budget and one statement over all of them (\cref{sec:tests}).
A rank within a set also changes with the set, and the pair's win probability does not (\cref{sec:scale-free}).

\citet{yan2022} study two paradoxes of profile-based comparison, cycle ranking and the survival of the non-fittest, in which comparing two algorithms at a time and all at once disagree.
They trace both to success criteria that depend on the compared algorithms, such as the best value any of them found, and prove that criteria anchored to the known optimum exclude both paradoxes from profile methods with a binary success criterion.
A pair's win probability depends on no other algorithm and needs no anchor.
A cycle can still occur in it; it is then a property of the algorithms on that distribution, not of the analysis, and the Pareto set is defined so that it keeps its meaning (\cref{sec:scale-free}).

\paragraph{Experiment design and stopping}
\citet{campelo2020} size an experiment before it runs: the number of instances follows from the power to detect a smallest relevant effect at a family-wise level under Holm's correction.
Within an instance, runs are added until the standard errors of the paired differences reach a target or a cap.
They plan for a test on differences of means, whereas here no sample size is fixed in advance and a race stops when every pair is resolved.
Their smallest relevant effect and our rope answer the same question, which difference matters, but the rope is also a possible outcome, an equivalence, while their effect size serves only to plan the experiment.

\citet{eftimov2025} choose the number of runs of one algorithm on one instance during the experiment.
After five runs, they remove outliers, compute the skewness of the remaining results, and add a run until the skewness lies within a tolerance of zero; their hypothesis is that a symmetric sample represents the algorithm's performance, and they evaluate it against samples of 50 runs.
Distributions of results are often bounded, heavy-tailed or mixtures, however, and for these a nonzero skewness is expected.
Skewness is a property of the distribution, not of the sample size: more runs estimate it better and do not bring it to zero.
Our criterion is the decision instead: a race adds instances until every pair is resolved at the threshold $\alpha$, and if the model is exactly specified, the posterior keeps its meaning under this stopping rule, which is the basis of the guarantee (\cref{sec:guarantee}).

\paragraph{Racing and configuration}
F-race~\cite{birattari2010} evaluates the surviving candidates on one instance after another and tests after each: a Friedman test on the ranks within instances and, after a rejection, the elimination of the candidates that are significantly worse than the best.
Between the last two candidates, it uses the Wilcoxon signed-rank test, which reads differences of values.
It stops at one survivor, at a maximum number of instances or when its budget is spent; a level holds for each test, and nothing is stated for the sequence of tests.
irace~\cite{lopez-ibanez2016} iterates such races, sampling new candidates between them, and offers the Friedman test and a $t$-test; it uses its tests ``as a selection heuristic'' and ``does not attempt to preserve the statistical significance level''.

The Bayesian race of \citet{benavoli2015} replaces F-race's tests by a Bayesian Friedman test under a Dirichlet process and, after a rejection, a joint comparison that accepts pairwise statements in decreasing order of posterior probability while their joint probability stays above a threshold.
It also removes algorithms that are indistinguishable, with a win probability within 0.05 of $\half$, and makes no claim for the race as a whole.
We take the Dirichlet process posterior from it, as two of our five models (\cref{sec:dp}), and equivalence on the scale of the win probability.
\Cref{sec:racers} builds a racer from its tests and races against it.

SPRINT-Race~\cite{zhang2015} is the closest in structure: a sequential test for each pair has three outcomes, that one dominates, that the other dominates, or neither, within an indifference zone on the probability scale, in which case both are kept; a Bonferroni bound over all pairs holds for the probability of any wrong decision; and the race ends when every pair is decided.
Its dominance is between performance vectors on one instance, in objective space, at one budget.
MO-irace~\cite{bracamonte-espinoza2026} keeps irace's race and replaces its tests by Pareto dominance: each configuration is ranked within every instance on each objective, the ranks are averaged over the instances seen, and a configuration is eliminated once another's mean ranks are at least as good on every objective.
No test enters, so it makes no error statement, and the authors trace premature eliminations to marginal differences of mean ranks.
In both, the objectives are measures of one run, such as solution cost and runtime.
Here the budgets take the place of objectives, and $A$ dominates $B$ when it wins on most instances at each budget, not when it is ahead at every budget on most instances; our guarantee bounds the expected number of false statements by their posterior probabilities, not the probability of any wrong decision (\cref{sec:guarantee}).

Successive halving and Hyperband~\cite{li2018} promote the candidates with the lowest loss at a small budget to a larger one.
They are defined on a single problem instance: a candidate enters only through its sequence of losses, one per budget, and their target is the final loss.
The procedure reads only the order of losses at a budget, so what a small budget has to carry is that order: the candidate that is best at the end must already be among the best early.
Hyperband hedges over how early, across brackets, one of which gives every candidate the full budget.
An algorithm that is ahead only late survives through the conservative brackets; one that is ahead only early is not the target, while here both are members of the Pareto set.

Utilitarian configuration~\cite{graham2023,graham2025,graham2026} asks the user for a utility of runtime, bounded and decreasing, which can express a cost of computing or a time after which an answer is no longer useful.
On a stream of instances, it returns a configuration whose expected utility is within a tolerance of the best, with a stated probability.
Neither the tolerance nor a cap on a run is set in advance: the tolerance it can guarantee shrinks as the search continues~\cite{graham2023}, and in the later form the set of configurations it covers grows as well~\cite{graham2025}.
It shares with our approach the explicit preference over time and the guarantee under continued sampling.
Its measure, however, is the runtime to a solution under a utility fixed in advance, for algorithms whose runs end with one, such as SAT and integer-program solvers; here it is the quality reached at each budget, and the preference may stay open until after the comparison.

A pair's win probability does not change when an algorithm joins the comparison, so algorithms can join a race under way (\cref{sec:procedure}).
Ranks within instances do change, and with them the Friedman statistic; the candidate sets within an F-race are nested, and new candidates enter only between races.
Plackett--Luce's estimate of a pair is unchanged only under its assumption (\cref{sec:pl}), while the models that estimate each pair leave it unchanged by construction.
Dueling bandits, too, work on pairs: on each trial they choose a pair and observe which of the two is preferred, and since an arm that beats every other need not exist, they define the winners by tournament solutions, among them the top cycle of \cref{sec:scale-free}~\cite{ramamohan2016}.
They assume that no pair is tied and minimize regret during sampling, whereas a race identifies a set with recorded probabilities and decides near-ties by the rope.

\paragraph{Per-instance selection}
Choosing an algorithm for each instance from the instance's features is a different task~\cite{kerschke2019,munoz2015}: \citet{oentaryo2015} pose it as learning a ranking of solvers from instance and solver features, and \citet{meunier2022} select from descriptors known beforehand, such as the dimension, the budget, the noise and the degree of parallelism.
We compare algorithms on one distribution of instances and use no features.
Our result is nonetheless a selector, with one descriptor, the preference over budgets: the posterior and a way of selection return an algorithm for any preference (\cref{sec:selection}), a known budget among them, for an instance of which nothing is known but the distribution it is drawn from.
The two also combine: features would be further conditions of the race, each with its own Pareto set (\cref{sec:conclusion}).
Where the budget is a descriptor, as for \citet{meunier2022}, the anytime Pareto set names the algorithms a selector has to consider at any budget, since every algorithm that leads at some budget belongs to it (\cref{sec:scale-free}).

\FloatBarrier

\section{Conclusion}
\label{sec:conclusion}

We proposed \textsc{PolarBear}, a Bayesian racing procedure that compares a set of anytime optimization algorithms through the win probabilities of all their pairs at each budget.
The win probability needs no normalization of objective values and is unchanged under a strictly increasing transformation of the values of each instance.
When the budget available at deployment is unknown, a comparison should deliver the anytime Pareto set, the algorithms that no other beats at every budget, from which a user selects once a preference over budgets is known.
The race identifies this set over any posterior on pairs: it runs only the algorithms whose pairs are still open, decides when to stop, and makes no more false statements than its posterior probabilities imply, on average over the prior of an exactly specified model.
On seven pools of recorded algorithm runs, whose win probabilities are known exactly, the models that estimate each pair make no more false statements than their posterior probabilities imply unless their prior is certain too early, while the models that compress the pairs onto one rating scale per budget make more, mostly because the truth follows such a scale only approximately, not through errors of estimation.
The number of instances a race needs varies severalfold between repeated races on the same pool, so no fixed sample size suits them all.
Across budgets, \textsc{PolarBear} reaches a complete answer where racers that race each budget on its own, given the same number of instances, end undecided at near-ties, while at one budget it is a race of the same kind as they are.
Selections from the returned set by the models that estimate each pair are wrong no more often than the lower bounds they state on the probability of a correct selection imply, and few selections of any model are wrong by more than the region of practical equivalence.

\paragraph{Limitations}
The comparison is made on one distribution of instances and reads no features of them: an algorithm that suits a part of the instances, recognizable by their features, can still be dominated over the distribution as a whole and leave the race.
The guarantee holds on average over the prior of an exactly specified model, not for the one fixed distribution of instances that a benchmarking study compares on; there, as on our pools, the false statements can only be measured (\cref{sec:guarantee,sec:five-models}).
On the pools, the rating models make more false statements than their posterior probabilities imply, and so does the tie-point prior in its first round and where ties are frequent (\cref{sec:five-models,sec:why}).
No model here lets the evidence at one budget inform the neighboring ones with calibrated statements: a prior under which the ratings change smoothly with the budget turns overconfident where outcomes depend on each other across budgets, as the runs on real instances do (\cref{sec:persistence}).
With many algorithms, a race has quadratically many pairs to decide, and the bound on false eliminations grows linearly, to 0.99 for 100 algorithms at a threshold of 0.99 (\cref{sec:scalability}).
A race is slowest on a pair as far from a directional statement as from an equivalence: the narrower the region of practical equivalence, which the user chooses, the more instances such a pair needs (\cref{sec:procedure}).
In our experiments, a race stops at a cap of 2{,}165 instances; 24 of the 700 races on the pools end there undecided, and since the cap also truncates the cost of racing each budget on its own, we compare such racers with \textsc{PolarBear} in instances only, not in evaluations (\cref{sec:setup,sec:five-models,sec:racers}).
We do not measure what a race would cost if it decided every pair at every budget instead of resolving it as a crossing.
Racers that race each budget on its own have to decide every pair at every budget, and they end undecided at near-ties within the same cap (\cref{sec:racers}).
The pointwise test was run on two pools only (\cref{sec:five-models}).

\paragraph{Future work}
There are several directions for future work.
Instance features can enter the race as further conditions, each feature value with its own Pareto set, so that the result of \textsc{PolarBear} is an algorithm selector over instance features and budget preferences (\cref{sec:related}).
The allocation policy is not tuned for posteriors that cost almost nothing to compute: at one budget, a decision after every instance shortens the races under the per-pair model (\cref{sec:racers}).
For Bradley--Terry, which compares one pair per instance, the policy counts every decision alike, although one decision can eliminate an algorithm and resolve its other pairs with it.
A race could also run asynchronously, continuing on instances allocated by the round before while the next round is computed, so that more runs execute in parallel and none waits for a decision, at the cost of some runs that the new round would not have asked for (\cref{sec:procedure}).
Since algorithms can join a race under way, \textsc{PolarBear} could serve as the objective of automated algorithm design, in which a design loop adds the algorithms it proposes to a running race and aims for membership in the returned set (\cref{sec:procedure}).
The probability that one algorithm at one budget beats another at a different budget is scale-free as well (\cref{sec:scale-free}), and would compare algorithms at unequal budgets, such as one algorithm at budget $t$ against another at $2t$, which the win probability, its diagonal, cannot.
From it, one could derive the budget factor at which two algorithms are even, a scale-free counterpart of a runtime ratio; like a runtime, it is unbounded, since it may lie beyond the grid or not exist at all, which makes it harder to estimate.
The comparison applies wherever anytime outcomes can be ordered on an instance, such as the returns of reinforcement learning agents over training time.
Finally, multi-objective optimization does not order its outcomes totally, since neither of two fronts need dominate the other, so a pair has a third outcome, incomparable, which is not a tie.
Its decisions then need new rules, since a pair can be incomparable on most instances while one front dominates the other on more of the rest.

\paragraph{Data availability}
The implementation of \textsc{PolarBear}, the records of every race, the truth of every pool and the scripts that compute the numbers and figures of this paper are available at \url{https://doi.org/10.5281/zenodo.24571091}.
The recorded runs of the pools exceed the archive's size limit; in their place, the truth of every pool holds the best-so-far value of each of its runs at every budget of the grid, from which its win probabilities are computed, so the scoring of every race can be repeated from the archive.
The MA-BBOB pools can be rebuilt from the archive; the GP-BBOB pool cannot be rebuilt exactly, since its budget is algorithm time measured on our hardware.

\begin{acks}
The authors gratefully acknowledge the resources on the LiCCA HPC cluster of the University of Augsburg, co-funded by the Deutsche Forschungsgemeinschaft (DFG, German Research Foundation) -- Project-ID 499211671.
\end{acks}

\bibliographystyle{ACM-Reference-Format}
\bibliography{references}


\begin{thebibliography}{79}


\ifx \showCODEN    \undefined \def \showCODEN     #1{\unskip}     \fi
\ifx \showISBNx    \undefined \def \showISBNx     #1{\unskip}     \fi
\ifx \showISBNxiii \undefined \def \showISBNxiii  #1{\unskip}     \fi
\ifx \showISSN     \undefined \def \showISSN      #1{\unskip}     \fi
\ifx \showLCCN     \undefined \def \showLCCN      #1{\unskip}     \fi
\ifx \shownote     \undefined \def \shownote      #1{#1}          \fi
\ifx \showarticletitle \undefined \def \showarticletitle #1{#1}   \fi
\ifx \showURL      \undefined \def \showURL       {\relax}        \fi
\providecommand\bibfield[2]{#2}
\providecommand\bibinfo[2]{#2}
\providecommand\natexlab[1]{#1}
\providecommand\showeprint[2][]{arXiv:#2}

\bibitem[Ait~Elhara et~al\mbox{.}(2013)]%
        {aitelhara2013}
\bibfield{author}{\bibinfo{person}{Ouassim Ait~Elhara}, \bibinfo{person}{Anne
  Auger}, {and} \bibinfo{person}{Nikolaus Hansen}.}
  \bibinfo{year}{2013}\natexlab{}.
\newblock \showarticletitle{A Median Success Rule for Non-Elitist Evolution
  Strategies: Study of Feasibility}. In \bibinfo{booktitle}{\emph{Proceedings
  of the 15th Annual Conference on {{Genetic}} and Evolutionary Computation}}.
  \bibinfo{publisher}{ACM}, \bibinfo{pages}{415--422}.
\newblock
\href{https://doi.org/10.1145/2463372.2463429}{doi:\nolinkurl{10.1145/2463372.2463429}}


\bibitem[Akimoto(2023)]%
        {akimoto2023}
\bibfield{author}{\bibinfo{person}{Youhei Akimoto}.}
  \bibinfo{year}{2023}\natexlab{}.
\newblock \bibinfo{title}{{{PSA-CMA-ES}}: {{CMA-ES}} with {{Population Size
  Adaptation}}}.
\newblock \bibinfo{howpublished}{Computer software, GitHub Gist, commit
  fb08361}.
\newblock
\urldef\tempurl%
\url{https://gist.github.com/youheiakimoto/f342a139a6dfcfd55dbc5e345dbf4f9f}
\showURL{%
\tempurl}


\bibitem[Akimoto et~al\mbox{.}(2014)]%
        {akimoto2014}
\bibfield{author}{\bibinfo{person}{Youhei Akimoto}, \bibinfo{person}{Anne
  Auger}, {and} \bibinfo{person}{Nikolaus Hansen}.}
  \bibinfo{year}{2014}\natexlab{}.
\newblock \showarticletitle{Comparison-Based Natural Gradient Optimization in
  High Dimension}. In \bibinfo{booktitle}{\emph{Proceedings of the 2014
  {{Annual Conference}} on {{Genetic}} and {{Evolutionary Computation}}}}.
  \bibinfo{publisher}{ACM}, \bibinfo{pages}{373--380}.
\newblock
\href{https://doi.org/10.1145/2576768.2598258}{doi:\nolinkurl{10.1145/2576768.2598258}}


\bibitem[Akimoto and Hansen(2016a)]%
        {akimoto2016}
\bibfield{author}{\bibinfo{person}{Youhei Akimoto} {and}
  \bibinfo{person}{Nikolaus Hansen}.} \bibinfo{year}{2016}\natexlab{a}.
\newblock \showarticletitle{Online {{Model Selection}} for {{Restricted
  Covariance Matrix Adaptation}}}.
\newblock In \bibinfo{booktitle}{\emph{Parallel {{Problem Solving}} from
  {{Nature}} -- {{PPSN XIV}}}}. \bibinfo{publisher}{Springer International
  Publishing}, \bibinfo{pages}{3--13}.
\newblock
\showISBNx{978-3-319-45823-6}
\showISSN{0302-9743}
\href{https://doi.org/10.1007/978-3-319-45823-6_1}{doi:\nolinkurl{10.1007/978-3-319-45823-6_1}}


\bibitem[Akimoto and Hansen(2016b)]%
        {akimoto2016a}
\bibfield{author}{\bibinfo{person}{Youhei Akimoto} {and}
  \bibinfo{person}{Nikolaus Hansen}.} \bibinfo{year}{2016}\natexlab{b}.
\newblock \showarticletitle{Projection-{{Based Restricted Covariance Matrix
  Adaptation}} for {{High Dimension}}}. In
  \bibinfo{booktitle}{\emph{Proceedings of the {{Genetic}} and {{Evolutionary
  Computation Conference}} 2016}}. \bibinfo{publisher}{ACM},
  \bibinfo{pages}{197--204}.
\newblock
\href{https://doi.org/10.1145/2908812.2908863}{doi:\nolinkurl{10.1145/2908812.2908863}}


\bibitem[Akimoto et~al\mbox{.}(2008)]%
        {akimoto2008}
\bibfield{author}{\bibinfo{person}{Youhei Akimoto}, \bibinfo{person}{Jun
  Sakuma}, \bibinfo{person}{Isao Ono}, {and} \bibinfo{person}{Shigenobu
  Kobayashi}.} \bibinfo{year}{2008}\natexlab{}.
\newblock \showarticletitle{Functionally Specialized {{CMA-ES}}: A Modification
  of {{CMA-ES}} Based on the Specialization of the Functions of Covariance
  Matrix Adaptation and Step Size Adaptation}. In
  \bibinfo{booktitle}{\emph{Proceedings of the 10th Annual Conference on
  {{Genetic}} and Evolutionary Computation}}. \bibinfo{publisher}{ACM},
  \bibinfo{pages}{479--486}.
\newblock
\href{https://doi.org/10.1145/1389095.1389188}{doi:\nolinkurl{10.1145/1389095.1389188}}


\bibitem[Arabas et~al\mbox{.}(2025)]%
        {arabas2025}
\bibfield{author}{\bibinfo{person}{Jaros{\l}aw Arabas}, \bibinfo{person}{Adam
  Stelmaszczyk}, \bibinfo{person}{Eryk Warchulski}, \bibinfo{person}{Dariusz
  Jagodzi{\'n}ski}, {and} \bibinfo{person}{Rafa{\l} Biedrzycki}.}
  \bibinfo{year}{2025}\natexlab{}.
\newblock \showarticletitle{Covariance {{Matrix Adaptation Evolution Strategy}}
  without a Matrix}. In \bibinfo{booktitle}{\emph{Proceedings of the
  {{Genetic}} and {{Evolutionary Computation Conference Companion}}}}.
  \bibinfo{publisher}{ACM}, \bibinfo{pages}{471--474}.
\newblock
\href{https://doi.org/10.1145/3712255.3726583}{doi:\nolinkurl{10.1145/3712255.3726583}}


\bibitem[Auger and Hansen(2005)]%
        {auger2005}
\bibfield{author}{\bibinfo{person}{A. Auger} {and} \bibinfo{person}{N.
  Hansen}.} \bibinfo{year}{2005}\natexlab{}.
\newblock \showarticletitle{A {{Restart CMA Evolution Strategy With Increasing
  Population Size}}}. In \bibinfo{booktitle}{\emph{2005 {{IEEE Congress}} on
  {{Evolutionary Computation}}}}, Vol.~\bibinfo{volume}{2}.
  \bibinfo{publisher}{IEEE}, \bibinfo{pages}{1769--1776}.
\newblock
\href{https://doi.org/10.1109/CEC.2005.1554902}{doi:\nolinkurl{10.1109/CEC.2005.1554902}}


\bibitem[Benavoli et~al\mbox{.}(2017)]%
        {benavoli2017}
\bibfield{author}{\bibinfo{person}{Alessio Benavoli}, \bibinfo{person}{Giorgio
  Corani}, \bibinfo{person}{Janez Dem{\v s}ar}, {and} \bibinfo{person}{Marco
  Zaffalon}.} \bibinfo{year}{2017}\natexlab{}.
\newblock \showarticletitle{Time for a {{Change}}: A {{Tutorial}} for
  {{Comparing Multiple Classifiers Through Bayesian Analysis}}}.
\newblock \bibinfo{journal}{\emph{Journal of Machine Learning Research}}
  \bibinfo{volume}{18}, \bibinfo{number}{77} (\bibinfo{year}{2017}),
  \bibinfo{pages}{1--36}.
\newblock
\showISSN{1533-7928}
\urldef\tempurl%
\url{http://jmlr.org/papers/v18/16-305.html}
\showURL{%
\tempurl}


\bibitem[Benavoli et~al\mbox{.}(2015)]%
        {benavoli2015}
\bibfield{author}{\bibinfo{person}{Alessio Benavoli}, \bibinfo{person}{Giorgio
  Corani}, \bibinfo{person}{Francesca Mangili}, {and} \bibinfo{person}{Marco
  Zaffalon}.} \bibinfo{year}{2015}\natexlab{}.
\newblock \showarticletitle{A {{Bayesian}} Nonparametric Procedure for
  Comparing Algorithms}. In \bibinfo{booktitle}{\emph{Proceedings of the 32nd
  {{International Conference}} on {{Machine Learning}}}}.
  \bibinfo{publisher}{PMLR}, \bibinfo{pages}{1264--1272}.
\newblock
\showISSN{1938-7228}
\urldef\tempurl%
\url{https://proceedings.mlr.press/v37/benavoli15.html}
\showURL{%
\tempurl}


\bibitem[Benavoli et~al\mbox{.}(2014)]%
        {benavoli2014}
\bibfield{author}{\bibinfo{person}{Alessio Benavoli},
  \bibinfo{person}{Francesca Mangili}, \bibinfo{person}{Giorgio Corani},
  \bibinfo{person}{Marco Zaffalon}, {and} \bibinfo{person}{Fabrizio Ruggeri}.}
  \bibinfo{year}{2014}\natexlab{}.
\newblock \showarticletitle{A {{Bayesian Wilcoxon}} Signed-Rank Test Based on
  the {{Dirichlet}} Process}. In \bibinfo{booktitle}{\emph{Proceedings of the
  31st {{International Conference}} on {{Machine Learning}}}}
  \emph{(\bibinfo{series}{Proceedings of {{Machine Learning Research}}},
  Vol.~\bibinfo{volume}{32})}. \bibinfo{publisher}{PMLR},
  \bibinfo{pages}{1026--1034}.
\newblock
\showISSN{1938-7228}
\urldef\tempurl%
\url{https://proceedings.mlr.press/v32/benavoli14.html}
\showURL{%
\tempurl}


\bibitem[Beyer and Sendhoff(2017)]%
        {beyer2017}
\bibfield{author}{\bibinfo{person}{Hans-Georg Beyer} {and}
  \bibinfo{person}{Bernhard Sendhoff}.} \bibinfo{year}{2017}\natexlab{}.
\newblock \showarticletitle{Simplify {{Your Covariance Matrix Adaptation
  Evolution Strategy}}}.
\newblock \bibinfo{journal}{\emph{IEEE Transactions on Evolutionary
  Computation}} \bibinfo{volume}{21}, \bibinfo{number}{5} (\bibinfo{date}{Oct.}
  \bibinfo{year}{2017}), \bibinfo{pages}{746--759}.
\newblock
\showISSN{1089-778X}
\href{https://doi.org/10.1109/TEVC.2017.2680320}{doi:\nolinkurl{10.1109/TEVC.2017.2680320}}


\bibitem[Birattari et~al\mbox{.}(2010)]%
        {birattari2010}
\bibfield{author}{\bibinfo{person}{Mauro Birattari}, \bibinfo{person}{Zhi
  Yuan}, \bibinfo{person}{Prasanna Balaprakash}, {and} \bibinfo{person}{Thomas
  St{\"u}tzle}.} \bibinfo{year}{2010}\natexlab{}.
\newblock \showarticletitle{F-{{Race}} and {{Iterated F-Race}}: {{An
  Overview}}}.
\newblock In \bibinfo{booktitle}{\emph{Experimental {{Methods}} for the
  {{Analysis}} of {{Optimization Algorithms}}}},
  \bibfield{editor}{\bibinfo{person}{Thomas {Bartz-Beielstein}},
  \bibinfo{person}{Marco Chiarandini}, \bibinfo{person}{Lu{\'i}s Paquete},
  {and} \bibinfo{person}{Mike Preuss}} (Eds.). \bibinfo{publisher}{Springer},
  \bibinfo{address}{Berlin, Heidelberg}, \bibinfo{pages}{311--336}.
\newblock
\showISBNx{978-3-642-02538-9}
\urldef\tempurl%
\url{https://doi.org/10.1007/978-3-642-02538-9_13}
\showURL{%
\tempurl}


\bibitem[{Bracamonte-Espinoza} et~al\mbox{.}(2026)]%
        {bracamonte-espinoza2026}
\bibfield{author}{\bibinfo{person}{Carlos {Bracamonte-Espinoza}},
  \bibinfo{person}{Nicol{\'a}s {Rojas-Morales}}, \bibinfo{person}{Carlos
  Hern{\'a}ndez}, {and} \bibinfo{person}{Diana {Gil-Silvestre}}.}
  \bibinfo{year}{2026}\natexlab{}.
\newblock \showarticletitle{{{MO}}-Irace: {{A Multi-Objective Racing
  Configurator}}}. In \bibinfo{booktitle}{\emph{Proceedings of the {{Genetic}}
  and {{Evolutionary Computation Conference Companion}}}}
  \emph{(\bibinfo{series}{{{GECCO}} '26 {{Companion}}})}.
  \bibinfo{publisher}{Association for Computing Machinery},
  \bibinfo{address}{New York, NY, USA}, \bibinfo{pages}{261--264}.
\newblock
\showISBNx{979-8-4007-2488-6}
\href{https://doi.org/10.1145/3795101.3805465}{doi:\nolinkurl{10.1145/3795101.3805465}}


\bibitem[Bradley and Terry(1952)]%
        {bradley1952}
\bibfield{author}{\bibinfo{person}{Ralph~Allan Bradley} {and}
  \bibinfo{person}{Milton~E. Terry}.} \bibinfo{year}{1952}\natexlab{}.
\newblock \showarticletitle{Rank {{Analysis}} of {{Incomplete Block Designs}}:
  The {{Method}} of {{Paired Comparisons}}}.
\newblock \bibinfo{journal}{\emph{Biometrika}} \bibinfo{volume}{39},
  \bibinfo{number}{3-4} (\bibinfo{date}{Dec.} \bibinfo{year}{1952}),
  \bibinfo{pages}{324--345}.
\newblock
\showISSN{0006-3444}
\href{https://doi.org/10.1093/biomet/39.3-4.324}{doi:\nolinkurl{10.1093/biomet/39.3-4.324}}


\bibitem[Brest et~al\mbox{.}(2006)]%
        {brest2006}
\bibfield{author}{\bibinfo{person}{Janez Brest}, \bibinfo{person}{Sao Greiner},
  \bibinfo{person}{Borko Boskovic}, \bibinfo{person}{Marjan Mernik}, {and}
  \bibinfo{person}{Viljem Zumer}.} \bibinfo{year}{2006}\natexlab{}.
\newblock \showarticletitle{Self-{{Adapting Control Parameters}} in
  {{Differential Evolution}}: {{A Comparative Study}} on {{Numerical Benchmark
  Problems}}}.
\newblock \bibinfo{journal}{\emph{IEEE Transactions on Evolutionary
  Computation}} \bibinfo{volume}{10}, \bibinfo{number}{6} (\bibinfo{date}{Dec.}
  \bibinfo{year}{2006}), \bibinfo{pages}{646--657}.
\newblock
\showISSN{1089-778X}
\href{https://doi.org/10.1109/TEVC.2006.872133}{doi:\nolinkurl{10.1109/TEVC.2006.872133}}


\bibitem[Calvo et~al\mbox{.}(2018)]%
        {calvo2018}
\bibfield{author}{\bibinfo{person}{Borja Calvo}, \bibinfo{person}{Josu
  Ceberio}, {and} \bibinfo{person}{Jose~A. Lozano}.}
  \bibinfo{year}{2018}\natexlab{}.
\newblock \showarticletitle{Bayesian Inference for Algorithm Ranking Analysis}.
  In \bibinfo{booktitle}{\emph{Proceedings of the {{Genetic}} and
  {{Evolutionary Computation Conference Companion}}}}
  \emph{(\bibinfo{series}{{{GECCO}} '18})}. \bibinfo{publisher}{Association for
  Computing Machinery}, \bibinfo{address}{New York, NY, USA},
  \bibinfo{pages}{324--325}.
\newblock
\showISBNx{978-1-4503-5764-7}
\href{https://doi.org/10.1145/3205651.3205658}{doi:\nolinkurl{10.1145/3205651.3205658}}


\bibitem[Calvo et~al\mbox{.}(2019)]%
        {calvo2019}
\bibfield{author}{\bibinfo{person}{Borja Calvo}, \bibinfo{person}{Ofer~M.
  Shir}, \bibinfo{person}{Josu Ceberio}, \bibinfo{person}{Carola Doerr},
  \bibinfo{person}{Hao Wang}, \bibinfo{person}{Thomas B{\"a}ck}, {and}
  \bibinfo{person}{Jose~A. Lozano}.} \bibinfo{year}{2019}\natexlab{}.
\newblock \showarticletitle{Bayesian Performance Analysis for Black-Box
  Optimization Benchmarking}. In \bibinfo{booktitle}{\emph{Proceedings of the
  {{Genetic}} and {{Evolutionary Computation Conference Companion}}}}
  \emph{(\bibinfo{series}{{{GECCO}} '19})}. \bibinfo{publisher}{Association for
  Computing Machinery}, \bibinfo{address}{New York, NY, USA},
  \bibinfo{pages}{1789--1797}.
\newblock
\showISBNx{978-1-4503-6748-6}
\href{https://doi.org/10.1145/3319619.3326888}{doi:\nolinkurl{10.1145/3319619.3326888}}


\bibitem[Campelo and Wanner(2020)]%
        {campelo2020}
\bibfield{author}{\bibinfo{person}{Felipe Campelo} {and}
  \bibinfo{person}{Elizabeth~F. Wanner}.} \bibinfo{year}{2020}\natexlab{}.
\newblock \showarticletitle{Sample Size Calculations for the Experimental
  Comparison of Multiple Algorithms on Multiple Problem Instances}.
\newblock \bibinfo{journal}{\emph{Journal of Heuristics}} \bibinfo{volume}{26},
  \bibinfo{number}{6} (\bibinfo{date}{Dec.} \bibinfo{year}{2020}),
  \bibinfo{pages}{851--883}.
\newblock
\showISSN{1572-9397}
\href{https://doi.org/10.1007/s10732-020-09454-w}{doi:\nolinkurl{10.1007/s10732-020-09454-w}}


\bibitem[De~Heide and Gr{\"u}nwald(2020)]%
        {deheide2020}
\bibfield{author}{\bibinfo{person}{Rianne De~Heide} {and}
  \bibinfo{person}{Peter~D. Gr{\"u}nwald}.} \bibinfo{year}{2020}\natexlab{}.
\newblock \showarticletitle{Why Optional Stopping Can Be a Problem for
  {{Bayesians}}}.
\newblock \bibinfo{journal}{\emph{Psychonomic Bulletin \& Review}}
  \bibinfo{volume}{28}, \bibinfo{number}{3} (\bibinfo{date}{Nov.}
  \bibinfo{year}{2020}), \bibinfo{pages}{795--812}.
\newblock
\showISSN{1069-9384}
\href{https://doi.org/10.3758/s13423-020-01803-x}{doi:\nolinkurl{10.3758/s13423-020-01803-x}}


\bibitem[{de Nobel} et~al\mbox{.}(2021)]%
        {denobel2021}
\bibfield{author}{\bibinfo{person}{Jacob {de Nobel}},
  \bibinfo{person}{Diederick Vermetten}, \bibinfo{person}{Hao Wang},
  \bibinfo{person}{Carola Doerr}, {and} \bibinfo{person}{Thomas B{\"a}ck}.}
  \bibinfo{year}{2021}\natexlab{}.
\newblock \showarticletitle{Tuning as a Means of Assessing the Benefits of New
  Ideas in Interplay with Existing Algorithmic Modules}. In
  \bibinfo{booktitle}{\emph{Proceedings of the {{Genetic}} and {{Evolutionary
  Computation Conference Companion}}}} \emph{(\bibinfo{series}{{{GECCO}}
  '21})}. \bibinfo{publisher}{Association for Computing Machinery},
  \bibinfo{address}{New York, NY, USA}, \bibinfo{pages}{1375--1384}.
\newblock
\showISBNx{978-1-4503-8351-6}
\href{https://doi.org/10.1145/3449726.3463167}{doi:\nolinkurl{10.1145/3449726.3463167}}


\bibitem[Eftimov and Koro{\v s}ec(2025)]%
        {eftimov2025}
\bibfield{author}{\bibinfo{person}{Tome Eftimov} {and} \bibinfo{person}{Peter
  Koro{\v s}ec}.} \bibinfo{year}{2025}\natexlab{}.
\newblock \showarticletitle{Adaptive {{Estimation}} of the {{Number}} of
  {{Algorithm Runs}} in {{Stochastic Optimization}}}.
\newblock In \bibinfo{booktitle}{\emph{Proceedings of the {{Genetic}} and
  {{Evolutionary Computation Conference}}}}. \bibinfo{publisher}{Association
  for Computing Machinery}, \bibinfo{address}{New York, NY, USA},
  \bibinfo{pages}{719--727}.
\newblock
\showISBNx{979-8-4007-1465-8}
\urldef\tempurl%
\url{https://doi.org/10.1145/3712256.3726327}
\showURL{%
\tempurl}


\bibitem[Fang and Santos(2019)]%
        {fang2019}
\bibfield{author}{\bibinfo{person}{Zheng Fang} {and} \bibinfo{person}{Andres
  Santos}.} \bibinfo{year}{2019}\natexlab{}.
\newblock \showarticletitle{Inference on {{Directionally Differentiable
  Functions}}}.
\newblock \bibinfo{journal}{\emph{The Review of Economic Studies}}
  \bibinfo{volume}{86}, \bibinfo{number}{1} (\bibinfo{date}{Jan.}
  \bibinfo{year}{2019}), \bibinfo{pages}{377--412}.
\newblock
\showISSN{0034-6527}
\href{https://doi.org/10.1093/restud/rdy049}{doi:\nolinkurl{10.1093/restud/rdy049}}


\bibitem[Fawcett et~al\mbox{.}(2023)]%
        {fawcett2023}
\bibfield{author}{\bibinfo{person}{Chris Fawcett}, \bibinfo{person}{Mauro
  Vallati}, \bibinfo{person}{Holger~H. Hoos}, {and} \bibinfo{person}{Alfonso~E.
  Gerevini}.} \bibinfo{year}{2023}\natexlab{}.
\newblock \bibinfo{title}{Competitions in {{AI}} -- {{Robustly Ranking Solvers
  Using Statistical Resampling}}}.
\newblock
\showeprint[arxiv]{2308.05062}~[cs]
\href{https://doi.org/10.48550/arXiv.2308.05062}{doi:\nolinkurl{10.48550/arXiv.2308.05062}}


\bibitem[Giordano and Broderick(2023)]%
        {giordano2023}
\bibfield{author}{\bibinfo{person}{Ryan Giordano} {and} \bibinfo{person}{Tamara
  Broderick}.} \bibinfo{year}{2023}\natexlab{}.
\newblock \bibinfo{title}{The {{Bayesian Infinitesimal Jackknife}} for
  {{Variance}}}.
\newblock
\showeprint[arxiv]{2305.06466}
\urldef\tempurl%
\url{https://arxiv.org/abs/2305.06466}
\showURL{%
\tempurl}


\bibitem[Glasmachers et~al\mbox{.}(2010)]%
        {glasmachers2010}
\bibfield{author}{\bibinfo{person}{Tobias Glasmachers}, \bibinfo{person}{Tom
  Schaul}, \bibinfo{person}{Sun Yi}, \bibinfo{person}{Daan Wierstra}, {and}
  \bibinfo{person}{J{\"u}rgen Schmidhuber}.} \bibinfo{year}{2010}\natexlab{}.
\newblock \showarticletitle{Exponential Natural Evolution Strategies}. In
  \bibinfo{booktitle}{\emph{Proceedings of the 12th Annual Conference on
  {{Genetic}} and Evolutionary Computation}}. \bibinfo{publisher}{ACM},
  \bibinfo{pages}{393--400}.
\newblock
\href{https://doi.org/10.1145/1830483.1830557}{doi:\nolinkurl{10.1145/1830483.1830557}}


\bibitem[Graham and {Leyton-Brown}(2025)]%
        {graham2025}
\bibfield{author}{\bibinfo{person}{Devon~R. Graham} {and}
  \bibinfo{person}{Kevin {Leyton-Brown}}.} \bibinfo{year}{2025}\natexlab{}.
\newblock \showarticletitle{Utilitarian {{Algorithm Configuration}} for
  {{Infinite Parameter Spaces}}}. In \bibinfo{booktitle}{\emph{International
  {{Conference}} on {{Learning Representations}}}},
  \bibfield{editor}{\bibinfo{person}{Y.~Yue}, \bibinfo{person}{A.~Garg},
  \bibinfo{person}{N.~Peng}, \bibinfo{person}{F.~Sha}, {and}
  \bibinfo{person}{R.~Yu}} (Eds.). \bibinfo{pages}{27687--27698}.
\newblock
\showeprint[arxiv]{2405.18246}
\urldef\tempurl%
\url{https://proceedings.iclr.cc/paper_files/paper/2025/hash/450cc7fb37f25890db74a68b424f2d65-Abstract-Conference.html}
\showURL{%
\tempurl}


\bibitem[Graham et~al\mbox{.}(2023)]%
        {graham2023}
\bibfield{author}{\bibinfo{person}{Devon~R. Graham}, \bibinfo{person}{Kevin
  {Leyton-Brown}}, {and} \bibinfo{person}{Tim Roughgarden}.}
  \bibinfo{year}{2023}\natexlab{}.
\newblock \showarticletitle{Utilitarian {{Algorithm Configuration}}}. In
  \bibinfo{booktitle}{\emph{Advances in {{Neural Information Processing
  Systems}}}}, Vol.~\bibinfo{volume}{36}. \bibinfo{publisher}{Curran
  Associates, Inc.}, \bibinfo{pages}{68939--68950}.
\newblock
\showeprint[arxiv]{2310.20401}
\href{https://doi.org/10.52202/075280-3017}{doi:\nolinkurl{10.52202/075280-3017}}


\bibitem[Graham et~al\mbox{.}(2026)]%
        {graham2026}
\bibfield{author}{\bibinfo{person}{Devon~R. Graham},
  \bibinfo{person}{Eros~Rojas Velez}, {and} \bibinfo{person}{Kevin
  {Leyton-Brown}}.} \bibinfo{year}{2026}\natexlab{}.
\newblock \showarticletitle{Practical, {{Utilitarian Algorithm
  Configuration}}}.
\newblock \bibinfo{journal}{\emph{Proceedings of the AAAI Conference on
  Artificial Intelligence}} \bibinfo{volume}{40}, \bibinfo{number}{43}
  (\bibinfo{date}{March} \bibinfo{year}{2026}), \bibinfo{pages}{36937--36945}.
\newblock
\showISSN{2374-3468}
\href{https://doi.org/10.1609/aaai.v40i43.41021}{doi:\nolinkurl{10.1609/aaai.v40i43.41021}}


\bibitem[Hansen(2023)]%
        {hansen2023}
\bibfield{author}{\bibinfo{person}{Nikolaus Hansen}.}
  \bibinfo{year}{2023}\natexlab{}.
\newblock \bibinfo{title}{The {{CMA Evolution Strategy}}: {{A Tutorial}}}.
\newblock
\showeprint[arxiv]{1604.00772}~[cs]
\href{https://doi.org/10.48550/arXiv.1604.00772}{doi:\nolinkurl{10.48550/arXiv.1604.00772}}


\bibitem[Hansen et~al\mbox{.}(2014)]%
        {hansen2014}
\bibfield{author}{\bibinfo{person}{Nikolaus Hansen}, \bibinfo{person}{Asma
  Atamna}, {and} \bibinfo{person}{Anne Auger}.}
  \bibinfo{year}{2014}\natexlab{}.
\newblock \showarticletitle{How to {{Assess Step-Size Adaptation Mechanisms}}
  in {{Randomised Search}}}.
\newblock In \bibinfo{booktitle}{\emph{Parallel {{Problem Solving}} from
  {{Nature}} -- {{PPSN XIII}}}}. \bibinfo{series}{Lecture {{Notes}} in
  {{Computer Science}}}, Vol.~\bibinfo{volume}{8672}.
  \bibinfo{publisher}{Springer International Publishing},
  \bibinfo{pages}{60--69}.
\newblock
\showISBNx{978-3-319-10762-2}
\showISSN{0302-9743}
\href{https://doi.org/10.1007/978-3-319-10762-2_6}{doi:\nolinkurl{10.1007/978-3-319-10762-2_6}}


\bibitem[Hansen et~al\mbox{.}(2016)]%
        {hansen2016}
\bibfield{author}{\bibinfo{person}{Nikolaus Hansen}, \bibinfo{person}{Anne
  Auger}, \bibinfo{person}{Dimo Brockhoff}, \bibinfo{person}{Dejan Tu{\v s}ar},
  {and} \bibinfo{person}{Tea Tu{\v s}ar}.} \bibinfo{year}{2016}\natexlab{}.
\newblock \bibinfo{title}{{{COCO}}: {{Performance Assessment}}}.
\newblock
\showeprint[arxiv]{1605.03560}~[cs]
\href{https://doi.org/10.48550/arXiv.1605.03560}{doi:\nolinkurl{10.48550/arXiv.1605.03560}}


\bibitem[Hansen et~al\mbox{.}(2009)]%
        {hansen2009}
\bibfield{author}{\bibinfo{person}{Nikolaus Hansen}, \bibinfo{person}{Steffen
  Finck}, \bibinfo{person}{Raymond Ros}, {and} \bibinfo{person}{Anne Auger}.}
  \bibinfo{year}{2009}\natexlab{}.
\newblock \bibinfo{booktitle}{\emph{Real-{{Parameter Black-Box Optimization
  Benchmarking}} 2009: {{Noiseless Functions Definitions}}}}.
\newblock \bibinfo{type}{Report}. \bibinfo{institution}{INRIA}.
\newblock
\urldef\tempurl%
\url{https://hal.inria.fr/inria-00362633}
\showURL{%
\tempurl}


\bibitem[Hansen and Ostermeier(2001)]%
        {hansen2001}
\bibfield{author}{\bibinfo{person}{Nikolaus Hansen} {and}
  \bibinfo{person}{Andreas Ostermeier}.} \bibinfo{year}{2001}\natexlab{}.
\newblock \showarticletitle{Completely {{Derandomized Self-Adaptation}} in
  {{Evolution Strategies}}}.
\newblock \bibinfo{journal}{\emph{Evolutionary Computation}}
  \bibinfo{volume}{9}, \bibinfo{number}{2} (\bibinfo{date}{June}
  \bibinfo{year}{2001}), \bibinfo{pages}{159--195}.
\newblock
\showISSN{1063-6560}
\href{https://doi.org/10.1162/106365601750190398}{doi:\nolinkurl{10.1162/106365601750190398}}


\bibitem[Hendriksen et~al\mbox{.}(2021)]%
        {hendriksen2021}
\bibfield{author}{\bibinfo{person}{Allard Hendriksen}, \bibinfo{person}{Rianne
  {de Heide}}, {and} \bibinfo{person}{Peter Gr{\"u}nwald}.}
  \bibinfo{year}{2021}\natexlab{}.
\newblock \showarticletitle{Optional {{Stopping}} with {{Bayes Factors}}: {{A
  Categorization}} and {{Extension}} of {{Folklore Results}}, with an
  {{Application}} to {{Invariant Situations}}}.
\newblock \bibinfo{journal}{\emph{Bayesian Analysis}} \bibinfo{volume}{16},
  \bibinfo{number}{3} (\bibinfo{date}{Sept.} \bibinfo{year}{2021}),
  \bibinfo{pages}{961--989}.
\newblock
\showISSN{1936-0975}
\href{https://doi.org/10.1214/20-BA1234}{doi:\nolinkurl{10.1214/20-BA1234}}


\bibitem[Hoffman and Gelman(2014)]%
        {hoffman2014}
\bibfield{author}{\bibinfo{person}{Matthew~D. Hoffman} {and}
  \bibinfo{person}{Andrew Gelman}.} \bibinfo{year}{2014}\natexlab{}.
\newblock \showarticletitle{The {{No-U}}-Turn Sampler: Adaptively Setting Path
  Lengths in {{Hamiltonian Monte Carlo}}}.
\newblock \bibinfo{journal}{\emph{Journal of Machine Learning Research}}
  \bibinfo{volume}{15}, \bibinfo{number}{47} (\bibinfo{date}{Jan.}
  \bibinfo{year}{2014}), \bibinfo{pages}{1593--1623}.
\newblock
\showISSN{1532-4435}
\urldef\tempurl%
\url{https://dl.acm.org/doi/10.5555/2627435.2638586}
\showURL{%
\tempurl}


\bibitem[Jesus et~al\mbox{.}(2020)]%
        {jesus2020}
\bibfield{author}{\bibinfo{person}{Alexandre~D. Jesus}, \bibinfo{person}{Arnaud
  Liefooghe}, \bibinfo{person}{Bilel Derbel}, {and} \bibinfo{person}{Lu{\'i}s
  Paquete}.} \bibinfo{year}{2020}\natexlab{}.
\newblock \showarticletitle{Algorithm Selection of Anytime Algorithms}. In
  \bibinfo{booktitle}{\emph{Proceedings of the 2020 {{Genetic}} and
  {{Evolutionary Computation Conference}}}} \emph{(\bibinfo{series}{{{GECCO}}
  '20})}. \bibinfo{publisher}{Association for Computing Machinery},
  \bibinfo{address}{New York, NY, USA}, \bibinfo{pages}{850--858}.
\newblock
\showISBNx{978-1-4503-7128-5}
\href{https://doi.org/10.1145/3377930.3390185}{doi:\nolinkurl{10.1145/3377930.3390185}}


\bibitem[Kerschke et~al\mbox{.}(2019)]%
        {kerschke2019}
\bibfield{author}{\bibinfo{person}{Pascal Kerschke}, \bibinfo{person}{Holger~H.
  Hoos}, \bibinfo{person}{Frank Neumann}, {and} \bibinfo{person}{Heike
  Trautmann}.} \bibinfo{year}{2019}\natexlab{}.
\newblock \showarticletitle{Automated {{Algorithm Selection}}: {{Survey}} and
  {{Perspectives}}}.
\newblock \bibinfo{journal}{\emph{Evolutionary Computation}}
  \bibinfo{volume}{27}, \bibinfo{number}{1} (\bibinfo{date}{March}
  \bibinfo{year}{2019}), \bibinfo{pages}{3--45}.
\newblock
\showISSN{1063-6560}
\href{https://doi.org/10.1162/evco_a_00242}{doi:\nolinkurl{10.1162/evco_a_00242}}


\bibitem[Kitagawa et~al\mbox{.}(2020)]%
        {kitagawa2020}
\bibfield{author}{\bibinfo{person}{Toru Kitagawa},
  \bibinfo{person}{Jos{\'e}~Luis Montiel~Olea}, \bibinfo{person}{Jonathan
  Payne}, {and} \bibinfo{person}{Amilcar Velez}.}
  \bibinfo{year}{2020}\natexlab{}.
\newblock \showarticletitle{Posterior Distribution of Nondifferentiable
  Functions}.
\newblock \bibinfo{journal}{\emph{Journal of Econometrics}}
  \bibinfo{volume}{217}, \bibinfo{number}{1} (\bibinfo{date}{July}
  \bibinfo{year}{2020}), \bibinfo{pages}{161--175}.
\newblock
\showISSN{0304-4076}
\href{https://doi.org/10.1016/j.jeconom.2019.10.009}{doi:\nolinkurl{10.1016/j.jeconom.2019.10.009}}


\bibitem[Krause et~al\mbox{.}(2017)]%
        {krause2017}
\bibfield{author}{\bibinfo{person}{Oswin Krause}, \bibinfo{person}{Tobias
  Glasmachers}, {and} \bibinfo{person}{Christian Igel}.}
  \bibinfo{year}{2017}\natexlab{}.
\newblock \showarticletitle{Qualitative and {{Quantitative Assessment}} of
  {{Step Size Adaptation Rules}}}. In \bibinfo{booktitle}{\emph{Proceedings of
  the 14th {{ACM}}/{{SIGEVO Conference}} on {{Foundations}} of {{Genetic
  Algorithms}}}}. \bibinfo{publisher}{ACM}, \bibinfo{pages}{139--148}.
\newblock
\href{https://doi.org/10.1145/3040718.3040725}{doi:\nolinkurl{10.1145/3040718.3040725}}


\bibitem[Krause et~al\mbox{.}(2016)]%
        {krause2016}
\bibfield{author}{\bibinfo{person}{Oswin Krause}, \bibinfo{person}{D{\'i}dac
  Rodr{\'i}guez~Arbon{\`e}s}, {and} \bibinfo{person}{Christian Igel}.}
  \bibinfo{year}{2016}\natexlab{}.
\newblock \showarticletitle{{{CMA-ES}} with {{Optimal Covariance Update}} and
  {{Storage Complexity}}}. In \bibinfo{booktitle}{\emph{Advances in {{Neural
  Information Processing Systems}} 29}},
  \bibfield{editor}{\bibinfo{person}{D.~Lee}, \bibinfo{person}{M.~Sugiyama},
  \bibinfo{person}{U.~Luxburg}, \bibinfo{person}{I.~Guyon}, {and}
  \bibinfo{person}{R.~Garnett}} (Eds.). \bibinfo{publisher}{Curran Associates,
  Inc.}, \bibinfo{pages}{370--378}.
\newblock
\urldef\tempurl%
\url{https://proceedings.neurips.cc/paper/2016/hash/289dff07669d7a23de0ef88d2f7129e7-Abstract.html}
\showURL{%
\tempurl}


\bibitem[Kruschke(2018)]%
        {kruschke2018}
\bibfield{author}{\bibinfo{person}{John~K. Kruschke}.}
  \bibinfo{year}{2018}\natexlab{}.
\newblock \showarticletitle{Rejecting or {{Accepting Parameter Values}} in
  {{Bayesian Estimation}}}.
\newblock \bibinfo{journal}{\emph{Advances in Methods and Practices in
  Psychological Science}} \bibinfo{volume}{1}, \bibinfo{number}{2}
  (\bibinfo{date}{June} \bibinfo{year}{2018}), \bibinfo{pages}{270--280}.
\newblock
\showISSN{2515-2459}
\href{https://doi.org/10.1177/2515245918771304}{doi:\nolinkurl{10.1177/2515245918771304}}


\bibitem[Lancaster(1965)]%
        {lancaster1965}
\bibfield{author}{\bibinfo{person}{H.~O. Lancaster}.}
  \bibinfo{year}{1965}\natexlab{}.
\newblock \showarticletitle{The {{Helmert Matrices}}}.
\newblock \bibinfo{journal}{\emph{The American Mathematical Monthly}}
  \bibinfo{volume}{72}, \bibinfo{number}{1} (\bibinfo{year}{1965}),
  \bibinfo{pages}{4--12}.
\newblock
\showISSN{0002-9890}
\showeprint[jstor]{2312989}
\href{https://doi.org/10.2307/2312989}{doi:\nolinkurl{10.2307/2312989}}


\bibitem[Lehmann and Romano(2005)]%
        {lehmann2005}
\bibfield{author}{\bibinfo{person}{E.L. Lehmann} {and}
  \bibinfo{person}{Joseph~P. Romano}.} \bibinfo{year}{2005}\natexlab{}.
\newblock \bibinfo{booktitle}{\emph{Testing {{Statistical Hypotheses}}}
  (\bibinfo{edition}{3} ed.)}.
\newblock \bibinfo{publisher}{Springer}, \bibinfo{address}{New York}.
\newblock
\showISBNx{978-0-387-98864-1}
\showISSN{1431-875X}
\href{https://doi.org/10.1007/0-387-27605-X}{doi:\nolinkurl{10.1007/0-387-27605-X}}


\bibitem[Li et~al\mbox{.}(2018)]%
        {li2018}
\bibfield{author}{\bibinfo{person}{Lisha Li}, \bibinfo{person}{Kevin Jamieson},
  \bibinfo{person}{Giulia DeSalvo}, \bibinfo{person}{Afshin Rostamizadeh},
  {and} \bibinfo{person}{Ameet Talwalkar}.} \bibinfo{year}{2018}\natexlab{}.
\newblock \showarticletitle{Hyperband: {{A Novel Bandit-Based Approach}} to
  {{Hyperparameter Optimization}}}.
\newblock \bibinfo{journal}{\emph{Journal of Machine Learning Research}}
  \bibinfo{volume}{18}, \bibinfo{number}{185} (\bibinfo{year}{2018}),
  \bibinfo{pages}{1--52}.
\newblock
\showISSN{1533-7928}
\urldef\tempurl%
\url{http://jmlr.org/papers/v18/16-558.html}
\showURL{%
\tempurl}


\bibitem[{L{\'o}pez-Ib{\'a}{\~n}ez} et~al\mbox{.}(2016)]%
        {lopez-ibanez2016}
\bibfield{author}{\bibinfo{person}{Manuel {L{\'o}pez-Ib{\'a}{\~n}ez}},
  \bibinfo{person}{J{\'e}r{\'e}mie {Dubois-Lacoste}}, \bibinfo{person}{Leslie
  P{\'e}rez~C{\'a}ceres}, \bibinfo{person}{Mauro Birattari}, {and}
  \bibinfo{person}{Thomas St{\"u}tzle}.} \bibinfo{year}{2016}\natexlab{}.
\newblock \showarticletitle{The Irace Package: {{Iterated}} Racing for
  Automatic Algorithm Configuration}.
\newblock \bibinfo{journal}{\emph{Operations Research Perspectives}}
  \bibinfo{volume}{3} (\bibinfo{date}{Jan.} \bibinfo{year}{2016}),
  \bibinfo{pages}{43--58}.
\newblock
\showISSN{2214-7160}
\href{https://doi.org/10.1016/j.orp.2016.09.002}{doi:\nolinkurl{10.1016/j.orp.2016.09.002}}


\bibitem[{L{\'o}pez-Ib{\'a}{\~n}ez} and St{\"u}tzle(2014)]%
        {lopez-ibanez2014}
\bibfield{author}{\bibinfo{person}{Manuel {L{\'o}pez-Ib{\'a}{\~n}ez}} {and}
  \bibinfo{person}{Thomas St{\"u}tzle}.} \bibinfo{year}{2014}\natexlab{}.
\newblock \showarticletitle{Automatically Improving the Anytime Behaviour of
  Optimisation Algorithms}.
\newblock \bibinfo{journal}{\emph{European Journal of Operational Research}}
  \bibinfo{volume}{235}, \bibinfo{number}{3} (\bibinfo{date}{June}
  \bibinfo{year}{2014}), \bibinfo{pages}{569--582}.
\newblock
\showISSN{0377-2217}
\href{https://doi.org/10.1016/j.ejor.2013.10.043}{doi:\nolinkurl{10.1016/j.ejor.2013.10.043}}


\bibitem[{L{\'o}pez-Ib{\'a}{\~n}ez} et~al\mbox{.}(2024)]%
        {lopez-ibanez2024}
\bibfield{author}{\bibinfo{person}{Manuel {L{\'o}pez-Ib{\'a}{\~n}ez}},
  \bibinfo{person}{Diederick Vermetten}, \bibinfo{person}{Johann Dreo}, {and}
  \bibinfo{person}{Carola Doerr}.} \bibinfo{year}{2024}\natexlab{}.
\newblock \showarticletitle{Using the {{Empirical Attainment Function}} for
  {{Analyzing Single-Objective Black-Box Optimization Algorithms}}}.
\newblock \bibinfo{journal}{\emph{IEEE Transactions on Evolutionary
  Computation}} (\bibinfo{year}{2024}), \bibinfo{pages}{1--1}.
\newblock
\showISSN{1941-0026}
\href{https://doi.org/10.1109/TEVC.2024.3462758}{doi:\nolinkurl{10.1109/TEVC.2024.3462758}}


\bibitem[Loshchilov(2014)]%
        {loshchilov2014}
\bibfield{author}{\bibinfo{person}{Ilya Loshchilov}.}
  \bibinfo{year}{2014}\natexlab{}.
\newblock \showarticletitle{A Computationally Efficient Limited Memory
  {{CMA-ES}} for Large Scale Optimization}. In
  \bibinfo{booktitle}{\emph{Proceedings of the 2014 {{Annual Conference}} on
  {{Genetic}} and {{Evolutionary Computation}}}}. \bibinfo{publisher}{ACM},
  \bibinfo{pages}{397--404}.
\newblock
\href{https://doi.org/10.1145/2576768.2598294}{doi:\nolinkurl{10.1145/2576768.2598294}}


\bibitem[Loshchilov(2017)]%
        {loshchilov2017}
\bibfield{author}{\bibinfo{person}{Ilya Loshchilov}.}
  \bibinfo{year}{2017}\natexlab{}.
\newblock \showarticletitle{{{LM-CMA}}: {{An Alternative}} to {{L-BFGS}} for
  {{Large-Scale Black Box Optimization}}}.
\newblock \bibinfo{journal}{\emph{Evolutionary Computation}}
  \bibinfo{volume}{25}, \bibinfo{number}{1} (\bibinfo{date}{March}
  \bibinfo{year}{2017}), \bibinfo{pages}{143--171}.
\newblock
\showISSN{1063-6560}
\href{https://doi.org/10.1162/EVCO_a_00168}{doi:\nolinkurl{10.1162/EVCO_a_00168}}


\bibitem[Luce(1979)]%
        {luce1979}
\bibfield{author}{\bibinfo{person}{R.~Duncan Luce}.}
  \bibinfo{year}{1979}\natexlab{}.
\newblock \bibinfo{booktitle}{\emph{Individual Choice Behavior: A Theoretical
  Analysis}}.
\newblock \bibinfo{publisher}{Greenwood Press}, \bibinfo{address}{Westport,
  Conn.}
\newblock
\showISBNx{978-0-313-20778-5}
\urldef\tempurl%
\url{https://catalog.hathitrust.org/Record/010607369}
\showURL{%
\tempurl}


\bibitem[Meunier et~al\mbox{.}(2022)]%
        {meunier2022}
\bibfield{author}{\bibinfo{person}{Laurent Meunier},
  \bibinfo{person}{Herilalaina Rakotoarison}, \bibinfo{person}{Pak~Kan Wong},
  \bibinfo{person}{Baptiste Roziere}, \bibinfo{person}{J{\'e}r{\'e}my Rapin},
  \bibinfo{person}{Olivier Teytaud}, \bibinfo{person}{Antoine Moreau}, {and}
  \bibinfo{person}{Carola Doerr}.} \bibinfo{year}{2022}\natexlab{}.
\newblock \showarticletitle{Black-{{Box Optimization Revisited}}: {{Improving
  Algorithm Selection Wizards Through Massive Benchmarking}}}.
\newblock \bibinfo{journal}{\emph{IEEE Transactions on Evolutionary
  Computation}} \bibinfo{volume}{26}, \bibinfo{number}{3} (\bibinfo{date}{June}
  \bibinfo{year}{2022}), \bibinfo{pages}{490--500}.
\newblock
\showISSN{1941-0026}
\href{https://doi.org/10.1109/TEVC.2021.3108185}{doi:\nolinkurl{10.1109/TEVC.2021.3108185}}


\bibitem[M{\"u}ller et~al\mbox{.}(2004)]%
        {muller2004a}
\bibfield{author}{\bibinfo{person}{Peter M{\"u}ller}, \bibinfo{person}{Giovanni
  Parmigiani}, \bibinfo{person}{Christian Robert}, {and}
  \bibinfo{person}{Judith Rousseau}.} \bibinfo{year}{2004}\natexlab{}.
\newblock \showarticletitle{Optimal {{Sample Size}} for {{Multiple Testing}}:
  {{The Case}} of {{Gene Expression Microarrays}}}.
\newblock \bibinfo{journal}{\emph{J. Amer. Statist. Assoc.}}
  \bibinfo{volume}{99}, \bibinfo{number}{468} (\bibinfo{date}{Dec.}
  \bibinfo{year}{2004}), \bibinfo{pages}{990--1001}.
\newblock
\showISSN{0162-1459}
\href{https://doi.org/10.1198/016214504000001646}{doi:\nolinkurl{10.1198/016214504000001646}}


\bibitem[Mu{\~n}oz et~al\mbox{.}(2015)]%
        {munoz2015}
\bibfield{author}{\bibinfo{person}{Mario~A. Mu{\~n}oz}, \bibinfo{person}{Yuan
  Sun}, \bibinfo{person}{Michael Kirley}, {and} \bibinfo{person}{Saman~K.
  Halgamuge}.} \bibinfo{year}{2015}\natexlab{}.
\newblock \showarticletitle{Algorithm Selection for Black-Box Continuous
  Optimization Problems: {{A}} Survey on Methods and Challenges}.
\newblock \bibinfo{journal}{\emph{Information Sciences}}  \bibinfo{volume}{317}
  (\bibinfo{date}{Oct.} \bibinfo{year}{2015}), \bibinfo{pages}{224--245}.
\newblock
\showISSN{0020-0255}
\href{https://doi.org/10.1016/j.ins.2015.05.010}{doi:\nolinkurl{10.1016/j.ins.2015.05.010}}


\bibitem[Newton(2004)]%
        {newton2004}
\bibfield{author}{\bibinfo{person}{M.~A. Newton}.}
  \bibinfo{year}{2004}\natexlab{}.
\newblock \showarticletitle{Detecting Differential Gene Expression with a
  Semiparametric Hierarchical Mixture Method}.
\newblock \bibinfo{journal}{\emph{Biostatistics}} \bibinfo{volume}{5},
  \bibinfo{number}{2} (\bibinfo{date}{April} \bibinfo{year}{2004}),
  \bibinfo{pages}{155--176}.
\newblock
\showISSN{1465-4644}
\href{https://doi.org/10.1093/biostatistics/5.2.155}{doi:\nolinkurl{10.1093/biostatistics/5.2.155}}


\bibitem[Oentaryo et~al\mbox{.}(2015)]%
        {oentaryo2015}
\bibfield{author}{\bibinfo{person}{Richard Oentaryo},
  \bibinfo{person}{Stephanus~Daniel Handoko}, {and}
  \bibinfo{person}{Hoong~Chuin Lau}.} \bibinfo{year}{2015}\natexlab{}.
\newblock \showarticletitle{Algorithm {{Selection}} via {{Ranking}}}.
\newblock \bibinfo{journal}{\emph{Proceedings of the AAAI Conference on
  Artificial Intelligence}} \bibinfo{volume}{29}, \bibinfo{number}{1}
  (\bibinfo{date}{Feb.} \bibinfo{year}{2015}).
\newblock
\showISSN{2374-3468}
\href{https://doi.org/10.1609/aaai.v29i1.9466}{doi:\nolinkurl{10.1609/aaai.v29i1.9466}}


\bibitem[Plackett(1975)]%
        {plackett1975}
\bibfield{author}{\bibinfo{person}{R.~L. Plackett}.}
  \bibinfo{year}{1975}\natexlab{}.
\newblock \showarticletitle{The {{Analysis}} of {{Permutations}}}.
\newblock \bibinfo{journal}{\emph{Journal of the Royal Statistical Society.
  Series C (Applied Statistics)}} \bibinfo{volume}{24}, \bibinfo{number}{2}
  (\bibinfo{year}{1975}), \bibinfo{pages}{193--202}.
\newblock
\showISSN{0035-9254}
\showeprint[jstor]{2346567}
\href{https://doi.org/10.2307/2346567}{doi:\nolinkurl{10.2307/2346567}}


\bibitem[Ramamohan et~al\mbox{.}(2016)]%
        {ramamohan2016}
\bibfield{author}{\bibinfo{person}{Siddartha~Y. Ramamohan},
  \bibinfo{person}{Arun Rajkumar}, \bibinfo{person}{Shivani Agarwal}, {and}
  \bibinfo{person}{Shivani Agarwal}.} \bibinfo{year}{2016}\natexlab{}.
\newblock \showarticletitle{Dueling {{Bandits}}: {{Beyond Condorcet Winners}}
  to {{General Tournament Solutions}}}. In \bibinfo{booktitle}{\emph{Advances
  in {{Neural Information Processing Systems}}}}, Vol.~\bibinfo{volume}{29}.
  \bibinfo{publisher}{Curran Associates, Inc.}
\newblock
\urldef\tempurl%
\url{https://proceedings.neurips.cc/paper/2016/hash/fccb3cdc9acc14a6e70a12f74560c026-Abstract.html}
\showURL{%
\tempurl}


\bibitem[Rasmussen and Williams(2005)]%
        {rasmussen2005}
\bibfield{author}{\bibinfo{person}{Carl~Edward Rasmussen} {and}
  \bibinfo{person}{Christopher K.~I. Williams}.}
  \bibinfo{year}{2005}\natexlab{}.
\newblock \bibinfo{booktitle}{\emph{Gaussian {{Processes}} for {{Machine
  Learning}}}}.
\newblock \bibinfo{publisher}{The MIT Press}.
\newblock
\showISBNx{978-0-262-25683-4}
\href{https://doi.org/10.7551/mitpress/3206.001.0001}{doi:\nolinkurl{10.7551/mitpress/3206.001.0001}}


\bibitem[Ravber et~al\mbox{.}(2024)]%
        {ravber2024}
\bibfield{author}{\bibinfo{person}{Miha Ravber}, \bibinfo{person}{Marjan
  Mernik}, \bibinfo{person}{Shih-Hsi Liu}, \bibinfo{person}{Marko {\v S}mid},
  {and} \bibinfo{person}{Matej {\v C}repin{\v s}ek}.}
  \bibinfo{year}{2024}\natexlab{}.
\newblock \showarticletitle{Confidence {{Bands Based}} on {{Rating
  Demonstrated}} on the {{CEC}} 2021 {{Competition Results}}}. In
  \bibinfo{booktitle}{\emph{2024 {{IEEE Congress}} on {{Evolutionary
  Computation}} ({{CEC}})}}. \bibinfo{pages}{1--8}.
\newblock
\href{https://doi.org/10.1109/CEC60901.2024.10611955}{doi:\nolinkurl{10.1109/CEC60901.2024.10611955}}


\bibitem[{Rojas-Delgado} et~al\mbox{.}(2022)]%
        {rojas-delgado2022}
\bibfield{author}{\bibinfo{person}{Jairo {Rojas-Delgado}},
  \bibinfo{person}{Josu Ceberio}, \bibinfo{person}{Borja Calvo}, {and}
  \bibinfo{person}{Jose~A. Lozano}.} \bibinfo{year}{2022}\natexlab{}.
\newblock \showarticletitle{Bayesian {{Performance Analysis}} for {{Algorithm
  Ranking Comparison}}}.
\newblock \bibinfo{journal}{\emph{IEEE Transactions on Evolutionary
  Computation}} \bibinfo{volume}{26}, \bibinfo{number}{6} (\bibinfo{date}{Dec.}
  \bibinfo{year}{2022}), \bibinfo{pages}{1281--1292}.
\newblock
\showISSN{1941-0026}
\href{https://doi.org/10.1109/TEVC.2022.3208110}{doi:\nolinkurl{10.1109/TEVC.2022.3208110}}


\bibitem[Rook et~al\mbox{.}(2024)]%
        {rook2024}
\bibfield{author}{\bibinfo{person}{Jeroen Rook}, \bibinfo{person}{Holger~H.
  Hoos}, {and} \bibinfo{person}{Heike Trautmann}.}
  \bibinfo{year}{2024}\natexlab{}.
\newblock \showarticletitle{Multi-Objective {{Ranking}} Using {{Bootstrap
  Resampling}}}. In \bibinfo{booktitle}{\emph{Proceedings of the {{Genetic}}
  and {{Evolutionary Computation Conference Companion}}}}
  \emph{(\bibinfo{series}{{{GECCO}} '24 {{Companion}}})}.
  \bibinfo{publisher}{Association for Computing Machinery},
  \bibinfo{address}{New York, NY, USA}, \bibinfo{pages}{155--158}.
\newblock
\showISBNx{979-8-4007-0495-6}
\href{https://doi.org/10.1145/3638530.3654436}{doi:\nolinkurl{10.1145/3638530.3654436}}


\bibitem[Ros and Hansen(2008)]%
        {ros2008}
\bibfield{author}{\bibinfo{person}{Raymond Ros} {and} \bibinfo{person}{Nikolaus
  Hansen}.} \bibinfo{year}{2008}\natexlab{}.
\newblock \showarticletitle{A {{Simple Modification}} in {{CMA-ES Achieving
  Linear Time}} and {{Space Complexity}}}.
\newblock In \bibinfo{booktitle}{\emph{Parallel {{Problem Solving}} from
  {{Nature}} -- {{PPSN X}}}}. \bibinfo{series}{Lecture {{Notes}} in {{Computer
  Science}}}, Vol.~\bibinfo{volume}{5199}. \bibinfo{publisher}{Springer Berlin
  Heidelberg}, \bibinfo{pages}{296--305}.
\newblock
\showISBNx{978-3-540-87700-4}
\showISSN{0302-9743}
\href{https://doi.org/10.1007/978-3-540-87700-4_30}{doi:\nolinkurl{10.1007/978-3-540-87700-4_30}}


\bibitem[Rouder(2014)]%
        {rouder2014}
\bibfield{author}{\bibinfo{person}{Jeffrey~N. Rouder}.}
  \bibinfo{year}{2014}\natexlab{}.
\newblock \showarticletitle{Optional Stopping: {{No}} Problem for
  {{Bayesians}}}.
\newblock \bibinfo{journal}{\emph{Psychonomic Bulletin \& Review}}
  \bibinfo{volume}{21}, \bibinfo{number}{2} (\bibinfo{date}{March}
  \bibinfo{year}{2014}), \bibinfo{pages}{301--308}.
\newblock
\showISSN{1069-9384}
\href{https://doi.org/10.3758/s13423-014-0595-4}{doi:\nolinkurl{10.3758/s13423-014-0595-4}}


\bibitem[Smith(1973)]%
        {smith1973}
\bibfield{author}{\bibinfo{person}{John~H. Smith}.}
  \bibinfo{year}{1973}\natexlab{}.
\newblock \showarticletitle{Aggregation of {{Preferences}} with {{Variable
  Electorate}}}.
\newblock \bibinfo{journal}{\emph{Econometrica}} \bibinfo{volume}{41},
  \bibinfo{number}{6} (\bibinfo{date}{Nov.} \bibinfo{year}{1973}),
  \bibinfo{pages}{1027}.
\newblock
\showISSN{0012-9682}
\href{https://doi.org/10.2307/1914033}{doi:\nolinkurl{10.2307/1914033}}


\bibitem[Tanabe and Fukunaga(2013)]%
        {tanabe2013}
\bibfield{author}{\bibinfo{person}{Ryoji Tanabe} {and} \bibinfo{person}{Alex
  Fukunaga}.} \bibinfo{year}{2013}\natexlab{}.
\newblock \showarticletitle{Success-History Based Parameter Adaptation for
  {{Differential Evolution}}}. In \bibinfo{booktitle}{\emph{2013 {{IEEE
  Congress}} on {{Evolutionary Computation}}}}. \bibinfo{pages}{71--78}.
\newblock
\showISSN{1941-0026}
\href{https://doi.org/10.1109/CEC.2013.6557555}{doi:\nolinkurl{10.1109/CEC.2013.6557555}}


\bibitem[Turner et~al\mbox{.}(2020)]%
        {turner2020}
\bibfield{author}{\bibinfo{person}{Heather~L. Turner}, \bibinfo{person}{Jacob
  {van Etten}}, \bibinfo{person}{David Firth}, {and} \bibinfo{person}{Ioannis
  Kosmidis}.} \bibinfo{year}{2020}\natexlab{}.
\newblock \showarticletitle{Modelling Rankings in {{R}}: The {{PlackettLuce}}
  Package}.
\newblock \bibinfo{journal}{\emph{Computational Statistics}}
  \bibinfo{volume}{35}, \bibinfo{number}{3} (\bibinfo{date}{Sept.}
  \bibinfo{year}{2020}), \bibinfo{pages}{1027--1057}.
\newblock
\showISSN{1613-9658}
\href{https://doi.org/10.1007/s00180-020-00959-3}{doi:\nolinkurl{10.1007/s00180-020-00959-3}}


\bibitem[{van Stein} et~al\mbox{.}(2024)]%
        {vanstein2024a}
\bibfield{author}{\bibinfo{person}{Niki {van Stein}},
  \bibinfo{person}{Diederick Vermetten}, \bibinfo{person}{Anna~V. Kononova},
  {and} \bibinfo{person}{Thomas B{\"a}ck}.} \bibinfo{year}{2024}\natexlab{}.
\newblock \bibinfo{title}{Explainable {{Benchmarking}} for {{Iterative
  Optimization Heuristics}}}.
\newblock
\showeprint[arxiv]{2401.17842}~[cs]
\href{https://doi.org/10.48550/arXiv.2401.17842}{doi:\nolinkurl{10.48550/arXiv.2401.17842}}


\bibitem[Vehtari et~al\mbox{.}(2021)]%
        {vehtari2021}
\bibfield{author}{\bibinfo{person}{Aki Vehtari}, \bibinfo{person}{Andrew
  Gelman}, \bibinfo{person}{Daniel Simpson}, \bibinfo{person}{Bob Carpenter},
  {and} \bibinfo{person}{Paul-Christian B{\"u}rkner}.}
  \bibinfo{year}{2021}\natexlab{}.
\newblock \showarticletitle{Rank-{{Normalization}}, {{Folding}}, and
  {{Localization}}: {{An Improved R\textasciicircum}} for {{Assessing
  Convergence}} of {{MCMC}} (with {{Discussion}})}.
\newblock \bibinfo{journal}{\emph{Bayesian Analysis}} \bibinfo{volume}{16},
  \bibinfo{number}{2} (\bibinfo{date}{June} \bibinfo{year}{2021}),
  \bibinfo{pages}{667--718}.
\newblock
\showISSN{1936-0975, 1931-6690}
\href{https://doi.org/10.1214/20-BA1221}{doi:\nolinkurl{10.1214/20-BA1221}}


\bibitem[Vehtari et~al\mbox{.}(2024)]%
        {vehtari2024}
\bibfield{author}{\bibinfo{person}{Aki Vehtari}, \bibinfo{person}{Daniel
  Simpson}, \bibinfo{person}{Andrew Gelman}, \bibinfo{person}{Yuling Yao},
  {and} \bibinfo{person}{Jonah Gabry}.} \bibinfo{year}{2024}\natexlab{}.
\newblock \showarticletitle{Pareto {{Smoothed Importance Sampling}}}.
\newblock \bibinfo{journal}{\emph{Journal of Machine Learning Research}}
  \bibinfo{volume}{25}, \bibinfo{number}{72} (\bibinfo{year}{2024}),
  \bibinfo{pages}{1--58}.
\newblock
\showISSN{1533-7928}
\urldef\tempurl%
\url{https://jmlr.org/papers/v25/19-556.html}
\showURL{%
\tempurl}


\bibitem[Vermetten et~al\mbox{.}(2025)]%
        {vermetten2025}
\bibfield{author}{\bibinfo{person}{Diederick Vermetten},
  \bibinfo{person}{Jeroen Rook}, \bibinfo{person}{Oliver~L. Preu{\ss}},
  \bibinfo{person}{Jacob {de Nobel}}, \bibinfo{person}{Carola Doerr},
  \bibinfo{person}{Manuel {L{\'o}pez-Iba{\~n}ez}}, \bibinfo{person}{Heike
  Trautmann}, {and} \bibinfo{person}{Thomas B{\"a}ck}.}
  \bibinfo{year}{2025}\natexlab{}.
\newblock \showarticletitle{{{MO-IOHinspector}}: {{Anytime Benchmarking}}
  of~{{Multi}}-Objective {{Algorithms Using IOHprofiler}}}. In
  \bibinfo{booktitle}{\emph{Evolutionary {{Multi-Criterion Optimization}}}},
  \bibfield{editor}{\bibinfo{person}{Hemant Singh}, \bibinfo{person}{Tapabrata
  Ray}, \bibinfo{person}{Joshua Knowles}, \bibinfo{person}{Xiaodong Li},
  \bibinfo{person}{Juergen Branke}, \bibinfo{person}{Bing Wang}, {and}
  \bibinfo{person}{Akira Oyama}} (Eds.). \bibinfo{publisher}{Springer Nature},
  \bibinfo{address}{Singapore}, \bibinfo{pages}{242--256}.
\newblock
\showISBNx{978-981-96-3506-1}
\href{https://doi.org/10.1007/978-981-96-3506-1_17}{doi:\nolinkurl{10.1007/978-981-96-3506-1_17}}


\bibitem[Vermetten et~al\mbox{.}(2022)]%
        {vermetten2022a}
\bibfield{author}{\bibinfo{person}{Diederick Vermetten}, \bibinfo{person}{Hao
  Wang}, \bibinfo{person}{Manuel {L{\'o}pez-Iba{\~n}ez}},
  \bibinfo{person}{Carola Doerr}, {and} \bibinfo{person}{Thomas B{\"a}ck}.}
  \bibinfo{year}{2022}\natexlab{}.
\newblock \showarticletitle{Analyzing the Impact of Undersampling on the
  Benchmarking and Configuration of Evolutionary Algorithms}. In
  \bibinfo{booktitle}{\emph{Proceedings of the {{Genetic}} and {{Evolutionary
  Computation Conference}}}} \emph{(\bibinfo{series}{{{GECCO}} '22})}.
  \bibinfo{publisher}{Association for Computing Machinery},
  \bibinfo{address}{New York, NY, USA}, \bibinfo{pages}{867--875}.
\newblock
\showISBNx{978-1-4503-9237-2}
\href{https://doi.org/10.1145/3512290.3528799}{doi:\nolinkurl{10.1145/3512290.3528799}}


\bibitem[Vermetten et~al\mbox{.}(2023)]%
        {vermetten2023}
\bibfield{author}{\bibinfo{person}{Diederick Vermetten},
  \bibinfo{person}{Furong Ye}, \bibinfo{person}{Thomas B{\"a}ck}, {and}
  \bibinfo{person}{Carola Doerr}.} \bibinfo{year}{2023}\natexlab{}.
\newblock \showarticletitle{{{MA-BBOB}}: {{Many-Affine Combinations}} of {{BBOB
  Functions}} for {{Evaluating AutoML Approaches}} in {{Noiseless Numerical
  Black-Box Optimization Contexts}}}. In \bibinfo{booktitle}{\emph{Proceedings
  of the {{Second International Conference}} on {{Automated Machine
  Learning}}}}. \bibinfo{publisher}{PMLR}, \bibinfo{pages}{7/1--14}.
\newblock
\showISSN{2640-3498}
\urldef\tempurl%
\url{https://proceedings.mlr.press/v224/vermetten23a.html}
\showURL{%
\tempurl}


\bibitem[Warchulski and Arabas(2021)]%
        {warchulski2021}
\bibfield{author}{\bibinfo{person}{Eryk Warchulski} {and}
  \bibinfo{person}{Jaroslaw Arabas}.} \bibinfo{year}{2021}\natexlab{}.
\newblock \showarticletitle{A {{New Step-Size Adaptation Rule}} for {{CMA-ES
  Based}} on the {{Population Midpoint Fitness}}}. In
  \bibinfo{booktitle}{\emph{2021 {{IEEE Congress}} on {{Evolutionary
  Computation}} ({{CEC}})}}. \bibinfo{publisher}{IEEE},
  \bibinfo{pages}{825--831}.
\newblock
\href{https://doi.org/10.1109/CEC45853.2021.9504829}{doi:\nolinkurl{10.1109/CEC45853.2021.9504829}}


\bibitem[Yan et~al\mbox{.}(2022)]%
        {yan2022}
\bibfield{author}{\bibinfo{person}{Yuan Yan}, \bibinfo{person}{Qunfeng Liu},
  \bibinfo{person}{Li}, {and} \bibinfo{person}{{Yun}}.}
  \bibinfo{year}{2022}\natexlab{}.
\newblock \showarticletitle{Paradox-Free Analysis for Comparing the Performance
  of Optimization Algorithms}.
\newblock \bibinfo{journal}{\emph{IEEE Transactions on Evolutionary
  Computation}} (\bibinfo{year}{2022}), \bibinfo{pages}{1--1}.
\newblock
\showISSN{1941-0026}
\href{https://doi.org/10.1109/tevc.2022.3199647}{doi:\nolinkurl{10.1109/tevc.2022.3199647}}


\bibitem[Ye et~al\mbox{.}(2022)]%
        {ye2022}
\bibfield{author}{\bibinfo{person}{Furong Ye}, \bibinfo{person}{Carola Doerr},
  \bibinfo{person}{Hao Wang}, {and} \bibinfo{person}{Thomas B{\"a}ck}.}
  \bibinfo{year}{2022}\natexlab{}.
\newblock \showarticletitle{Automated {{Configuration}} of {{Genetic
  Algorithms}} by {{Tuning}} for {{Anytime Performance}}}.
\newblock \bibinfo{journal}{\emph{IEEE Transactions on Evolutionary
  Computation}} \bibinfo{volume}{26}, \bibinfo{number}{6} (\bibinfo{date}{Dec.}
  \bibinfo{year}{2022}), \bibinfo{pages}{1526--1538}.
\newblock
\showISSN{1941-0026}
\href{https://doi.org/10.1109/tevc.2022.3159087}{doi:\nolinkurl{10.1109/tevc.2022.3159087}}


\bibitem[Zhang and Sanderson(2009)]%
        {zhang2009}
\bibfield{author}{\bibinfo{person}{Jingqiao Zhang} {and}
  \bibinfo{person}{Arthur~C. Sanderson}.} \bibinfo{year}{2009}\natexlab{}.
\newblock \showarticletitle{{{JADE}}: {{Adaptive Differential Evolution With
  Optional External Archive}}}.
\newblock \bibinfo{journal}{\emph{IEEE Transactions on Evolutionary
  Computation}} \bibinfo{volume}{13}, \bibinfo{number}{5} (\bibinfo{date}{Oct.}
  \bibinfo{year}{2009}), \bibinfo{pages}{945--958}.
\newblock
\showISSN{1941-0026}
\href{https://doi.org/10.1109/TEVC.2009.2014613}{doi:\nolinkurl{10.1109/TEVC.2009.2014613}}


\bibitem[Zhang et~al\mbox{.}(2022)]%
        {zhang2022a}
\bibfield{author}{\bibinfo{person}{Lu Zhang}, \bibinfo{person}{Bob Carpenter},
  \bibinfo{person}{Andrew Gelman}, {and} \bibinfo{person}{Aki Vehtari}.}
  \bibinfo{year}{2022}\natexlab{}.
\newblock \showarticletitle{Pathfinder: {{Parallel}} Quasi-{{Newton}}
  Variational Inference}.
\newblock \bibinfo{journal}{\emph{Journal of Machine Learning Research}}
  \bibinfo{volume}{23}, \bibinfo{number}{306} (\bibinfo{year}{2022}),
  \bibinfo{pages}{1--49}.
\newblock
\showISSN{1533-7928}
\urldef\tempurl%
\url{http://jmlr.org/papers/v23/21-0889.html}
\showURL{%
\tempurl}


\bibitem[Zhang et~al\mbox{.}(2015)]%
        {zhang2015}
\bibfield{author}{\bibinfo{person}{Tiantian Zhang}, \bibinfo{person}{Michael
  Georgiopoulos}, {and} \bibinfo{person}{Georgios~C. Anagnostopoulos}.}
  \bibinfo{year}{2015}\natexlab{}.
\newblock \showarticletitle{{{SPRINT Multi-Objective Model Racing}}}. In
  \bibinfo{booktitle}{\emph{Proceedings of the 2015 {{Annual Conference}} on
  {{Genetic}} and {{Evolutionary Computation}}}}
  \emph{(\bibinfo{series}{{{GECCO}} '15})}. \bibinfo{publisher}{Association for
  Computing Machinery}, \bibinfo{address}{New York, NY, USA},
  \bibinfo{pages}{1383--1390}.
\newblock
\showISBNx{978-1-4503-3472-3}
\href{https://doi.org/10.1145/2739480.2754791}{doi:\nolinkurl{10.1145/2739480.2754791}}


\end{thebibliography}

\appendix

\section{The Temporal Prior}
\label{app:helmert}

The temporal prior of \cref{sec:persistence} keeps the likelihood of \cref{sec:pl} and replaces its prior with one under which the utilities change smoothly with the budget.
With $\tau_t \in [0, 1]$ the position of budget $t$ between the first and the last budget of the grid, on the scale on which the grid is spaced, the ratings at $t$ are $\theta(t) = \operatorname{softmax}(\mu(t))$, where the utilities $\mu_i(t) = \log \theta_i(t)$, up to a constant, are
\begin{equation*}
	\mu(t) = \beta + H \bigl(\eta(\tau_t) - \bar\eta\bigr), \qquad \eta_q \sim \mathcal{GP}\bigl(0, \sigma^2 k_\ell\bigr), \quad q = 1, \dots, n - 1,
\end{equation*}
where $\bar\eta$ is the average of $\eta$ over the budgets of the grid and $k_\ell$ is the Mat\'ern 3/2 kernel~\cite{rasmussen2005} with length scale $\ell$.
The intercept $\beta$, each algorithm's average utility over the grid, is the centered logarithm of ratings drawn uniformly from the simplex, the prior of \cref{sec:pl}.
The $n - 1$ processes share the amplitude $\sigma \sim \mathrm{LogNormal}(0, 0.5)$ and the length scale $\ell \sim \mathrm{InverseGamma}(5, 2)$, and the $n \times (n - 1)$ Helmert matrix $H$ maps them to $n$ utilities that sum to zero, since utilities are identified only up to a constant at each budget.
The posterior is sampled by MCMC alone, without Pathfinder (\cref{sec:inference}).

The prior assumes the independence of irrelevant alternatives, as Plackett--Luce does; that the ratings change smoothly with the budget; and that, given the ratings, the rankings of one instance at different budgets are independent.
It is exactly specified when the ratings are a draw from this prior and the third assumption holds.

The Helmert matrix $H \in \mathbb{R}^{n \times (n-1)}$~\cite{lancaster1965} has the entries
\begin{equation*}
  H_{ik} =
  \begin{cases}
    1/\sqrt{k(k+1)} & \text{if } i \leq k, \\
    -k/\sqrt{k(k+1)} & \text{if } i = k + 1, \\
    0 & \text{if } i > k + 1,
  \end{cases}
  \qquad i = 1, \dots, n, \quad k = 1, \dots, n - 1.
\end{equation*}
Its columns are orthonormal, $H^\top H = I_{n-1}$, and each sums to zero, so $H$ maps $\mathbb{R}^{n-1}$ onto the vectors of $\mathbb{R}^n$ whose entries sum to zero.
The utilities of the temporal prior are confined to this subspace, since adding a constant to every utility at a budget leaves the ratings $\operatorname{softmax}(\mu(t))$ unchanged.
Together with $\mathbf{1}/\sqrt{n}$, the columns form an orthonormal basis of $\mathbb{R}^n$, so $H H^\top = I_n - \tfrac{1}{n}\mathbf{1}\mathbf{1}^\top$.
Independent processes $\eta_q$ with a common kernel therefore give the utilities at any two budgets a prior covariance proportional to $I_n - \tfrac{1}{n}\mathbf{1}\mathbf{1}^\top$, which treats all algorithms alike, although $H$ itself depends on their order.

\section{Algorithms and Their Configuration}
\label{app:cmaes}

All algorithms are our own implementations.
None of them is given the budget up front: no parameter depends on it, and a run continues until the experiment stops it externally.
The run of an algorithm to one budget is therefore a prefix of its run to any larger budget, so each label in \cref{tab:algorithms} names one anytime algorithm and not a family of them (\cref{sec:scale-free}).
Random search draws uniformly from the search domain, in batches of 32.

\begin{table}[tbp]
\caption[The algorithms of the pools]{The algorithms of the pools. Each pool also holds random search.}
\label{tab:algorithms}
\small
\begin{tabular}{@{}lll@{}}
\toprule
Label & Mechanism & Source \\
\midrule
\multicolumn{3}{@{}l}{\emph{MA-BBOB: CMA-ES under eleven rules for the step size}} \\
\alg{csa}     & cumulative step-size adaptation & \cite{hansen2001} \\
\alg{tpa}     & two-point adaptation & \cite{hansen2014} \\
\alg{msr}     & median success rule & \cite{aitelhara2013} \\
\alg{psr}     & population success rule & \cite{loshchilov2014} \\
\alg{xnes}    & rule of the exponential natural evolution strategy (NES) & \cite{glasmachers2010} \\
\alg{m-xnes}  & exponential NES rule, mean-step form & \cite{krause2017} \\
\alg{lp-xnes} & exponential NES rule, log-normal-prior form & \cite{krause2017} \\
\alg{ncsa}    & normalized cumulative adaptation$^{a}$ & \cite{akimoto2023} \\
\alg{hsa}     & hybrid adaptation & \cite{akimoto2008} \\
\alg{ppmf}    & previous-population midpoint fitness rule & \cite{warchulski2021} \\
\alg{none}    & fixed step size & \\
\midrule
\multicolumn{3}{@{}l}{\emph{GP-BBOB: CMA-ES under eight representations of the covariance matrix}} \\
\alg{cov}     & full matrix & \cite{hansen2001} \\
\alg{sep}     & diagonal & \cite{ros2008} \\
\alg{mat}     & square-root matrix of MA-ES & \cite{beyer2017} \\
\alg{cho}     & Cholesky factor & \cite{krause2016} \\
\alg{vd}      & diagonal with one direction & \cite{akimoto2014} \\
\alg{vkd}     & diagonal with several directions & \cite{akimoto2016a,akimoto2016} \\
\alg{lm}      & limited memory of direction vectors & \cite{loshchilov2017} \\
\alg{mf}      & archive without a matrix & \cite{arabas2025} \\
\midrule
\multicolumn{3}{@{}l}{\emph{GP-BBOB: differential evolution under three rules for its parameters}} \\
\alg{jade}    & JADE & \cite{zhang2009} \\
\alg{jde}     & jDE & \cite{brest2006} \\
\alg{shade}   & SHADE & \cite{tanabe2013} \\
\bottomrule
\end{tabular}

\smallskip
{\footnotesize $^{a}$Taken from a public implementation, since we know of no publication of it.}
\end{table}

\paragraph{CMA-ES}
The CMA-ES variants share one configuration.
The population is $\lambda = 4 + \lfloor 3 \ln d \rfloor$, which is 6, 8, 10, 12, 14 and 15 at the six dimensions of MA-BBOB and 20 at $d = 240$.
The first generation is drawn uniformly from the search domain, and its best point becomes the initial mean.
The initial step size is 0.2 of the range of each coordinate.
A candidate outside the domain is drawn again from the search distribution, up to 100 times, and clipped to the domain if it is still outside.
The variant with the full covariance matrix uses logarithmic recombination weights on the better half of the population and the active update with negative weights~\cite{hansen2023}.

A run consists of successive random (re-)starts of the strategy.
A start ends when its search distribution degenerates or its progress stalls.
The distribution counts as degenerate when the step size falls below $10^{-12}$ or rises above $10^{3}$ times its initial value, or when the covariance is ill-conditioned or no longer moves a candidate along some coordinate or axis.
Progress counts as stalled when the objective values of a generation are flat, or when they have not improved over $\min(20{,}000, \lfloor 100 + 100\, d^{1.5} / \lambda \rfloor)$ generations.
A variant applies the criteria that its representation of the covariance allows.
The next start doubles the population, as in IPOP-CMA-ES~\cite{auger2005}.
Unlike there, it begins with probability $\half$ at the best point found so far, perturbed by Gaussian noise with a standard deviation of 0.1 of each coordinate's range, and otherwise with a new uniform first generation.
Its initial step size is drawn from a log-normal distribution with median 0.2 and a standard deviation of 0.5 on the logarithmic scale.
Nothing else is carried over from one start to the next.

\paragraph{The two benchmarks}
On MA-BBOB the eleven variants use the full covariance matrix and differ in the rule that adapts the step size; one of them keeps the step size fixed.
On GP-BBOB the eight variants differ in how they represent the covariance matrix, and they do not all use the same step-size rule: \alg{cov}, \alg{sep}, \alg{mat}, \alg{cho} and \alg{vd} use cumulative step-size adaptation, \alg{vkd} two-point adaptation, \alg{lm} the population success rule, and \alg{mf} the midpoint fitness rule.

The three variants of differential evolution use a population of 100, which is constant within a start and doubles from one start to the next.
JADE and SHADE use the current-to-$p$best mutation with an archive of twice the population, jDE the rand/1 mutation, and all three binomial crossover; the scale factor and the crossover rate start at 0.5.
A coordinate outside the domain is drawn again near the bound it violated.
A start ends after 1{,}000 generations without an improvement of the best value found, or when the normalized diversity of the population has fallen to $10^{-8}$, and the next start begins as for CMA-ES.

\paragraph{Time per generation on GP-BBOB}
The budget of GP-BBOB is algorithm time, so the cost of a generation enters the comparison.
\Cref{tab:generation-times} gives it for each algorithm, with the evaluations that a run completes in 60\,s.
A time is the algorithm time between two recorded evaluations of a run's first start, divided by the number of generations between them, so it is an average over many generations.
Of the CMA-ES variants, only \alg{cov} makes its most expensive update less often than in every generation: it decomposes its covariance matrix every second or third generation at the dimensions of GP-BBOB, and its time counts that step at its share.
The populations differ between the three families, so that the times compare within a family.

The redraws of a candidate outside the domain count toward the algorithm time.
At the dimensions of GP-BBOB, a candidate drawn with a step size near the initial one has some coordinate outside the domain in nearly every draw, so that a generation uses all 100 redraws before its candidates are clipped, until the step size has decreased.
A draw is most expensive for \alg{mf}, which combines all vectors of its archive, so it is most penalized by this boundary rule.

\begin{table}[tbp]
\caption[Time per generation and evaluations on GP-BBOB]{Time per generation and evaluations in 60\,s of algorithm time on GP-BBOB: the median over the 5{,}000 runs of each algorithm, with the quartiles in brackets. The population is that of the first start; the CMA-ES variants have 18 to 21, depending on the dimension. A generation of \alg{mf} is a population and one further evaluation, and 4{,}782 of its runs have a time. For random search a generation is one batch.}
\label{tab:generation-times}
\small
\begin{tabular}{@{}lrr@{\;\;}lr@{\;\;}l@{}}
\toprule
Algorithm & Population & \multicolumn{2}{l}{Time per generation (ms)} & \multicolumn{2}{l}{Evaluations ($10^{3}$)} \\
\midrule
\alg{cov}    & 20  & 11.3 & [9.7, 13.1]   & 106.6 & [92.6, 122.8] \\
\alg{sep}    & 20  & 7.2  & [6.0, 8.5]    & 167.2 & [143.2, 200.7] \\
\alg{mat}    & 20  & 29.1 & [9.4, 34.3]   & 71.1  & [43.5, 146.5] \\
\alg{cho}    & 20  & 8.6  & [7.3, 10.1]   & 140.3 & [120.7, 164.6] \\
\alg{vd}     & 20  & 7.5  & [6.3, 8.8]    & 159.6 & [137.4, 188.5] \\
\alg{vkd}    & 20  & 8.1  & [6.8, 9.5]    & 160.7 & [137.5, 189.9] \\
\alg{lm}     & 20  & 7.8  & [6.5, 9.2]    & 178.8 & [146.9, 226.2] \\
\alg{mf}     & 20  & 78.5 & [14.5, 332.0] & 43.6  & [4.2, 88.2] \\
\midrule
\alg{jade}   & 100 & 34.3 & [28.4, 39.7]  & 175.4 & [151.4, 211.3] \\
\alg{jde}    & 100 & 28.6 & [23.5, 33.1]  & 209.4 & [181.0, 255.9] \\
\alg{shade}  & 100 & 34.0 & [28.0, 39.6]  & 176.7 & [151.8, 214.9] \\
\midrule
\alg{random} & 32  & 7.7  & [6.4, 9.1]    & 251.5 & [214.7, 300.1] \\
\bottomrule
\end{tabular}
\end{table}

\section{GP-BBOB}
\label{app:gpbbob}

GP-BBOB extends MA-BBOB~\cite{vermetten2023} by letting the weights of the mixture depend on the point $x \in [-5, 5]^d$.
An instance is
\begin{equation*}
  f(x) = \sum_{i \in S(x)} w_i(x)\, f_i(x), \qquad w_i(x) = \frac{\exp(z_i(x)/T)}{\sum_{j \in S(x)} \exp(z_j(x)/T)},
\end{equation*}
where $f_1, \dots, f_{24}$ are the BBOB functions, each with its own random instance transformation, its optimum drawn uniformly from $[-5, 5]^d$ and no offset, $S(x)$ is the set of the $k(x)$ functions with the largest logits $z_i(x)$, and $T$ is a temperature.
The logits are random Fourier features of $x/5$,
\begin{equation*}
  z_i(x) = \sqrt{2/m} \sum_{j=1}^{m} c_{ji} \cos\bigl(\omega_j^\top x/5 + b_j\bigr), \qquad \omega_j \sim \mathcal{N}(0, \ell^{-2} I_d), \quad b_j \sim \mathcal{U}[0, 2\pi), \quad c_{ji} \sim \mathcal{N}(0, \sigma^2/m),
\end{equation*}
approximately a draw from a Gaussian process (GP, hence the name) with a squared-exponential kernel of length scale $\ell$ and variance $\sigma^2/m$.
A second, scalar field $h(x)$ of the same form, with $m_k$ features, length scale $\ell_k$ and coefficient variance $\sigma_k^2/m_k$, sets the number of functions: $k(x)$ is the smallest $k$ with $F(k) \geq \Phi(h(x))$, where $\Phi$ is the standard normal distribution function and $F$ that of a log-normal distribution with parameters 1.1 and 0.5 truncated to $[2, 24]$.
Where the selected set changes, $f$ jumps, and its optimum is deliberately not known.
\Cref{tab:priors-gpbbob} gives the distribution from which each instance's parameters are drawn.

\begin{table}[tbp]
\caption{The distribution of a GP-BBOB instance's parameters. $\mathrm{LN}(\mu, s)$ is the log-normal distribution with parameters $\mu$ and $s$ of the underlying normal, and $\mathrm{U}\{a, \dots, b\}$ the uniform distribution on the integers from $a$ to $b$.}
\label{tab:priors-gpbbob}
\small
\begin{tabular}{@{}lll@{}}
\toprule
Parameter & Distribution & Note \\
\midrule
$d$, dimension & $\lfloor \mathcal{N}(240, 30^2) \rfloor$ & redrawn until in $[2, 320]$ \\
instances of $f_1, \dots, f_{24}$ & $\mathrm{U}\{1, \dots, 2^{31} - 1\}$ & independently for each \\
$T$, temperature & $\mathrm{LN}(-0.5, 0.5)$ & \\
$m$, features of the logits & $\mathrm{U}\{32, \dots, 128\}$ & \\
$\ell$, length scale of the logits & $\mathrm{LN}(3, 1)$ & on $x/5$ \\
$\sigma$, scale of the logits & $\mathrm{LN}(0, 0.7)$ & coefficient variance $\sigma^2/m$ \\
$m_k$, features of $h$ & $\mathrm{U}\{16, \dots, 64\}$ & \\
$\ell_k$, length scale of $h$ & $\mathrm{LN}(0, 1)$ & on $x/5$ \\
$\sigma_k$, scale of $h$ & $\mathrm{LN}(0, 0.7)$ & coefficient variance $\sigma_k^2/m_k$ \\
\bottomrule
\end{tabular}
\end{table}

\section{Truths Drawn from a Model's Own Prior in Detail}
\label{app:prior-truths}

This appendix gives what \cref{sec:prior-truths} leaves out.
\Cref{tab:prior-truths} lists, for every model, size and test, the directional statements' $X/\Lambda$ with its 95\,\% interval over races, and the false eliminations and false equivalences beside their bounds.
It includes the per-pair model raced on pairwise observations, for the comparison with Bradley--Terry.
Pooled over the two sizes, Plackett--Luce makes 8 false eliminations against a bound of 7.82 under the joint test, Bradley--Terry 7 against 7.25, the per-pair model 9 against 8.10, and the near-ignorance prior 5 against 5.39; under the pointwise test, 9, 11, 11 and 5 are false, against bounds of 9.34, 8.95, 9.80 and 8.91.
The largest excess of false eliminations over their bound is the near-ignorance prior's in the large size under the joint test: 2 of its 116 eliminations are false, in two races, against a bound of 0.32.
False equivalences number at most 5 in any row, too few for intervals.

A race is stopped after 2{,}165 blocks, the number of instances that a pair with win probability $\half + \epsilon/2$ is estimated to need.
After $N$ instances, the posterior standard deviation of its win probability is about $s_1/\sqrt{N}$, where $s_1 = \half$ is the standard deviation of one outcome at a win probability of $\half$.
A relation is declared once $z_\alpha$ such standard deviations, with $z_\alpha$ the standard normal quantile at $\alpha$, fit within half a rope, $z_\alpha s_1/\sqrt{N} \leq \epsilon/2$, which gives $N = \lceil (2 z_\alpha s_1 / \epsilon)^2 \rceil = 2{,}165$ at $\alpha = 0.99$ and $\epsilon = 0.05$.

A truth drawn from the Dirichlet process puts its weight on a few instance types, each with its own ranking at every budget, and an instance is a draw of one type; a race meets about five of them.
The near-ignorance prior's interval in the large size shows why the intervals resample races: its 35 false directional statements come from 6 of the 105 races, and an interval that treated them as independent would be [0.80, 1.60].

A fit that fails the diagnostic checks of \cref{sec:inference} is repeated with more effort, and one whose Monte Carlo error is too large for any probability to pass decides nothing (\cref{sec:procedure}).
In the small size, 17\,\% to 22\,\% of Plackett--Luce's fit attempts fail a check, and 23\,\% to 25\,\% of Bradley--Terry's.
In the large size, 42\,\% of Plackett--Luce's attempts fail a check and 4\,\% to 6\,\% more have a Monte Carlo error too large for any probability to pass, and 48\,\% to 51\,\% of Bradley--Terry's fail a check, with 8\,\% more of too large an error.
The per-pair model has no check to fail, since its posterior is in closed form.

\begin{table}[tbp]
\caption[False statements on truths drawn from a model's own prior]{False statements on truths drawn from a model's own prior, by model, size and test: $X/\Lambda$ for directional statements with 95\,\% bootstrap intervals over races, and false eliminations and false equivalences beside their bounds, as false / bound. ``Per-pair, pairwise'' is the per-pair model raced on the pairwise observations of Bradley--Terry. The tie-point prior has no guarantee (\cref{sec:dp}).}
\label{tab:prior-truths}
\small
\begin{tabular}{@{}llrllll@{}}
\toprule
Model & Size & Races & Test & Directional $X/\Lambda$ & Eliminations & Equivalences \\
\midrule
Plackett--Luce & small & 1{,}347 & joint & 1.05 [0.69, 1.43] & 7 / 7.12 & 3 / 2.29 \\
 &  &  & pointwise & 1.03 [0.69, 1.42] & 8 / 8.37 & 5 / 2.28 \\
 & large & 208 & joint & 0.92 [0.59, 1.29] & 1 / 0.70 & 1 / 0.76 \\
 &  &  & pointwise & 1.07 [0.70, 1.45] & 1 / 0.97 & 1 / 0.74 \\
\addlinespace
Bradley--Terry & small & 1{,}310 & joint & 1.20 [0.81, 1.61] & 7 / 6.74 & 3 / 1.87 \\
 &  &  & pointwise & 1.20 [0.81, 1.62] & 10 / 8.22 & 4 / 1.84 \\
 & large & 154 & joint & 0.99 [0.62, 1.40] & 0 / 0.50 & 1 / 0.35 \\
 &  &  & pointwise & 1.03 [0.66, 1.45] & 1 / 0.73 & 1 / 0.27 \\
\addlinespace
Per-pair & small & 1{,}063 & joint & 1.01 [0.66, 1.39] & 8 / 7.41 & 1 / 2.48 \\
 &  &  & pointwise & 1.03 [0.67, 1.41] & 10 / 8.79 & 3 / 2.34 \\
 & large & 129 & joint & 0.86 [0.52, 1.25] & 1 / 0.69 & 0 / 0.55 \\
 &  &  & pointwise & 0.80 [0.47, 1.17] & 1 / 1.01 & 0 / 0.50 \\
\addlinespace
Per-pair, pairwise & small & 1{,}275 & joint & 1.22 [0.86, 1.61] & 12 / 8.56 & 4 / 2.65 \\
 &  &  & pointwise & 1.00 [0.67, 1.36] & 10 / 10.31 & 3 / 2.55 \\
\addlinespace
Near-ignorance & small & 1{,}052 & joint & 0.43 [0.13, 0.87] & 3 / 5.07 & 0 / 0.97 \\
 &  &  & pointwise & 0.40 [0.10, 0.84] & 3 / 8.08 & 0 / 0.94 \\
 & large & 105 & joint & 1.15 [0.26, 2.38] & 2 / 0.32 & 0 / 0.50 \\
 &  &  & pointwise & 1.12 [0.22, 2.34] & 2 / 0.83 & 0 / 0.38 \\
\addlinespace
Tie-point & small & 1{,}052 & joint & 2.03 [0.92, 3.42] & 3 / 3.65 & 0 / 0.82 \\
 &  &  & pointwise & 2.28 [1.05, 3.86] & 4 / 4.48 & 0 / 0.81 \\
 & large & 105 & joint & 8.17 [1.50, 17.4] & 2 / 0.29 & 0 / 0.38 \\
 &  &  & pointwise & 8.16 [1.55, 17.5] & 2 / 0.59 & 0 / 0.40 \\
\bottomrule
\end{tabular}
\end{table}

\section{Jackknife Calibration and the Dependence Study}
\label{app:calibration}

For a quantity $f$ of the posterior draws, such as a difference of utilities at a budget, let $c_k$ be the posterior covariance of $f$ with $\ell_k$, the summed log-likelihood of the rankings of instance $k$, and $\bar c$ the mean of the $c_k$ over the fit's $N$ instances.
The infinitesimal jackknife variance of $f$ is $V_{\mathrm{IJ}}(f) = \frac{N}{N - 1} \sum_{k=1}^{N} (c_k - \bar c)^2$, and its calibration ratio is $r = V_{\mathrm{IJ}}(f) / \operatorname{Var}_{\mathrm{post}}(f)$.
Every draw $d$ of $f$ is replaced by $m + \sqrt{\max(1, r)}\,(d - m)$, where $m$ is the posterior mean of $f$, so calibration only ever widens.
An average over budgets is formed before it is scaled, so its ratio carries the covariance across budgets that shared instances create, and a joint probability over budgets is the share of draws in which every scaled per-budget quantity passes.
The $c_k$ are estimated from the draws, and the Monte Carlo error this adds to $V_{\mathrm{IJ}}$ is not corrected.
Each quantity is scaled by its own ratio, so a calibrated draw no longer orders the algorithms as one set of ratings does, and a calibrated rating model, unlike an uncalibrated one, can make a cycle of eliminations (\cref{sec:resolution}).

The rest of this appendix follows the paragraphs of \cref{sec:persistence} and gives what they leave out.
\Cref{fig:dependence-truths} shows one truth of the study and instances drawn from it at each strength.
An instance adds Gumbel noise to every utility at every budget, as in \cref{sec:pl}, and the noise of one algorithm on one instance follows a Gaussian copula whose correlation between the budgets at $\tau$ and $\tau'$ is $\rho^{9|\tau - \tau'|}$, so that neighboring budgets of the 10-budget grid have correlation $\rho$.
The strength $\rho$ is not a property of the truth, whose utilities are as smooth at $\rho = 0.99$ as at $\rho = 0$; it sets how far the rankings of one instance at different budgets agree, while at a single budget the ranking follows Plackett--Luce at every strength.
On 20 and 30 budgets, neighboring budgets lie closer together than on 10 and correlate more, with $\rho^{9/19}$ and $\rho^{9/29}$.
On these truths, the independent prior's $X/\Lambda$ for directional statements is 0.76 [0.37, 1.19], 0.94 [0.51, 1.45], 0.85 [0.17, 1.97] and 1.83 [0.32, 3.93] at $\rho = 0$, 0.5, 0.9 and 0.99 under the joint test, and 0.83 [0.43, 1.28], 1.09 [0.60, 1.65], 0.93 [0.23, 2.08] and 1.93 [0.34, 4.13] under the pointwise test.
Under the joint test, the races with a false statement drop from 12 and 14 of 90 at $\rho = 0$ and 0.5 to 7 and 7 at 0.9 and 0.99, while the bound stays near 17 at every strength; at $\rho = 0.99$, 33 false statements come from 7 races.
The pointwise test makes 462 eliminations, 1 of them false, against a recorded bound of 4.72, and the joint test 460, none false, against 1.32.
Under the pointwise and the joint test, 176 and 177 of the 360 races end at the cap of 256 blocks.
\Cref{tab:coverage} gives the coverage of the intervals at fixed sizes for both priors, uncalibrated and calibrated, at every strength and size.

The calibrated independent prior races all 360 truths, with an $X/\Lambda$ of 0.97 [0.58, 1.46].
The two calibrated priors are compared on the 270 truths that both race, 45 at each of $\rho = 0$ and 0.5 and 90 at each of 0.9 and 0.99.
On these, the temporal prior makes more directional statements, 22{,}989 against 21{,}301, at higher recorded probabilities: their bound is 36.3 against 53.1, while twice as many are false, 107 in 39 races against 54 in 26.
Race by race, the temporal prior's ratio is 2.9 [2.2, 4.0] times the independent prior's.
Where the temporal prior is exactly specified, at $\rho = 0$, its ratio is 1.60 [0.27, 3.23], against 0.77 [0.23, 1.39] for the independent prior.
These rest on 45 truths only, and the temporal prior's interval still includes 1, so it is consistent with its bound there.
Under dependence, the ratios are 3.10 [1.60, 5.16] and 1.07 [0.47, 1.83].
Of these races, 134 of the temporal prior's and 128 of the independent prior's finish before the cap, both on 119, where a race takes 3{,}202 CPU-seconds under the temporal prior and 1{,}397 under the independent prior.
Against the uncalibrated independent prior, both finish on 122, where the temporal prior needs 0.91 [0.83, 0.99] of the instances at 3.7 times the compute.

\begin{table}[tbp]
\caption[Coverage at fixed sizes in the dependence study]{The share of central 90\,\% intervals of a difference of two algorithms' utilities at a budget that cover the true difference, after the first 20 or 80 instances of the truths of \cref{sec:persistence}, with 95\,\% bootstrap intervals that resample truths.}
\label{tab:coverage}
\small
\begin{tabular}{@{}rrllll@{}}
\toprule
& & \multicolumn{2}{c}{Temporal prior} & \multicolumn{2}{c}{Independent prior} \\
\cmidrule(lr){3-4}\cmidrule(l){5-6}
$\rho$ & Instances & uncalibrated & calibrated & uncalibrated & calibrated \\
\midrule
$0$ & 20 & 0.89 [0.88, 0.91] & 0.90 [0.88, 0.91] & 0.89 [0.88, 0.90] & 0.90 [0.89, 0.91] \\
$0$ & 80 & 0.89 [0.88, 0.91] & 0.89 [0.88, 0.91] & 0.90 [0.89, 0.91] & 0.91 [0.90, 0.91] \\
$0.5$ & 20 & 0.72 [0.69, 0.74] & 0.86 [0.84, 0.88] & 0.89 [0.87, 0.90] & 0.89 [0.88, 0.91] \\
$0.5$ & 80 & 0.75 [0.73, 0.77] & 0.87 [0.85, 0.89] & 0.89 [0.88, 0.90] & 0.90 [0.88, 0.91] \\
$0.9$ & 20 & 0.61 [0.58, 0.65] & 0.84 [0.82, 0.87] & 0.88 [0.86, 0.90] & 0.89 [0.87, 0.91] \\
$0.9$ & 80 & 0.68 [0.65, 0.71] & 0.87 [0.84, 0.89] & 0.89 [0.88, 0.91] & 0.90 [0.88, 0.92] \\
$0.99$ & 20 & 0.59 [0.55, 0.62] & 0.87 [0.84, 0.90] & 0.90 [0.87, 0.93] & 0.91 [0.88, 0.93] \\
$0.99$ & 80 & 0.60 [0.57, 0.64] & 0.87 [0.84, 0.89] & 0.89 [0.87, 0.92] & 0.90 [0.88, 0.92] \\
\bottomrule
\end{tabular}
\end{table}

\begin{figure}[tbp]
\centering
\includegraphics[width=\textwidth]{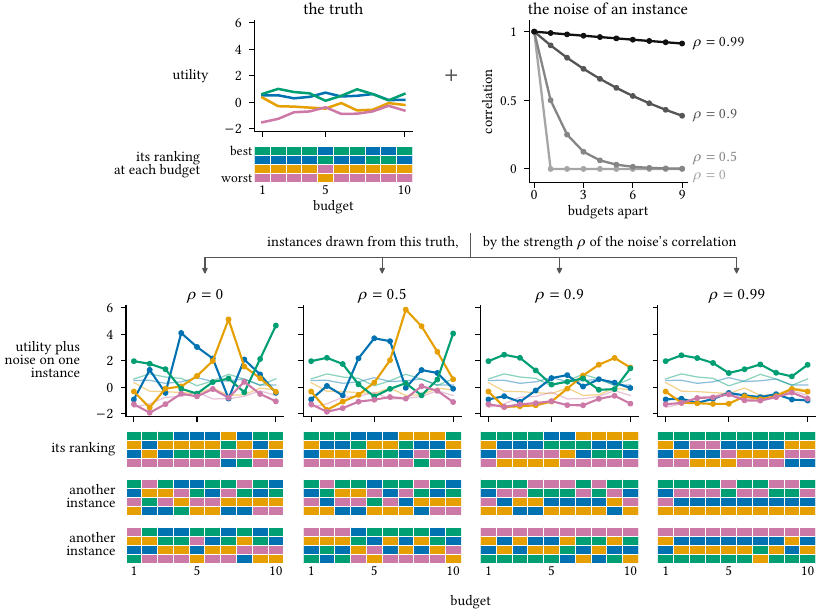}
\caption[One truth of the dependence study and instances drawn from it]{One truth of the dependence study on 10 budgets, and instances drawn from it. Top: the truth, as the utilities of four algorithms, one color each, with the ranking they imply at each budget; beside it, the correlation of an instance's noise at two budgets by their distance, for the four strengths $\rho$. Below, per strength: one instance's utilities plus noise, with the truth in thin lines; the ranking this instance yields at each budget; and the rankings of two more instances. The truth is the same in every column and the same noise draws serve every strength, so the columns differ only in the dependence; the study itself draws a new truth and new noise for every strength. At $\rho = 0$ an instance's ranking changes from budget to budget; at $\rho = 0.99$ an instance keeps nearly one ranking, changing it only where the truth's utilities lie close together, and a different one from the next instance.}
\Description{On top, a panel of four utility curves over ten budgets with a strip of colored tiles for their ranking at each budget, and a panel of four curves of correlation against distance, one per strength. Below, four columns, one per strength, each with a panel of four noisy curves over the faint utility curves and three strips of colored tiles. The noisy curves and the strips change from budget to budget at strength zero and stay nearly constant at strength 0.99.}
\label{fig:dependence-truths}
\end{figure}

\section{The Races on the Pools in Detail}
\label{app:replays}

This appendix gives what \cref{sec:replays} leaves out, in the order of its subsections.

\subsection{Setup}

\Cref{tab:settings} lists the settings of the races and, by pool, the fixed cost of a round, $\kappa$.
For Plackett--Luce it is the duration of its fits in a preliminary race, and the per-pair model and the two Dirichlet process priors take the same value.
The four then decide at a similar rhythm: by pool, the median race takes 8 to 28.5 rounds under Plackett--Luce and 8 to 25.5 under the three pair models.
Bradley--Terry takes its own $\kappa$, from the duration of its own fits, and starts from one comparison per pair.
Whether a smaller $\kappa$ would end the races of the pair models earlier is not measured over all budgets; at one budget, a decision after every instance ends the per-pair races after 588 instances in the mean, against 797 as recorded (\cref{sec:racers}).

\subsection{MA-BBOB}

The race of \cref{fig:one-race} eliminates \alg{random} after 17 instances, \alg{lp-xnes} and \alg{none} after 59, \alg{xnes} after 174, \alg{ppmf} after 199 and \alg{msr} after 219.
Its last elimination, \alg{csa} against \alg{msr}, has most of the posterior mass with \alg{csa} ahead by more than the rope at all 30 budgets, and the probability that \alg{csa} dominates \alg{msr} reaches 0.997 at instance 219.
Its first crossing, \alg{m-xnes} against \alg{tpa}, has most of the mass with \alg{m-xnes} ahead at 15 budgets and behind at 13, and is resolved in the first round, after 8 instances.
The two pairs that end the race have most of their mass within the rope at 14 and at 13 of the 30 budgets after 724 instances.
For \alg{hsa} against \alg{csa}, the race declares that \alg{hsa} beats \alg{csa} at 13 budgets and the reverse at 3; for \alg{csa} against \alg{ncsa}, it declares that \alg{csa} beats \alg{ncsa} at 18 budgets and that the two are equivalent at 2, which also makes the pair a crossing (\cref{sec:resolution}).

\Cref{tab:one-race-models} gives the races of the other four models on the same seed; all five return the same set.
Under Plackett--Luce the race eliminates the same six algorithms earlier, after 8 to 138 instances against 17 to 219, because the rating scale pools the evidence across pairs.
It still ends after 710 instances, about as many as under the per-pair model, because the last pairs among the Pareto set take the same data under both.
The pool means agree in direction, $2.34 \times 10^{8}$ evaluations against $3.05 \times 10^{8}$ (\cref{tab:cost}); what the earlier decisions cost in false statements is the subject of \cref{sec:five-models}.

The race under Bradley--Terry is the longest of the five because \alg{xnes} is eliminated last, in the race's last round, after 2{,}247 comparisons.
In the truth, the four algorithms that dominate \alg{xnes}, \alg{csa}, \alg{ncsa}, \alg{hsa} and \alg{tpa}, lead it by only 0.55 to 0.58 at the last budget, so its elimination waits for that budget: the race declares that \alg{csa} and \alg{ncsa} beat it at the other 29 budgets within 867 and 1{,}230 instances, and at the last only at instance 4{,}230.
Under Bradley--Terry an instance is one comparison of one pair, so this evidence comes one pair at a time, where the other models compare \alg{xnes} with every remaining algorithm on each instance and eliminate it after 29 to 188 instances.
A race also decides only at the end of a round: after the round that ends at 1{,}488 instances, with seven algorithms and five open pairs left, the allocation plans a round of 2{,}742 comparisons, and at its end all five pairs are resolved.
Over the 20 races of this pool, \alg{xnes} is the last elimination under Bradley--Terry in 13, and is eliminated in the race's last round in 9.

\Cref{tab:leaders} gives the leaders from $d = 20$ on, in the order of \cref{sec:truth}.
The order has three departures: \alg{tpa} also leads at the first budget at $d = 30$ and at the first two at $d = 40$, \alg{csa} and \alg{ncsa} alternate once more at $d = 40$ before \alg{tpa} takes over, and one budget at $d = 20$, 38{,}566 evaluations, has no leader.
At $d = 2$, the member of the Pareto set that comes closest to beating \alg{ncsa} at every budget is \alg{tpa}, whose win probability against it falls to 0.460 at its lowest; for \alg{csa} the lowest is 0.455, for \alg{hsa} 0.434 and for \alg{xnes} 0.393 (\cref{fig:ncsa}).

\subsection{Comparing Models across All Experiments}

\Cref{tab:results-pools} gives the columns of \cref{tab:results} by pool.
Whether a set that is correct within the rope is also exact depends mostly on the pool: no race of any model returns exactly the Pareto set at $d = 2$, where every race also returns a member of the widened set outside it, and every race does at $d = 30$.
The false equivalences of Plackett--Luce come from every pool but $d = 30$, 110 of its 214 from $d = 2$, and those of Bradley--Terry from $d = 2$ and $d = 5$ alone; the three pair models make at most two on any pool.
By pool, Bradley--Terry costs 1.47 to 2.75 times as much as the per-pair model, most at $d = 5$, and the two Dirichlet process priors range from 0.85 of it, the tie-point prior at $d = 2$, to 1.15, the near-ignorance prior at $d = 40$.

Under the pointwise test, the per-pair races at $d = 5$ return sets that are correct within the rope in all 20 races, and the same set as under the joint test in 19; the other leaves out \alg{psr}, a member of the widened set outside the Pareto set.
They make 37 false directional statements against a bound of 58.6, where the joint test makes 40 against 60.9, and neither test makes a false elimination.
They cost 0.66 [0.62, 0.70] of the joint test's evaluations and less in each of the 20 races, although the median race draws 1.08 times the instances, since a run stops at the largest budget at which one of its open pairs holds no relation (\cref{sec:tests}): the runs of the pointwise races use 0.69 of the budget on average.
At $d = 30$, both tests return exactly the Pareto set in all 20 races, with 41 false directional statements against 65.6 under the pointwise test and 50 against 77.5 under the joint test, and the pointwise test costs 1.00 [0.94, 1.06] of the joint test's evaluations.
Its lower recorded probabilities (\cref{sec:tests}) raise its bound on false eliminations per race to 0.095 at $d = 5$ and 0.10 at $d = 30$, against 0.012 and 0.028 under the joint test.

Under Bradley--Terry an instance is one comparison, so its median race draws 2{,}610 to 10{,}307 instances by pool, against 592 to 1{,}369 under the other four models.
A race at the cap ends at its last round before 2{,}165 instances, since the allocation sets how far apart the rounds are: the capped races of the four models end after 1{,}221 to 2{,}163 instances, and those of Bradley--Terry, whose cap counts the comparisons of one pair, after 1{,}535 to 9{,}167.
Of the fits of Plackett--Luce, 3{,}478 of 5{,}391 fail the diagnostic checks of \cref{sec:inference}, 57\,\% to 69\,\% by pool; all of them are fits by Pathfinder, and all 800 fits by MCMC pass.
Of those of Bradley--Terry, 25{,}778 of 36{,}010 fail, all but one by Pathfinder, and no fit of the three pair models fails.

After 100, 300 and 1{,}000 instances, the per-pair races of a pool still have a median of 9.5 to 37, 2 to 19 and 0 to 1 open pairs among their remaining algorithms (\cref{fig:open-pairs} shows all five models), and under any model at most 4 of a pool's 20 races end within 300 instances.
Of the 480 races of the four models on MA-BBOB, one makes more runs than the fixed design of 12{,}000, and the 228 that draw more than 1{,}000 instances are none at $d = 30$ and 30 to 56 of 80 on each other pool.
Bradley--Terry makes more runs than the fixed design in 43 of its 120 races on MA-BBOB, at $d = 2$, 5 and 20.

The cost of inference also shows in compute, in CPU-hours per race.
Plackett--Luce takes 1.2 to 6.4 on the six MA-BBOB pools and 4.0 on GP-BBOB, and Bradley--Terry 6.7 over its races at $d = 2$, 10, 20 and 40, 9.3 at $d = 30$ and 8.5 on GP-BBOB.
The two Dirichlet process priors take 0.15 to 0.41, in runs that race them together, and the per-pair model 0.18 under the pointwise test at $d = 5$ and 0.44 on GP-BBOB.
On GP-BBOB, this compute is small beside the algorithm runs a race draws, which take 59.0 hours of algorithm time under Plackett--Luce, 109.0 under Bradley--Terry and 64.9 under the per-pair model (\cref{tab:cost}), so the compute of the race adds 7\,\%, 8\,\% and less than 1\,\% to it.
The compute of the other races was not measured: the per-pair model under the joint test on MA-BBOB, the near-ignorance prior at $d = 30$, and Bradley--Terry at $d = 5$.

\subsection{Why the Models Differ}

Against their limit, the directional statements of Plackett--Luce are furthest over their bound on GP-BBOB, at 1.93 [0.62, 3.84], where they are at 2.03 against the truth; on that pool the limit changes little for Bradley--Terry either, whose ratio moves from 1.93 to 1.71.
Scored against the limit, 81 statements of Plackett--Luce and 107 of Bradley--Terry that are true of the truth become false.
The limit is taken over the algorithms that a round still holds; against the limit over all twelve algorithms, the ratio of Plackett--Luce is 1.64 [1.28, 2.11] and that of Bradley--Terry 1.21 [0.94, 1.58].
For equivalences, the limit lowers the ratio of Plackett--Luce from 20.55 to 2.21 [1.33, 3.18], 23 false against a bound of 10.41, and that of Bradley--Terry from 4.57 to 1.52 [0.54, 2.85], 8 false against 5.25, which is no longer clearly over its bound.
Of the 23, 11 are at $d = 40$.
Under a rating model, a false equivalence is a pair whose limit lies closer to one half than its truth; Bradley--Terry makes such equivalences only at $d = 2$ and $d = 5$.

A rating model can declare that $A$ beats $B$ at a budget where the win share of $A$ over $B$ on the race's own instances is at most one half, since its estimate for one pair draws on the others; the races record these statements and keep them.
Plackett--Luce makes 546 such statements, and 273 of them are false, 27\,\% of its 1{,}024 false directional statements, 225 of them at $d = 2$ and $d = 5$.
Bradley--Terry makes 7{,}609, of which 81 are false, against 625 false directional statements in all.
The pair models make none, since their posterior for a pair rests on that pair's wins and losses alone.

\Cref{tab:why-pools} gives, by pool, the cycles at single budgets and the ties in the truth, and when the tie-point prior makes its false statements.
The seven cycles lie at five budgets; the one at $d = 20$ lies at 38{,}566 evaluations, the budget without a leader in \cref{tab:leaders}.
In the 6{,}757 rounds of the pair models' 420 races, none holds a cycle of dominances, nor even of pairs in which one algorithm has more wins than losses at every budget.
Ties are most frequent at the last budget, where runs reach the optimum to machine precision: on every MA-BBOB pool, at least 98.9\,\% of the equal values at the last budget are exactly zero.
On GP-BBOB, whose optima are unknown, ties all but do not occur, with 193 equal values among 9.9 million comparisons.
Where ties are frequent, the tie-point prior makes more of its false statements in the first round and more with a recorded probability of exactly 1: 88\,\% and 84\,\% at $d = 2$, against 56\,\% to 69\,\% and 21\,\% to 32\,\% from $d = 20$ on.
After the first round, its ratio is clearly over the bound only at $d = 2$ and $d = 5$, at 2.37 [1.40, 3.46] and 1.79 [1.05, 2.56] (\cref{fig:why}).

The races bear out that the probability of a directional statement under the tie-point prior does not depend on its strength (\cref{sec:dp}): at $d = 30$, its ratio is 2.89, 2.76 and 2.88 at the strengths $s = 0$, 1 and 8.
Nor does the dependence across budgets that its shared weights carry explain its excess.
At $d = 30$, a race with the product of the per-budget probabilities in place of the joint probability returns the same set in all 20 races, draws the same instances in 17, and makes the same directional statements in 15 and the same evaluations in 14.
The medians agree, at 659 instances and $3.00 \times 10^{8}$ evaluations, and so does the ratio of false directional statements to their bound, 2.76.

\subsection{Comparing \textsc{PolarBear} with F-Race and B-Race}

\Cref{tab:racers-one} sets the three racers and \textsc{PolarBear} side by side at one budget, step by step.
For B-race (Benavoli tests), it also fixes the four points that \citet{benavoli2015} leave open.
The strength of the prior is the weight of one instance, and after a comparison every survivor but the best that loses an accepted statement leaves.
A pair counts as indistinguishable once the posterior probability that its win probability lies within the rope reaches the threshold, and its member with the larger rank sum then leaves.

To see what the rounds cost at one budget, we rebuild the decisions of the per-pair races from the instances they drew; decided at the recorded rounds, the rebuilt races reproduce all 20 recorded races in their eliminations, instances and evaluations.
Decided after every instance from the eighth on, they draw 588 instances in the mean, against 797 as recorded and 642 for B-race (pairwise tests), and spend $1.82 \times 10^{8}$ evaluations, against $2.66 \times 10^{8}$ and $2.05 \times 10^{8}$.
In the four races that end before B-race's, \alg{hsa} eliminates \alg{ncsa} after 47 instances, which B-race's best, \alg{tpa}, eliminates after 221; \alg{ncsa} eliminates \alg{xnes} after 571 instances, against 1{,}031; \alg{csa} eliminates \alg{xnes} after 280, against 637; and one race stops after 1{,}211 instances with \alg{tpa} and \alg{ncsa} equivalent, where B-race eliminates \alg{ncsa} after 1{,}646.
Of its eliminations, 5 of 219 are false of the pool's win probabilities, against a bound of 1.53; in the recorded races, 4 of 220 against 0.77, and for B-race (pairwise tests), 3 of 220 against 1.51.
This does not show what the allocation of \cref{sec:procedure} would choose at $\kappa = 0$, which need not be a decision after every instance.

F-race's test for two survivors drops the instances on which the two values are equal, ranks the absolute differences on the others, and sums the ranks of each sign; an algorithm's share of the rank sum is the part of the total that its wins hold.
Over the 5{,}000 instances of the $d = 30$ pool at its last budget, \alg{tpa} has a win probability and a share of 0.548 and 0.485 against \alg{ncsa}, 0.566 and 0.507 against \alg{csa}, and 0.580 and 0.528 against \alg{hsa}.
All 16 races that F-race ends at the cap are down to two algorithms after 50 to 560 instances.
Of the 16 races at the cap, 12 are left with \alg{tpa} and \alg{ncsa}, 2 with \alg{tpa} and \alg{csa}, and 2 with \alg{tpa} and \alg{hsa}; only \alg{ncsa} lies within the rope of \alg{tpa}, so the 12 sets are correct within the rope and the other four are not.
Of the four races that finish, three return \alg{tpa}, an exact set, and one returns \alg{ncsa} after eliminating \alg{tpa} at instance 1{,}485.

Near-ties are of two kinds: a change of leader puts one at the budgets around it, 55 of the 89 budgets inside the rope, 2 of them without a leader, and the other 34 lie inside stretches where several algorithms stay within the rope of each other, such as the budgets of $d = 30$ from 4{,}520 evaluations on, where \alg{hsa}, \alg{csa}, \alg{ncsa} and \alg{tpa} lead in turn.
Next to a change of leader, F-race and the two B-races end 49\,\%, 46\,\% and 49\,\% of their races at the cap, inside a stretch 42.5\,\%, 21\,\% and 29\,\%, and outside the rope 18\,\%, and 2 and 1 of 1{,}820.
B-race (Benavoli tests) has its rule for indistinguishable pairs, with the same rope of 0.05, and still ends at the cap at near-ties.
By a normal approximation for a pair without ties, showing at a threshold of 0.99 that a win probability of 0.53 lies inside the rope takes about 3{,}400 instances, more than the cap of 2{,}165; the racers' posterior differs from this approximation in its prior and in ties, which lower the variance.

Outside the rope, the 330 races of F-race at the cap are down to two survivors after 47, 108 and 192 instances at the quartiles, and at most after 821, and spend 94\,\% of their instances under the Wilcoxon test.
All 330 end with two survivors, the leader among them in 328.
Over these races, the better algorithm's win probability has the quartiles 0.558, 0.563 and 0.569, and its share of the rank sum 0.492, 0.500 and 0.513; the pairs are most often \alg{hsa} and \alg{csa} at $d = 5$ to 40, in 212 races, and \alg{xnes} and \alg{hsa} at $d = 2$, in 54.
The better algorithm's share of the rank sum is below one half in 164 of the 330 races.
The sign test we put in the Wilcoxon test's place is two-sided and exact, drops ties, and runs at significance 0.01 after every instance from the one at which the race is down to two; it decides its 329 races after 211, 331 and 536 instances at the quartiles.
Inside the rope, of the 823 races at the cap, 633 end with two survivors, 162 with three, 21 with four and 7 with five, and only the 633 reach the Wilcoxon test; the sign test decides 308 of them, 302 for the algorithm with the higher win probability and 6 for the other.
With the sign test, 515 of the 1{,}780 races inside the rope still end at the cap, against 823.
Over all budgets, F-race makes 2{,}038 eliminations with two survivors, 9 of them without a test because one algorithm is nowhere worse than the other.
Of these, 333 are false, and 326 of the 333 agree with the side the rank sum takes in the pool; while more than two algorithms survive, 138 of its 36{,}184 eliminations are false, 0.4\,\%.
The two questions do not differ through the scale of the values alone: on the logarithm of the value plus $10^{-8}$, the rank sum changes side for 47 of the 1{,}958 pairs of a leader and a rival over the budgets of the six pools, and favors the leader for 1{,}915 of them, against 1{,}878 on the values, so the size of a win counts on either scale.

Counting only the races that finish, the union of a racer's survivors over the budgets misses a member of the Pareto set in 67, 81 and 82 of the 120 repetitions, against 5, 25 and 23 with the races that end at the cap.
All three racers return \alg{ncsa} at $d = 2$ only through races that end undecided: from finished races alone it is missing in 18, 20 and 20 of the 20 repetitions.

\subsection{Selection from the Pareto Set}

The one combination of pool and preference whose truth has no single preferred algorithm is $d = 2$ under the uniform budget preference, where \alg{tpa}, \alg{xnes}, \alg{csa} and \alg{hsa} form a cycle, so that every selection among them is strictly correct.
The cycle is one of near-ties: in each of the six pairs among the four, one algorithm is preferred to the other by a margin of 0.003 to 0.022, well inside the rope.
\Cref{tab:selection-pools} gives, by pool and preference, the share of selections that are strictly correct and the number of those that are not, beside their bound.
Besides the near-ties of \cref{sec:selection-results}, the rating models fall short under the uniform budget preference at $d = 5$, at 0.25 and 0.35, where the pair models are at 0.90 to 0.95, and at the last budget of $d = 30$, at 0.60 and 0.45, where the pair models are at 0.95 to 1.
Under the uniform budget preference, the three pair models make 3, 2 and 1 selections that are not strictly correct, against bounds of 12.4, 19.9 and 21.2.
By preference, the ratios of the pair models are 1.08 to 1.20 at the first budget, with every interval including 1, and 0.05 to 0.48 under the other two preferences; Plackett--Luce is at 1.61 to 4.93, and Bradley--Terry at 3.68 [2.27, 6.05] under the uniform budget preference, while at the first and at the last budget its intervals include 1.
Of the selections that are not strictly correct over all three preferences, 85 of the 90 of Plackett--Luce, 62 of the 75 of Bradley--Terry and 97 of the 101 of the pair models are correct within the rope.
At the last budget of $d = 40$, with a margin of 0.025, Plackett--Luce selects \alg{ncsa} in 19 of its 20 races, where the truth selects \alg{tpa}; these selections are correct within the rope.

\Cref{tab:selection-ways} sets the three ways of \cref{sec:selection} side by side for the rating models under the uniform budget preference; at the first and at the last budget the three are the same event and agree in every race.
Selection on utilities selects another algorithm than selection on win probabilities in 2 races of Plackett--Luce, both on GP-BBOB, and in 1 of Bradley--Terry, at $d = 2$.
Its shares of strictly correct selections differ from those of selection on win probabilities by at most 0.01.
Selection on ratings does so in 8 races of Plackett--Luce, spread over five pools, 3 on GP-BBOB, 2 at $d = 5$ and one each at $d = 2$, 10 and 40, and in 15 of Bradley--Terry, 5 at $d = 20$, 3 each at $d = 40$ and on GP-BBOB, and 2 each at $d = 5$ and $d = 30$.
In selection on ratings, the selections of Bradley--Terry are less often strictly correct and less often within the rope, and both rating models state lower probabilities.

\begin{table}[tbp]
\caption[The settings of the races on the pools]{The settings of the 700 races on the pools. The fixed cost of a round, $\kappa$, is set by pool and given in blocks.}
\label{tab:settings}
\label{tab:kappa}
\small
\begin{tabular}{@{}lrrrrrrr@{}}
\toprule
Setting & \multicolumn{7}{@{}l@{}}{Value} \\
\midrule
Test over budgets & \multicolumn{7}{@{}l@{}}{joint} \\
Decision threshold, $\alpha$ & \multicolumn{7}{@{}l@{}}{0.99} \\
Rope, $\epsilon$ & \multicolumn{7}{@{}l@{}}{0.05} \\
Blocks in the initial design, $N_0$ & \multicolumn{7}{@{}l@{}}{8 (Bradley--Terry: 1/pair)} \\
Margin for Monte Carlo error, $z$ & \multicolumn{7}{@{}l@{}}{2} \\
Resolution of a crossing & \multicolumn{7}{@{}l@{}}{early} \\
Calibration & \multicolumn{7}{@{}l@{}}{none} \\
Block cap & \multicolumn{7}{@{}l@{}}{2{,}165} \\
Failed diagnostic checks before the next stage & \multicolumn{7}{@{}l@{}}{3} \\
\midrule
 & \multicolumn{6}{c}{MA-BBOB, $d$} & \\
\cmidrule(lr){2-7}
Fixed cost of a round, $\kappa$, in blocks & 2 & 5 & 10 & 20 & 30 & 40 & GP-BBOB \\
\midrule
\quad Plackett--Luce, per-pair, both Dirichlet models & 27.8 & 30.5 & 22.1 & 14.0 & 5.7 & 10.7 & 1.6 \\
\quad Bradley--Terry & 16.5 & 106.2 & 5.3 & 3.0 & 11.2 & 2.0 & 0.6 \\
\bottomrule
\end{tabular}
\end{table}

\begin{table}[tbp]
\caption[One race under the five models]{The race of \cref{fig:one-race} and the races of the other four models on the same seed of the $d = 30$ pool: the runs of each algorithm, and the instances and runs of the race. All five return the six algorithms above the line. Under every model but Bradley--Terry, an algorithm below the line is eliminated after as many instances as it has runs.}
\label{tab:one-race-models}
\small
\begin{tabular}{@{}lrrrrr@{}}
\toprule
Algorithm & Plackett--Luce & Bradley--Terry & Per-pair & Tie-point prior & Near-ignorance prior \\
\midrule
\alg{csa} & 710 & 1{,}021 & 724 & 710 & 766 \\
\alg{hsa} & 602 & 1{,}298 & 724 & 336 & 668 \\
\alg{m-xnes} & 138 & 170 & 107 & 107 & 99 \\
\alg{ncsa} & 710 & 1{,}611 & 724 & 710 & 766 \\
\alg{psr} & 138 & 998 & 199 & 206 & 200 \\
\alg{tpa} & 87 & 187 & 219 & 206 & 200 \\
\midrule
\alg{random} & 8 & 18 & 17 & 8 & 8 \\
\alg{lp-xnes} & 16 & 66 & 59 & 64 & 66 \\
\alg{none} & 16 & 147 & 59 & 64 & 66 \\
\alg{xnes} & 29 & 2{,}247 & 174 & 171 & 188 \\
\alg{ppmf} & 138 & 419 & 199 & 206 & 200 \\
\alg{msr} & 87 & 278 & 219 & 206 & 200 \\
\midrule
Instances & 710 & 4{,}230 & 724 & 710 & 766 \\
Runs & 2{,}679 & 8{,}460 & 3{,}424 & 2{,}994 & 3{,}427 \\
\bottomrule
\end{tabular}
\end{table}

\begin{table}[tbp]
\caption[The leaders from $d = 20$ on]{The budgets, in evaluations, at which each algorithm leads in the truth of the MA-BBOB pools from $d = 20$ on. At $d = 20$, no algorithm leads at 38{,}566 evaluations.}
\label{tab:leaders}
\small
\begin{tabular}{@{}llll@{}}
\toprule
Leader & $d = 20$ & $d = 30$ & $d = 40$ \\
\midrule
\alg{tpa} & & 100 & 100 and 127 \\
\alg{m-xnes} & 100 to 530 & 127 to 672 & 161 to 853 \\
\alg{psr} & 672 to 1{,}083 & 853 to 1{,}743 & 1{,}083 to 2{,}807 \\
\alg{hsa} & 1{,}374 to 5{,}736 & 2{,}212 to 7{,}279 & 3{,}562 to 9{,}237 \\
\alg{csa} & 7{,}279 & 9{,}237 and 11{,}721 & 11{,}721; 30{,}392 \\
\alg{ncsa} & 9{,}237 to 30{,}392 & 14{,}874 to 48{,}939 & 14{,}874 to 23{,}950; 38{,}566 to 78{,}805 \\
\alg{tpa} & 48{,}939 to $10^{5}$ & 62{,}102 to $10^{5}$ & $10^{5}$ \\
\bottomrule
\end{tabular}
\end{table}

\begin{table}[tbp]
\caption[Results by pool]{The columns of \cref{tab:results} by pool, 20 races per model and pool. False eliminations and false equivalences stand against their bound; directional statements give the false ones over their bound, $X/\Lambda$, with a 95\,\% interval that resamples races. Evaluations, on GP-BBOB algorithm time, are relative to the per-pair races of the same pool, paired by seed, since the races of one seed draw their instances from the same sequence; the 95\,\% interval resamples seeds.}
\label{tab:results-pools}
\footnotesize
\setlength{\tabcolsep}{3pt}
\begin{tabular}{@{}lccccccc@{}}
\toprule
 & \multicolumn{2}{c}{Returned sets} & False & Directional & False & & Races at \\
\cmidrule(lr){2-3}
Pool & within & exact & eliminations & statements, $X/\Lambda$ & equivalences & Evaluations & the cap \\
\midrule
\multicolumn{8}{@{}l}{\emph{Plackett--Luce}} \\
$d = 2$ & 20 & 0 & 0 / 0.11 & 6.02 {\scriptsize [4.88, 7.20]} & 110 / 4.43 & 0.67 {\scriptsize [0.62, 0.72]} & 0 \\
$d = 5$ & 10 & 1 & 10 / 0.18 & 9.38 {\scriptsize [7.51, 11.32]} & 62 / 3.37 & 0.90 {\scriptsize [0.83, 0.99]} & 0 \\
$d = 10$ & 20 & 3 & 0 / 0.15 & 3.45 {\scriptsize [2.31, 4.59]} & 15 / 0.96 & 0.95 {\scriptsize [0.86, 1.03]} & 2 \\
$d = 20$ & 20 & 20 & 1 / 0.19 & 2.77 {\scriptsize [2.04, 3.55]} & 14 / 0.44 & 0.66 {\scriptsize [0.60, 0.73]} & 0 \\
$d = 30$ & 20 & 20 & 0 / 0.21 & 2.15 {\scriptsize [1.04, 3.55]} & 0 / 0.16 & 0.77 {\scriptsize [0.65, 0.90]} & 0 \\
$d = 40$ & 20 & 6 & 0 / 0.16 & 4.48 {\scriptsize [2.83, 6.30]} & 5 / 0.45 & 1.14 {\scriptsize [1.01, 1.28]} & 1 \\
GP-BBOB & 20 & 19 & 0 / 0.26 & 2.03 {\scriptsize [0.62, 3.97]} & 8 / 0.60 & 0.91 {\scriptsize [0.82, 1.02]} & 1 \\
\addlinespace
\multicolumn{8}{@{}l}{\emph{Bradley--Terry}} \\
$d = 2$ & 20 & 0 & 0 / 0.23 & 1.49 {\scriptsize [1.10, 1.95]} & 17 / 2.17 & 1.97 {\scriptsize [1.75, 2.20]} & 0 \\
$d = 5$ & 20 & 0 & 0 / 0.12 & 2.04 {\scriptsize [1.43, 2.72]} & 7 / 2.06 & 2.75 {\scriptsize [2.53, 2.98]} & 1 \\
$d = 10$ & 20 & 8 & 0 / 0.37 & 1.23 {\scriptsize [0.78, 1.75]} & 0 / 0.24 & 1.59 {\scriptsize [1.43, 1.79]} & 1 \\
$d = 20$ & 20 & 14 & 0 / 0.39 & 0.75 {\scriptsize [0.49, 1.06]} & 0 / 0.17 & 1.54 {\scriptsize [1.27, 1.86]} & 5 \\
$d = 30$ & 20 & 20 & 0 / 0.36 & 1.13 {\scriptsize [0.83, 1.45]} & 0 / 0.12 & 1.85 {\scriptsize [1.54, 2.22]} & 0 \\
$d = 40$ & 20 & 4 & 0 / 0.41 & 0.95 {\scriptsize [0.62, 1.33]} & 0 / 0.19 & 1.47 {\scriptsize [1.26, 1.72]} & 3 \\
GP-BBOB & 20 & 20 & 0 / 0.63 & 1.93 {\scriptsize [0.82, 3.55]} & 0 / 0.30 & 1.68 {\scriptsize [1.37, 2.09]} & 0 \\
\addlinespace
\multicolumn{8}{@{}l}{\emph{Per-pair}} \\
$d = 2$ & 20 & 0 & 0 / 0.34 & 0.57 {\scriptsize [0.29, 0.88]} & 0 / 3.77 & 1 & 0 \\
$d = 5$ & 20 & 1 & 0 / 0.23 & 0.66 {\scriptsize [0.43, 0.92]} & 2 / 3.48 & 1 & 0 \\
$d = 10$ & 20 & 5 & 0 / 0.30 & 0.57 {\scriptsize [0.28, 0.95]} & 0 / 0.83 & 1 & 1 \\
$d = 20$ & 20 & 18 & 0 / 0.39 & 0.43 {\scriptsize [0.21, 0.72]} & 0 / 0.53 & 1 & 0 \\
$d = 30$ & 20 & 20 & 0 / 0.55 & 0.64 {\scriptsize [0.40, 0.99]} & 0 / 0.17 & 1 & 0 \\
$d = 40$ & 20 & 4 & 0 / 0.33 & 0.87 {\scriptsize [0.54, 1.24]} & 1 / 0.30 & 1 & 1 \\
GP-BBOB & 20 & 18 & 0 / 0.68 & 1.01 {\scriptsize [0.56, 1.51]} & 0 / 0.84 & 1 & 1 \\
\addlinespace
\multicolumn{8}{@{}l}{\emph{Tie-point prior}} \\
$d = 2$ & 20 & 0 & 0 / 0.19 & 15.04 {\scriptsize [10.46, 20.29]} & 0 / 1.44 & 0.85 {\scriptsize [0.80, 0.90]} & 1 \\
$d = 5$ & 20 & 1 & 0 / 0.17 & 6.98 {\scriptsize [3.83, 11.90]} & 0 / 1.23 & 0.92 {\scriptsize [0.88, 0.95]} & 0 \\
$d = 10$ & 20 & 6 & 0 / 0.22 & 4.01 {\scriptsize [2.43, 6.05]} & 0 / 0.38 & 0.89 {\scriptsize [0.82, 0.96]} & 1 \\
$d = 20$ & 20 & 18 & 0 / 0.25 & 1.52 {\scriptsize [0.94, 2.22]} & 0 / 0.28 & 0.93 {\scriptsize [0.89, 0.96]} & 0 \\
$d = 30$ & 20 & 20 & 0 / 0.33 & 2.76 {\scriptsize [1.73, 4.03]} & 0 / 0.07 & 0.99 {\scriptsize [0.95, 1.04]} & 0 \\
$d = 40$ & 20 & 5 & 0 / 0.22 & 3.17 {\scriptsize [2.27, 4.17]} & 2 / 0.16 & 1.03 {\scriptsize [0.95, 1.13]} & 1 \\
GP-BBOB & 20 & 20 & 0 / 0.46 & 3.08 {\scriptsize [1.41, 5.29]} & 0 / 0.44 & 0.92 {\scriptsize [0.83, 1.02]} & 0 \\
\addlinespace
\multicolumn{8}{@{}l}{\emph{Near-ignorance prior}} \\
$d = 2$ & 20 & 0 & 0 / 0.21 & 0.75 {\scriptsize [0.41, 1.16]} & 0 / 1.52 & 0.88 {\scriptsize [0.82, 0.93]} & 1 \\
$d = 5$ & 20 & 1 & 0 / 0.21 & 0.67 {\scriptsize [0.45, 0.90]} & 2 / 1.46 & 0.93 {\scriptsize [0.89, 0.98]} & 1 \\
$d = 10$ & 20 & 6 & 0 / 0.26 & 0.52 {\scriptsize [0.26, 0.87]} & 0 / 0.45 & 0.98 {\scriptsize [0.92, 1.04]} & 1 \\
$d = 20$ & 20 & 17 & 0 / 0.32 & 0.21 {\scriptsize [0.05, 0.42]} & 0 / 0.27 & 0.94 {\scriptsize [0.89, 0.98]} & 0 \\
$d = 30$ & 20 & 20 & 0 / 0.40 & 0.38 {\scriptsize [0.19, 0.62]} & 1 / 0.08 & 1.10 {\scriptsize [1.00, 1.22]} & 0 \\
$d = 40$ & 20 & 4 & 0 / 0.24 & 0.42 {\scriptsize [0.25, 0.62]} & 1 / 0.21 & 1.15 {\scriptsize [1.06, 1.26]} & 1 \\
GP-BBOB & 20 & 20 & 0 / 0.55 & 0.66 {\scriptsize [0.31, 1.07]} & 0 / 0.47 & 0.95 {\scriptsize [0.88, 1.04]} & 0 \\
\bottomrule
\end{tabular}
\end{table}

\begin{table}[tbp]
\caption[The truths and the tie-point prior by pool]{By pool, the cycles of the truth at single budgets, the share of equal values over all 66 pairs and 5{,}000 instances, over the 30 budgets and at the last, and the false directional statements of the tie-point prior: all of them, those declared in the first round, on eight instances, and those with a recorded probability of exactly 1.}
\label{tab:why-pools}
\small
\begin{tabular}{@{}lllrrrrr@{}}
\toprule
 & \multicolumn{2}{c}{Cycles} & \multicolumn{2}{c}{Ties, \%} & \multicolumn{3}{c}{False statements, tie-point prior} \\
\cmidrule(lr){2-3}\cmidrule(lr){4-5}\cmidrule(l){6-8}
Pool & number & at budgets & all & last & all & first round & probability 1 \\
\midrule
$d = 2$ & 2 & 14{,}874 & 18.5 & 31.8 & 404 & 355 & 338 \\
$d = 5$ & 0 & & 12.6 & 29.8 & 186 & 152 & 131 \\
$d = 10$ & 3 & 2{,}807; 18{,}874 & 7.7 & 26.6 & 123 & 102 & 82 \\
$d = 20$ & 1 & 38{,}566 & 3.3 & 18.2 & 50 & 28 & 16 \\
$d = 30$ & 0 & & 2.0 & 10.5 & 100 & 67 & 27 \\
$d = 40$ & 0 & & 1.4 & 6.7 & 90 & 61 & 19 \\
GP-BBOB & 1 & 10\,s & 0.002 & 0.002 & 162 & 112 & 37 \\
\midrule
All pools & 7 & & & & 1{,}115 & 877 & 650 \\
\bottomrule
\end{tabular}
\end{table}

\begin{table}[tbp]
\caption[Selection by pool]{Selection on win probabilities by pool and preference, 20 races per model and pool: the share of selections that are strictly correct, and the selections that are not strictly correct over all seven pools, against their bound.}
\label{tab:selection-pools}
\small
\setlength{\tabcolsep}{4.5pt}
\begin{tabular}{@{}lccccc@{}}
\toprule
Pool & Plackett--Luce & Bradley--Terry & Per-pair & Tie-point prior & Near-ignorance prior \\
\midrule
\multicolumn{6}{@{}l}{\emph{All budgets equally}} \\
$d = 2$ & 1.00 & 1.00 & 1.00 & 1.00 & 1.00 \\
$d = 5$ & 0.25 & 0.35 & 0.90 & 0.90 & 0.95 \\
$d = 10$ & 0.70 & 0.95 & 1.00 & 1.00 & 1.00 \\
$d = 20$ & 1.00 & 0.90 & 1.00 & 1.00 & 1.00 \\
$d = 30$ & 1.00 & 1.00 & 1.00 & 1.00 & 1.00 \\
$d = 40$ & 1.00 & 0.90 & 1.00 & 1.00 & 1.00 \\
GP-BBOB & 0.70 & 0.75 & 0.95 & 1.00 & 1.00 \\
Not strictly correct & 27 / 5.5 & 23 / 6.3 & 3 / 12.4 & 2 / 19.9 & 1 / 21.2 \\
\addlinespace
\multicolumn{6}{@{}l}{\emph{First budget only}} \\
$d = 2$ & 1.00 & 1.00 & 1.00 & 1.00 & 1.00 \\
$d = 5$ & 0.00 & 0.30 & 0.40 & 0.40 & 0.35 \\
$d = 10$ & 1.00 & 0.70 & 0.95 & 1.00 & 1.00 \\
$d = 20$ & 0.25 & 0.45 & 0.45 & 0.45 & 0.45 \\
$d = 30$ & 0.95 & 0.95 & 0.95 & 0.90 & 1.00 \\
$d = 40$ & 1.00 & 1.00 & 1.00 & 1.00 & 1.00 \\
GP-BBOB & 1.00 & 1.00 & 1.00 & 1.00 & 1.00 \\
Not strictly correct & 36 / 15.0 & 32 / 24.7 & 25 / 20.8 & 25 / 21.3 & 24 / 22.3 \\
\addlinespace
\multicolumn{6}{@{}l}{\emph{Last budget only}} \\
$d = 2$ & 1.00 & 0.90 & 1.00 & 0.90 & 0.95 \\
$d = 5$ & 1.00 & 1.00 & 1.00 & 1.00 & 1.00 \\
$d = 10$ & 1.00 & 1.00 & 1.00 & 1.00 & 1.00 \\
$d = 20$ & 1.00 & 1.00 & 1.00 & 1.00 & 1.00 \\
$d = 30$ & 0.60 & 0.45 & 1.00 & 0.95 & 1.00 \\
$d = 40$ & 0.05 & 0.65 & 0.65 & 0.75 & 0.85 \\
GP-BBOB & 1.00 & 1.00 & 0.95 & 0.95 & 1.00 \\
Not strictly correct & 27 / 16.7 & 20 / 24.8 & 8 / 22.8 & 9 / 18.8 & 4 / 18.5 \\
\addlinespace
\multicolumn{6}{@{}l}{\emph{All three preferences}} \\
Not strictly correct & 90 / 37.2 & 75 / 55.8 & 36 / 56.0 & 36 / 60.1 & 29 / 62.0 \\
\bottomrule
\end{tabular}
\end{table}

\begin{table}[tbp]
\caption[Three ways of selecting]{The three ways of selecting for the rating models under the uniform budget preference, over all seven pools, 140 races per model: the share of selections that are strictly correct, the share correct within the rope, and the mean stated probability.}
\label{tab:selection-ways}
\small
\begin{tabular}{@{}llccc@{}}
\toprule
Model & Selection on & strict & rope & stated \\
\midrule
Plackett--Luce & win probabilities & 0.81 & 1.00 & 0.96 \\
 & utilities & 0.80 & 1.00 & 0.96 \\
 & ratings & 0.81 & 0.99 & 0.94 \\
\addlinespace
Bradley--Terry & win probabilities & 0.84 & 0.98 & 0.96 \\
 & utilities & 0.84 & 0.98 & 0.95 \\
 & ratings & 0.75 & 0.94 & 0.93 \\
\bottomrule
\end{tabular}
\end{table}

\begin{table}[tbp]
\caption[The three racers and \textsc{PolarBear} at one budget]{The three racers of \cref{sec:racers} and \textsc{PolarBear} under the per-pair model at one budget. All four see the same instances, decide for the first time at the eighth, and stop at the cap of 2{,}165 instances at the latest; the threshold is 0.99, and for F-race a significance of 0.01. The \emph{best} is the survivor with the lowest rank sum, the ranks taken within instances among the survivors. F-race averages the ranks of tied values, and the other three count a tie as half a win and half a loss.}
\label{tab:racers-one}
\scriptsize
\setlength{\tabcolsep}{3pt}
\renewcommand{\arraystretch}{1.3}
\begin{tabular}{@{}L{50pt}L{80pt}L{80pt}L{80pt}L{80pt}@{}}
\toprule
 & F-race & B-race (Benavoli tests) & B-race (pairwise tests) & \textsc{PolarBear}, per-pair model \\
\midrule
Reads & the ranks of the values within an instance & a Dirichlet process posterior over the instances, under the tie-point prior of strength 1 (\cref{sec:dp}) & for every pair, $\mathrm{Beta}(1 + w, 1 + l)$ on its wins and losses (\cref{sec:per-pair}) & the same posterior \\
Decides & after every instance & after every instance, once there are as many instances as survivors & after every instance & at the end of a round (\cref{sec:procedure}) \\
Tests first & whether any survivors differ, by the Friedman test & the same, by the Bayesian Friedman test & nothing & nothing \\
Eliminates & every survivor whose rank sum exceeds the best's by more than a critical difference & every loser of a statement that one survivor beats another, the statements taken in order while their joint probability stays above the threshold; never the best & every survivor that the best beats with a posterior probability at the threshold & every survivor that any other survivor beats with a posterior probability at the threshold \\
With two survivors & a paired Wilcoxon signed-rank test on the differences of the values & as with more & as with more & as with more \\
A pair within the rope & no rule & the one with the larger rank sum leaves, when the posterior probability of the rope reaches the threshold & no rule & both stay and the pair is resolved, when the posterior probability of the rope reaches the threshold \\
Ends & at one survivor & at one survivor & at one survivor & when every pair of survivors is resolved \\
Bound on false eliminations & none & none & the sum of the posterior probabilities against them & the same sum, and one for the equivalences \\
\bottomrule
\end{tabular}
\end{table}

\FloatBarrier
\section{Supplementary Floats}
\label{app:supplement}

This appendix holds further tables and figures for \cref{sec:validation} and \cref{sec:replays}.

\begin{table}[tbp]
\caption[Supplement: the cost of a race by pool]{The mean cost of a race for every model on every pool, over the pool's 20 races, with 95\,\% bootstrap intervals over the races: evaluations in $10^{8}$ on the MA-BBOB pools and hours of algorithm time on GP-BBOB.}
\label{tab:cost}
\small
\setlength{\tabcolsep}{4.5pt}
\begin{tabular}{@{}lccccc@{}}
\toprule
Pool & Plackett--Luce & Bradley--Terry & Per-pair & Tie-point prior & Near-ignorance prior \\
\midrule
\multicolumn{6}{@{}l}{\emph{MA-BBOB, evaluations in $10^{8}$}} \\
$d = 2$ & 5.64 & 16.60 & 8.44 & 7.18 & 7.39 \\
 & {\footnotesize [5.41, 5.88]} & {\footnotesize [15.36, 17.96]} & {\footnotesize [7.85, 9.05]} & {\footnotesize [6.56, 7.79]} & {\footnotesize [6.77, 8.00]} \\
$d = 5$ & 6.76 & 20.54 & 7.48 & 6.85 & 6.96 \\
 & {\footnotesize [6.19, 7.32]} & {\footnotesize [19.32, 21.76]} & {\footnotesize [7.01, 7.97]} & {\footnotesize [6.33, 7.35]} & {\footnotesize [6.40, 7.53]} \\
$d = 10$ & 5.44 & 9.17 & 5.75 & 5.12 & 5.61 \\
 & {\footnotesize [5.00, 5.89]} & {\footnotesize [8.44, 9.87]} & {\footnotesize [5.24, 6.29]} & {\footnotesize [4.56, 5.69]} & {\footnotesize [5.09, 6.14]} \\
$d = 20$ & 3.65 & 8.47 & 5.49 & 5.08 & 5.15 \\
 & {\footnotesize [3.33, 3.96]} & {\footnotesize [7.15, 9.90]} & {\footnotesize [5.07, 5.93]} & {\footnotesize [4.63, 5.55]} & {\footnotesize [4.72, 5.59]} \\
$d = 30$ & 2.34 & 5.66 & 3.05 & 3.02 & 3.34 \\
 & {\footnotesize [2.17, 2.49]} & {\footnotesize [4.92, 6.55]} & {\footnotesize [2.69, 3.44]} & {\footnotesize [2.68, 3.37]} & {\footnotesize [3.00, 3.72]} \\
$d = 40$ & 4.85 & 6.28 & 4.26 & 4.40 & 4.89 \\
 & {\footnotesize [4.40, 5.30]} & {\footnotesize [5.29, 7.35]} & {\footnotesize [3.80, 4.76]} & {\footnotesize [3.89, 4.92]} & {\footnotesize [4.42, 5.39]} \\
\addlinespace
\multicolumn{6}{@{}l}{\emph{GP-BBOB, hours of algorithm time}} \\
GP-BBOB & 59.0 & 109.0 & 64.9 & 59.6 & 61.8 \\
 & {\footnotesize [51.8, 65.7]} & {\footnotesize [94.8, 123.4]} & {\footnotesize [56.1, 74.2]} & {\footnotesize [50.4, 68.5]} & {\footnotesize [53.7, 69.7]} \\
\bottomrule
\end{tabular}
\end{table}

\begin{table}[tbp]
\caption[Supplement: selection on win probabilities]{Selection on win probabilities over all seven pools, 140 races per model: the share of selections that are strictly correct, the share correct within the rope, and the mean stated value, a probability for the rating models and a lower bound for the pair models.}
\label{tab:selection}
\small
\setlength{\tabcolsep}{4.5pt}
\begin{tabular}{@{}lccccccccc@{}}
\toprule
 & \multicolumn{3}{c}{All budgets equally} & \multicolumn{3}{c}{First budget only} & \multicolumn{3}{c}{Last budget only} \\
\cmidrule(lr){2-4}\cmidrule(lr){5-7}\cmidrule(l){8-10}
Model & strict & rope & stated & strict & rope & stated & strict & rope & stated \\
\midrule
Plackett--Luce & 0.81 & 1.00 & 0.96 & 0.74 & 0.99 & 0.89 & 0.81 & 0.97 & 0.88 \\
Bradley--Terry & 0.84 & 0.98 & 0.96 & 0.77 & 0.95 & 0.82 & 0.86 & 0.98 & 0.82 \\
Per-pair & 0.98 & 1.00 & 0.91 & 0.82 & 0.99 & 0.85 & 0.94 & 1.00 & 0.84 \\
Tie-point prior & 0.99 & 1.00 & 0.85 & 0.82 & 0.99 & 0.85 & 0.94 & 1.00 & 0.87 \\
Near-ignorance prior & 0.99 & 1.00 & 0.84 & 0.83 & 1.00 & 0.84 & 0.97 & 1.00 & 0.87 \\
\bottomrule
\end{tabular}
\end{table}

\begin{figure}[tbp]
\centering
\includegraphics[width=\textwidth]{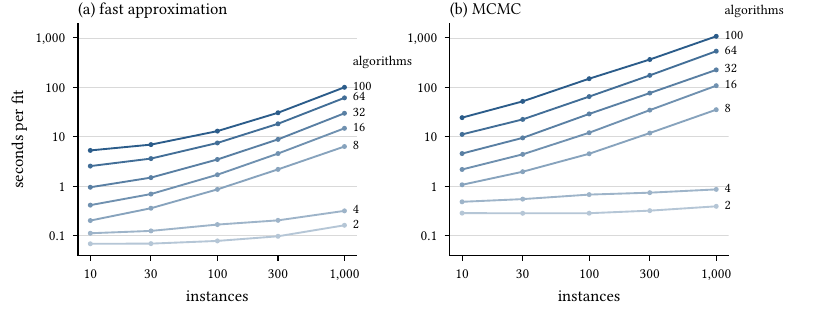}
\caption[Supplement: fit times of Plackett--Luce]{Seconds per fit of Plackett--Luce at one budget against the number of instances, for 2 to 100 algorithms: Pathfinder, the fast approximation (left) and MCMC (right). Each point is the median of five fits on 4 workers. Below 100 instances the cost is nearly flat, because identical rankings are merged and few algorithms have few distinct rankings. All 210 fits by MCMC pass their diagnostic checks, and 202 of the 210 fast approximations.}
\Description{Two panels with logarithmic axes, seconds per fit against the number of instances, one line per number of algorithms.}
\label{fig:fit-times}
\end{figure}

\begin{figure}[tbp]
\centering
\includegraphics[width=\textwidth]{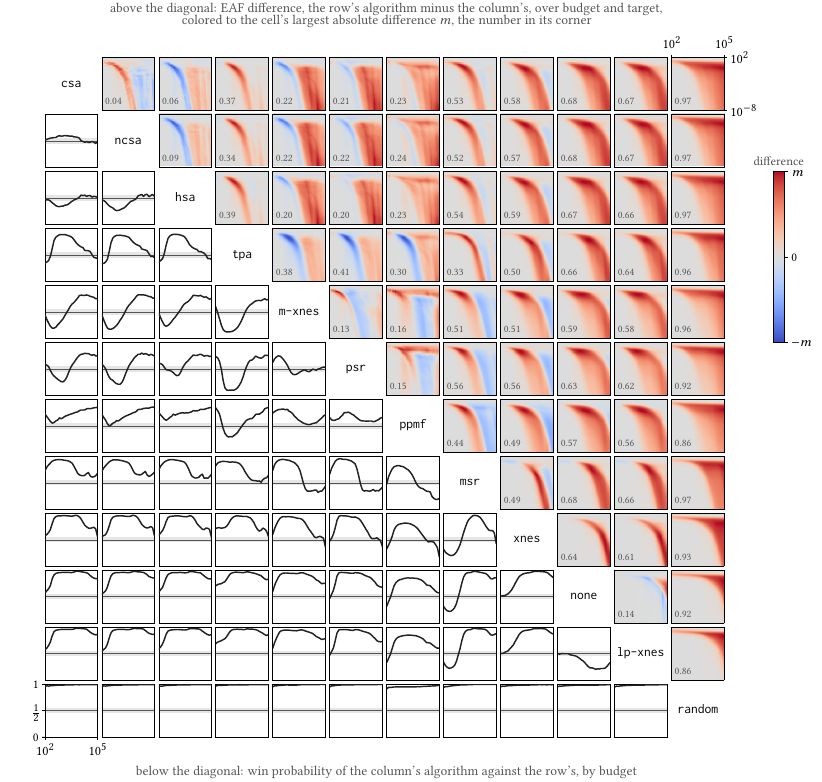}
\caption[Supplement: EAF differences and win probabilities for every pair]{Every pair of the $d = 30$ pool, with the algorithms in the order of their mean AOCC. Above the diagonal: the difference of the two empirical attainment functions over budget and target, the row's algorithm minus the column's, red where the row's algorithm reaches a target on more instances. Below the diagonal, mirrored: the same pair's win probability by budget, of the column's algorithm against the row's, with the rope as a band. As in \cref{fig:eaf}, the colors of a cell span the pair's largest absolute difference, the number in the cell's corner.}
\Description{A grid of twelve rows and columns with the algorithms' names on the diagonal, heat maps above it and curves of win probability below it.}
\label{fig:eaf-pairs}
\end{figure}

\begin{figure}[tbp]
\centering
\includegraphics[width=\textwidth]{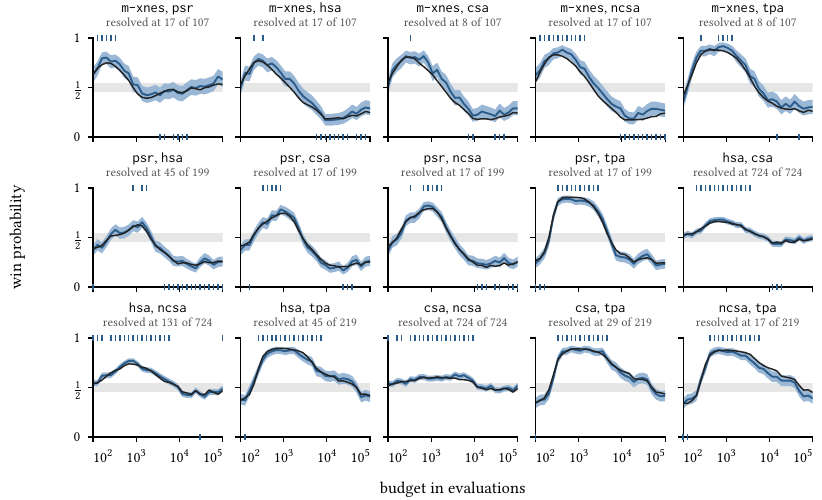}
\caption[Supplement: every pair of the returned set]{The race of \cref{fig:one-race}: the final posterior of every pair of the returned set against the pool's truth, with the mean and its 90\,\% band, the rope, ticks at the budgets where the race declared a directional statement, and the number of instances after which the pair was resolved.}
\Description{Fifteen panels of win probability over the budgets, one per pair of the returned set.}
\label{fig:one-race-pairs}
\end{figure}

\begin{figure}[p]
\centering
\includegraphics[width=\textwidth]{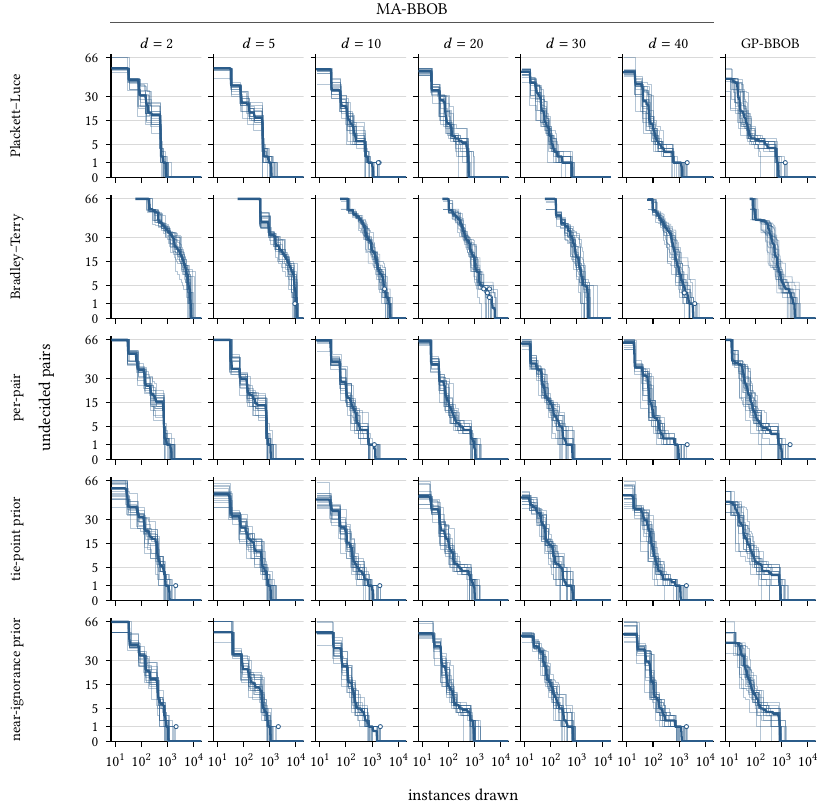}
\caption[Supplement: open pairs over the race]{The number of open pairs among the algorithms still in the race against the instances drawn, for every model and pool: the 20 races and their median.}
\Description{A grid of panels, one per model and pool, each with twenty step lines falling to zero.}
\label{fig:open-pairs}
\end{figure}

\begin{figure}[p]
\centering
\includegraphics[width=\textwidth]{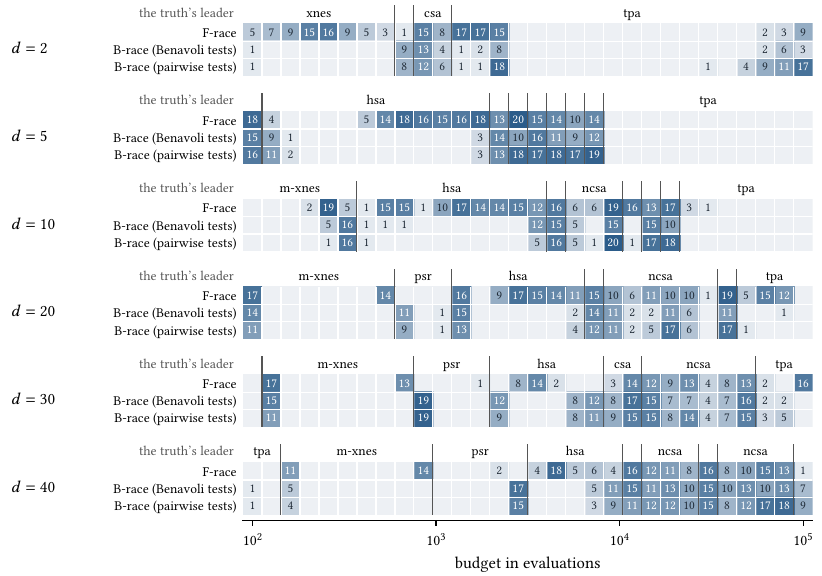}
\caption[Supplement: per-budget racers on six dimensions]{The lower part of \cref{fig:racers} for all six MA-BBOB pools: the truth's leader by budget and, for each racer, the number of its 20 races at that budget that end at the cap.}
\Description{Six panels, one per dimension, each with the leaders' names over thirty budgets and three rows of thirty shaded squares.}
\label{fig:racers-all}
\end{figure}

\begin{figure}[tbp]
\centering
\includegraphics[width=\textwidth]{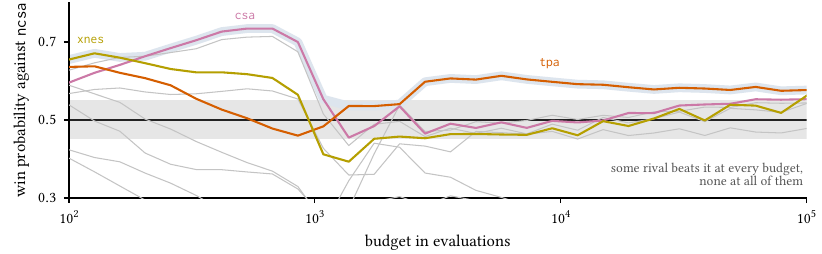}
\caption[Supplement: every rival of \alg{ncsa} at $d = 2$]{The $d = 2$ pool's truth: the win probability of every rival against \alg{ncsa} by budget, with the rope as a band and the strongest rival at each budget highlighted. Some rival beats \alg{ncsa} at every budget, and none at all of them.}
\Description{One panel of win probability over the budgets with a line at one half and a band around it. It holds eleven curves, three of them named.}
\label{fig:ncsa}
\end{figure}

\FloatBarrier
\clearpage
\listoffigures
\listoftables

\end{document}